\documentclass[twoside,11pt]{article}

\usepackage{jmlr2e}

\usepackage{url}            % simple URL typesetting
\usepackage{booktabs}       % professional-quality tables
\usepackage{amsfonts}       % blackboard math symbols
\usepackage{nicefrac}       % compact symbols for 1/2, etc.
\usepackage{microtype}      % microtypography
\usepackage{amssymb}
\usepackage{bbm}
\usepackage{bm}
\usepackage{array}
\usepackage{mdwmath}
\usepackage{mdwtab}
\usepackage{eqparbox}
\usepackage{algorithm}
\usepackage{algpseudocode}
\usepackage{algorithmicx}
\algdef{SE}[DOWHILE]{Do}{doWhile}{\algorithmicdo}[1]{\algorithmicwhile\ #1}%
\usepackage{enumerate}
\usepackage{enumitem} 
    \usepackage[]{nomencl}   
    \makenomenclature
\usepackage[hidelinks]{hyperref} % adds hyper links inside the generated pdf file
\usepackage{graphicx}
\usepackage{epstopdf}

\usepackage{multirow}

\usepackage{amsmath}
\usepackage{bm} % optional
\usepackage{stmaryrd}
\usepackage{mathtools}
\usepackage{threeparttable}

\usepackage{xcolor}
\usepackage{subfig}
\ShortHeadings{Smooth Robust Tensor Completion}{Shen \textit{et al.}}
\firstpageno{1}

\allowdisplaybreaks
\begin{document}

\title{Smooth Robust Tensor Completion for Background/Foreground Separation with Missing Pixels: Novel Algorithm with Convergence Guarantee}

\author{\name Bo Shen \email boshen@vt.edu\\
    %   \addr Department of Industrial and Systems Engineering\\
    %   Virginia Tech\\
    %   Blacksburg, Virginia 24061, USA
    %   \AND
    \name Weijun Xie \email wxie@vt.edu\\
    %   \addr Department of Industrial and Systems Engineering\\
    %   Virginia Tech\\
    %   Blacksburg, Virginia 24061, USA 
    %   \AND
       \name Zhenyu (James) Kong \email zkong@vt.edu\\
       \addr Department of Industrial and Systems Engineering\\
       Virginia Tech\\
       Blacksburg, Virginia 24061, USA}

\editor{}

\maketitle

\begin{abstract}%   <- trailing '%' for backward compatibility of .sty file
 Robust PCA (RPCA) and its tensor extension, namely, Robust Tensor PCA (RTPCA), provide an effective framework for background/foreground separation by decomposing the data into low-rank and sparse components, which contain the background and the foreground (moving objects), respectively.  However, in real-world applications, the presence of missing pixels is a very common but challenging issue due to errors in the acquisition process or manufacturer defects. RPCA and RTPCA are not able to recover the background and foreground simultaneously with missing pixels.   The objective of this study is to address the problem of background/foreground separation with missing pixels by combining the  video recovery,  background/foreground separation into a  single framework.  To achieve this, a smooth robust tensor completion  (SRTC) model is proposed to recover the data and decompose it into  the static background and smooth foreground, respectively.   An efficient algorithm  based on tensor proximal alternating minimization (tenPAM) is implemented  to solve the proposed model with global convergence guarantee under very mild conditions. Extensive experiments on real data demonstrate that the proposed method significantly outperforms the state-of-the-art approaches for background/foreground separation with missing pixels.
\end{abstract}

\begin{keywords}
Robust Tensor Completion (RTC), Spatio-temporal Continuity, Low-rankness, Tensor Proximal Alternating Minimization (tenPAM), Global Convergence.
\end{keywords}

	\setlength{\nomlabelwidth}{2cm} 
	\setlength{\nomitemsep}{-0.8\parsep}
\nomenclature[01]{$H,W,T$}{The height, width, and number of an image frame}%
\nomenclature[02]{$(r_1,r_2,r_3)$}{The multi-linear rank in Tucker Decomposition}%
\nomenclature[03]{$\lambda$}{The balance coefficient in the proposed objective function}%
\nomenclature[04]{$\Omega$}{The  index set of the observed elements}%
\nomenclature[05]{$\BFcalX$}{The order three tensor in $\mathbb{R}^{H\times W \times T}$ represented by $\{\BFX_1,\cdots,\BFX_T\}$}%
\nomenclature[06]{$\BFX_t$}{$t$-th image frame in $\mathbb{R}^{H\times W}$}%
\nomenclature[07]{$\BFcalL$}{The low-rank tensor (static video background)}%
\nomenclature[08]{$\BFcalS$}{The smooth tensor (smooth moving objects)}%
\nomenclature[09]{$\BFcalX \times_n\BFV$}{The mode-$n$ multiplication  of  a  tensor $\BFcalX$ with  a  matrix $\BFV$}%
\nomenclature[10]{$\BFcalC$}{The core tensor in Tucker decomposition}%
\nomenclature[11]{$\BFU_1,\BFU_2,\BFU_3$}{The factor matrices in Tucker decomposition}%
\nomenclature[12]{$\BFU$}{The set of  factor matrices in Tucker decomposition, namely, $\{\BFU_1,\BFU_2,\BFU_3\}$}%
\nomenclature[13]{$\BFU\BFU^{\top}$}{The set of  factor matrices in Tucker decomposition, namely, $\{\BFU_1\BFU_1^{\top},\BFU_2\BFU_2^{\top},\BFU_3\BFU_3^{\top}\}$}%
\nomenclature[14]{$\bm{f}$}{The auxiliary variable}%
\nomenclature[15]{$\BFD_h,\BFD_v,\BFD_t$}{Three vectorizations of the  difference operation along with the  horizontal, vertical, and temporal directions}%
\nomenclature[16]{$\BFD$}{The concatenated difference operation, namely, $[\BFD_h^{\top},\BFD_v^{\top},\BFD_t^{\top}]^{\top}$}%
\nomenclature[17]{$\lVert\cdot\rVert_F$}{The Frobenius norm}%
\nomenclature[18]{$\lVert\cdot\rVert_1$}{The $\ell_1$ norm}%
\nomenclature[19]{$\lVert\cdot\rVert_2$}{The $\ell_2$ norm}%
\nomenclature[20]{$\lVert\cdot\rVert$}{The 2-operator norm}%
\nomenclature[21]{$\lVert\cdot\rVert_{TV1}$}{The anisotropic total variation norm}%
\nomenclature[22]{$\rho$}{The positive coefficient for proximal term}%
\nomenclature[23]{$\texttt{vec}(\cdot)$}{The vectorization operator}%
\nomenclature[24]{$\texttt{ten}(\cdot)$}{The tensorization operator}%
\nomenclature[25]{$\bm{\lambda}^{\bm{f}}$}{The Lagrange multiplier vector and tensor}%
\nomenclature[26]{$\beta^{\bm{f}}$}{The positive penalty scalars}%
\nomenclature[27]{$c_1,c_2$}{The coefficients in the adaptive updating scheme for $\beta^{\bm{f}}$}%
\nomenclature[28]{$\texttt{fftn}(\cdot)$}{The fast 3D Fourier transform}%
\nomenclature[29]{$\texttt{ifftn}(\cdot)$}{The inverse fast 3D Fourier transform}%
\nomenclature[30]{$\textnormal{soft}(\cdot,\cdot)$}{The soft-thresholding operator}%
\nomenclature[31]{$\gamma$}{The parameter associated with convergence rate in ADMM}%
\nomenclature[32]{$\textnormal{Err}(\cdot)$}{The error of the auxiliary variable}%
% \nomenclature[33]{$\textbf{\textnormal{relChgA}}$}{The relative change of $\BFcalA$}%
% \nomenclature[34]{$\textbf{\textnormal{relErrA}}$}{The relative error of $\BFcalA$}%

    \printnomenclature
\section{Introduction} \label{sec: intro}
Background/foreground separation is a fundamental step for moving object detection in many video data applications~\citep{zhou2012moving}.  It is usually performed by separating the moving objects called ``foreground'' from the static objects called ``background''~\citep{bouwmans2017decomposition}.    In many real-world applications, the presence of missing pixels is a very common but challenging issue~\citep{firtha2008methods,liu2012tensor,ren2021incremental} due to errors in the acquisition process or manufacturer defects. \begin{figure}[!ht]
\centering
\includegraphics[width=1\textwidth]{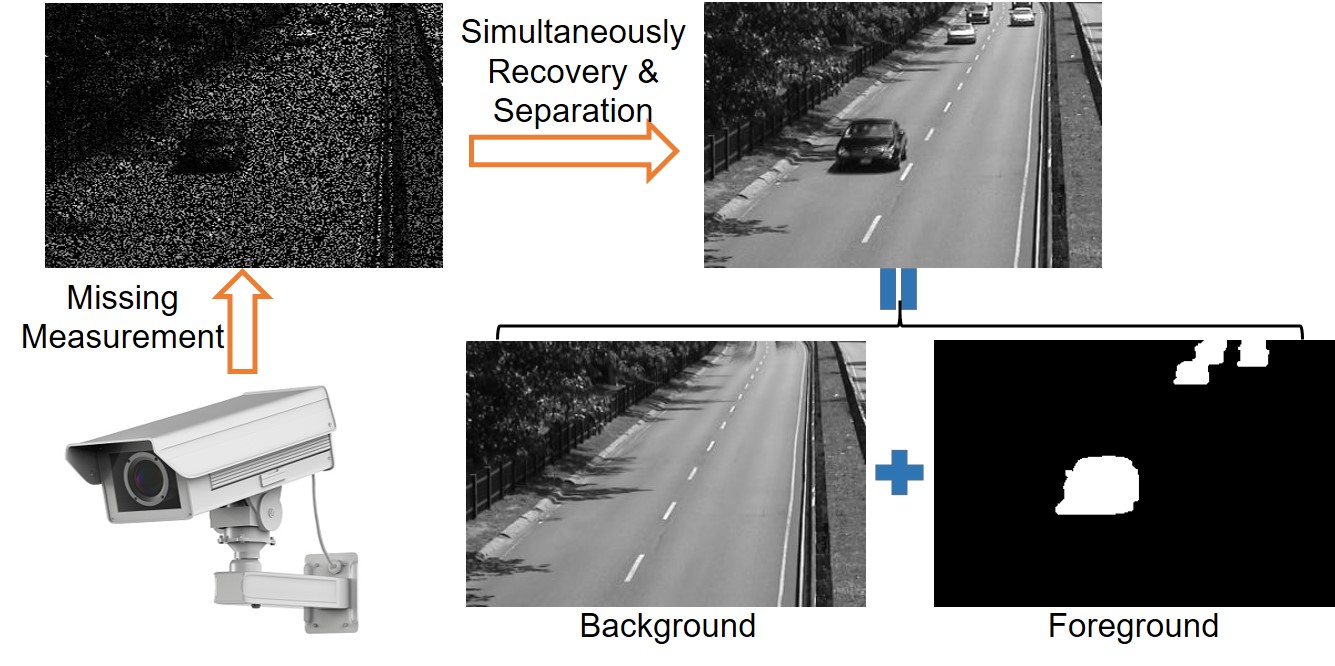} 
\caption{The framework of the imaging system with missing pixels.}
    \label{fig: general structure}
\end{figure}  In this paper, we are interested in background/foreground separation with missing pixels as shown in Figure~\ref{fig: general structure}, which aims to recover the original video with high fidelity and meanwhile accurately separate the moving objects from the video background based on partially observed pixels. The video imaging system first captures missing measurements from the scenes, and then transmits these measurements to the processing center for recovery and background/foreground separation.

In the current literature, research on background/foreground separation is based on decomposition of the video data into low-rank and sparse components. It is  an effective framework to separate the foreground  from the background,  which are modeled by the sparse   and low-rank components, respectively. Among them, the most representative problem formulation is the robust principal component analysis (RPCA)~\citep{candes2011robust}, which is a modification of the widely used statistical procedure named principal component analysis (PCA). RPCA decomposes the video data $\BFX$ into the sum of a low-rank component $\BFL$ and  a sparse component $\BFE$, where the low-rankness and sparsity are measured by the nuclear norm $\|\cdot\|_*$ and $\ell_1$ norm $\|\cdot\|_1$, respectively. One major disadvantage of RPCA is that it can only deal with 2-D matrix data since the nuclear norm $\|\cdot\|_*$ is designed for matrix. However, real-world data is usually multi-dimensional in nature, where rich information is stored in multi-way arrays known as tensors~\citep{kolda2009tensor}. For example, a greyscale video is  3-D data, which stacks multiple images along with the time  domain; a color image is also 3-D data that has three channels: red, green, and blue, where each channel is a 2-D image. To apply RPCA to these data sets, the multi-way tensor data has to be reconstructed into a matrix. Such a preprocessing usually leads to information loss and performance degradation since the structure information in the data is deteriorated. To address this issue, it is necessary to consider extending RPCA to manipulate the tensor data directly by taking advantage of its multi-dimensional structure. %However, it is challenging to do so since the numerical algebra of tensors is fraught with many computationally hard problems~\citep{hillar2013most,anandkumar2017homotopy}.  

Contributed by the newly developed tensor multiplication scheme on t-SVD~\citep{kilmer2011factorization}, \citet{zhang2014novel} proposed the tensor tubal rank as well as the tensor nuclear norm for image denoising. Based on the tensor nuclear norm, \citet{lu2016tensor} developed robust tensor PCA (RTPCA) by extending RPCA from 2-D matrix to 3-D tensor data, aiming  to exactly recover a low-rank tensor contaminated by sparse errors.  More specifically,  it tries to recover the low-rank tensor $\BFcalL$ and sparse tensor $\BFcalE$ from the data tensor $\BFcalX$, which can be represented $\BFcalX=\BFcalL+\BFcalE$. % To further improve RTPCA, some work has been proposed with different objective functions and constraints~\citep{liu2018improved,lu2019tensor,yang2020low}.
% However, none of them is able to address the background/foreground separation problem in noisy cases because the low-rank and sparse components extracted from RPCA or RTPCA algorithms do not consider the noise components in their analyses. There are works~\citep{zhou2011godec,cao2015total} that decompose the video $\BFcalX$ into three components, namely, $\BFcalL$, $\BFcalS$, and $\BFcalE$. They focus on using $\BFcalL + \BFcalE$ to model the dynamic background, where $\BFcalL$ represents the static background and $\BFcalE$ represents the small changes in the dynamic background. They are all vector-based methods without using the structure information in the tensor.  

Recent research on the missing value estimation problem in video data is focused on matrix completion (MC)~\citep{candes2009exact}. MC is able to recover the original signal $\BFcalX$ from a partially observed signal $\BFcalX_{\Omega}$ (or called the undersampled/incomplete signal), where $\Omega$ is a subset containing 2D coordinates of sampled entries. In this pioneering work of~\citep{candes2009exact}, the signal $\BFcalX$ is recovered by solving  a convex relaxation of the rank minimization problem based on the nuclear norm $\|\cdot\|_*$. Furthermore, \citet{zhou2017tensor} extended matrix completion to tensor completion (TC) based on the tensor nuclear norm.  However, none of RPCA/RTPCA and MC/TC is able to address the background/foreground separation with missing pixels because RPCA/RTPCA cannot  recover the video data and MC/TC cannot separate the background and foreground.   There are work on robust matrix completion/tensor completion~\citep{chen2015matrix,he2019robust,he2020robust,li2021robust,shang2017bilinear,huang2021robust}, which combines RPCA/RTPCA and MC/TC. Specifically, they aim to  reconstruct a signal from its noisy observations of a small, random subset of its entries. The problem with these methods is that their learned sparse component cannot represent the foreground since  the sparse component tends to set the foreground to zero for the positions of those missing pixels. As  a  result, the obtained foreground is incomplete due to the sparsity modeling.

%  Deep learning-based methods for background/foreground separation have also shown promise. In~\citep{braham2016deep}, the authors present a background subtraction algorithm based on spatial features learned with convolutional neural networks. In~\citep{de2017background}, background estimation at each pixel is carried out by weightless neural networks designed to learn pixel color frequency during video play, and all networks share the same rule for memory retention during training. Another supervised learning-based approach is to apply a semantic segmentation model trained on a labeled data set containing objects of interest to directly produce foreground masks for each frame of a video. In the field of semantic segmentation, deep convolutional networks are the most popular models, such as Mask R-CNN~\citep{he2017mask} and DeepLab~\citep{chen2018encoder}. The disadvantage of such supervised methods is that they  require the classes of foreground objects during training. Instead, unsupervised techniques like RPCA and RTPCA are applicable to data sets with arbitrary semantic content.  While deep learning-based methods have shown promise in the noiseless regime, they are generally not designed to perform background/foreground separation in the presence of noise, which is the focus of our work.

The objective of this study is to address the problem of background/foreground separation with missing pixels by combining the  video recovery,  background/foreground separation into one single framework.  Compared to the conventional method, this new method need not fully sense all the video pixels, and thus heavily reduces the computational and storage costs and even the energy consumption of imaging sensors. To achieve this objective, a smooth robust tensor completion (SRTC) is proposed to recover the data tensor $\BFcalX$ and decompose it into a low-rank tensor (background) $\BFcalL$ and a smooth tensor (foreground) $\BFcalS$, namely, $\BFcalX = \BFcalL + \BFcalS$. In the SRTC, the background is modeled by the low-rank Tucker decomposition~\citep{kolda2009tensor}. The  spatio-temporal continuity is applied to formulate the moving objects (foreground)~\citep{cao2015total,cao2016total,shen2021robust}. That is, the moving objects in video foreground are spatially continuous in both their support regions and their intensity values in these regions. Moreover, the moving objects are also temporally continuous among succeeding frames.  To summarize, the contributions of this paper are as follows: 
\begin{itemize}
\item Propose the  smooth robust tensor completion model for background/foreground separation with missing pixels by simultaneously recovering the tensor data  and decomposing it  into a low-rank tensor and a smooth tensor, respectively;
\item Implement an efficient tensor proximal alternating minimization (tenPAM) algorithm to solve the proposed model;
\item Analyze the convergence of the iterative sequence generated by the tenPAM algorithm and prove it to be globally convergent to the stationary points under very mild conditions.
\end{itemize}

The remainder of this paper is organized as follows. A brief review of notation and related research work is provided in Section~\ref{sec: notation and background}. The proposed model and algorithm to solve this model are introduced in Section~\ref{sec: proposed method}, followed by the convergence analysis in Section~\ref{sec: convergence analysis}. Numerical studies in  Section~\ref{sec: numerical study} are provided for testing and validation of the proposed method. Finally, the conclusions are discussed in Section~\ref{sec: conclusion}. 

% Numerical studies in   Section~\ref{sec: numerical study} and real-world  additive manufacturing application in Section~\ref{sec: real-world case study} are provided for testing and validation of the proposed method. Finally, the conclusions and future work are discussed in Section~\ref{sec: conclusion}. 

\section{Notation and Research Background} \label{sec: notation and background}
In Section~\ref{subsec: tensor basis},   the notation and basics in multi-linear algebra used in this paper are reviewed. Then, the robust tensor PCA, tensor completion, and robust matrix/tensor completion are reviewed briefly in Section~\ref{subsec: related work}. Afterward, the research gaps in the existing work are identified in Section~\ref{subsec: reseach gap}.
\subsection{Notation and Tensor Basis} \label{subsec: tensor basis}
Throughout this paper, scalars are denoted by lowercase letters, e.g., $x$; vectors are denoted by lowercase boldface letters, e.g., $\bm{x}$; matrices are denoted by uppercase boldface, e.g., $\BFX$; and tensors are denoted by calligraphic letters, e.g., $\BFcalX$. The order of a tensor is the number of its modes or dimensions. A real-valued tensor of order $N$ is denoted by $\BFcalX \in \mathbb{R}^{I_1\times I_2 \times \dots \times I_N}$ and its entries by $\BFcalX(i_1,i_2,\cdots,i_N)$.  The multi-linear Tucker rank of an $N$-order tensor is the tuple of the ranks of the mode-$n$ unfoldings $\BFX_{(n)}\in\mathbb{R}^{I_n\times (I_1\times\cdots\times I_{n-1}\times I_{n+1}\times \cdots \times I_N)}$. The inner product of two same-sized tensors $\BFcalX$ and $\BFcalY$ is the sum of the products of their entries, namely,  $ \left\langle \BFcalX,\BFcalY \right\rangle =\sum_{i_1 }{\cdots \sum_{i_N }{\BFcalX \left(i_1,\dots ,i_N\right) \cdot \BFcalY \left(i_1,\dots ,i_N\right)}}$. Following  the definition of inner product, the Frobenius norm of a tensor $\BFcalX$ is defined as ${\left\|\BFcalX\right\|}_F=\sqrt{\left\langle \BFcalX,\BFcalX\right\rangle }$. 
 The mode-$n$ multiplication of a tensor $\BFcalX$ with a matrix $\BFU$ amounts to the multiplication of all mode-$n$ vector fibers with $\BFU$, namely,$(\BFcalX \times_n\BFU)(i_1,\cdots,i_{n-1},j_n,i_{n+1},\cdots,i_N)=\sum_{i_n}\BFcalX(i_1,\cdots,i_N)\cdot \BFU(j_n,i_n)$.
 Unfolding $\BFcalX$ along the $n$-mode is denoted as $\BFX_{\left(n\right)} \in {\mathbb{R}}^{I_n\mathbf{\times }\left(I_1\times \cdots \times I_{n-1}\times I_{n+1}\times \cdots \times I_N\right)}$, %The column vectors of $\BFA_{\left(n\right)}$ are the $n$-mode vectors of $\BFcalA$. 
in particular, if $\BFcalX=\BFcalC \times_1 \BFU^{(1)} \times_2 \dots \times_N \BFU^{(N)}$ with $\BFcalC \in \mathbb{R}^{P_1\times \cdots \times P_N}$ and $\BFU$, then $\BFX_{(n)}= \BFU^{(n)}\BFC_{(n)}(\BFU^{(N)}\otimes \cdots \otimes \BFU^{(n+1)} \otimes \BFU^{(n-1)}\otimes \cdots \otimes\BFU^{(1)})^{\top}$, where $\otimes$ is the Kronecker product.
\subsection{Related Work} \label{subsec: related work}
In this subsection, three directions of related work to motivate the research in this paper are introduced here. 
\subsubsection{Robust Tensor PCA}   \label{subsubsec: RPCA}
As the tensor extension of the popular robust PCA~\citep{candes2011robust}, the recent proposed RTPCA~\citep{lu2016tensor} aims to recover the low-rank tensor $\BFcalL_0 \in \mathbb{R}^{I_1\times I_2\times I_3}$ and sparse tensor $\BFcalE_0 \in \mathbb{R}^{I_1\times I_2\times I_3}$ from their sum. RTPCA solves the following convex optimization problem
\begin{equation*} \label{eq: RTPCA convex}
\underset{\BFcalL,\BFcalS}{\min} \ \|\BFcalL\|_{TNN}+\lambda  \|\BFcalE\|_1, \
\textnormal{s.t.} \  \BFcalX=\BFcalL+\BFcalE,
\end{equation*}
where $\|\cdot\|_{TNN}$ is their proposed tensor nuclear norm, which is a convex relaxation of the tensor tubal rank. The tensor nuclear norm and tensor tubal rank  are defined based on the t-SVD proposed in~\citep{zhang2014novel}. % In the paper, the analysis shows that $\rho=1/\sqrt{\max (I_1,I_2)I_3}$ guarantees the exact recovery under tensor incoherence condition. 
Following this direction, to further exploit the low-rank structures in tensor data, \citet{liu2018improved} extracted a low-rank component for the core matrix whose entries are from the diagonal elements of the core tensor. Based on this idea, they defined a new tensor nuclear norm and proposed a creative algorithm to deal with RTPCA problems. Other than the work based on the tensor tubal rank, \citet{yang2020low} considered a new model for RTPCA based on  tensor train rank. These methods are applied to background/foreground separation, image/video denosing, etc.
\subsubsection{Tensor Completion}  \label{subsubsec: TC}
Motivated by tensor nuclear norm, \citet{zhou2017tensor} proposed a novel low-rank tensor factorization method for efficiently solving the 3-way tensor completion problem. It aims at exactly recovering a low-rank tensor $\BFcalX \in \mathbb{R}^{I_1\times I_2\times I_3}$  from an incomplete observation $\BFcalF \in \mathbb{R}^{I_1\times I_2\times I_3}$.  Accordingly,
its mathematical model is written as
\begin{equation*} \label{eq: TC convex}
\underset{\BFcalX}{\min} \ \|\BFcalX\|_{TNN}, \
\textnormal{s.t.} \  \BFcalP_{\Omega}(\BFcalX)= \BFcalP_{\Omega}(\BFcalF),
\end{equation*}
where $\Omega$ is the index set of the observed elements, $\BFcalP$ is a linear operator that extracts entries in $\Omega$  and fills
the entries not in $\Omega$ with zeros. In the optimization process, their method only needs to update two smaller tensors, which can be more efficiently conducted than computing t-SVD. Furthermore, they prove that the proposed alternating minimization algorithm can
converge to a Karush–Kuhn–Tucker point.
\subsubsection{Robust Matrix/Tensor Completion} \label{subsubsec: RTC}
Robust matrix completion aims to recover a low-rank matrix $\BFL\in \mathbb{R}^{I_1\times I_2}$ from a subset of noisy entries perturbed by complex noises. \citet{chen2015matrix}  provided a  robust matrix completion model as follows
\begin{equation*} \label{eq: robust MC convex}
\underset{\BFL,\BFE}{\min} \ \|\BFL\|_{*}+\lambda \|\BFE\|_{2,1}, \
\textnormal{s.t.} \  \BFcalP_{\Omega}(\BFX)= \BFcalP_{\Omega}(\BFL+\BFE),
\end{equation*}
where $\|\cdot\|_{*}$ is the matrix nuclear norm, and $\|\cdot\|_{2,1}$ is  the sum of the column $\ell2$ norms of a matrix  and a convex surrogate of its column sparsity. \citet{fan2017fast} proposed novel bilinear factor matrix norm minimization models by defining the double nuclear norm and Frobenius/nuclear hybrid norm. \citet{he2019robust} proposed a novel robust and fast matrix completion method based on the maximum correntropy criterion, which is extended to the tensor version in the work of \citep{he2020robust}.  \citet{li2021matrix} considered column outliers and sparse noise. The $\ell_{2,1}$ norm based objective function makes the recovered matrix keeps a low-rank structure and lets the algorithm robust to column outliers, while the regularization term based on $\ell_1$ norm can alleviate the influence of sparse noise. \citet{huang2021robust} proposed robust tensor ring completion, where the low-rank tensor component is constrained by the weighted sum of nuclear norms of its balanced unfoldings, while the sparse component is regularized by its $\ell_1$ norm.
\subsection{Research Gap Identification} \label{subsec: reseach gap}
In real-world applications, the video data often contains missing pixels due to errors in the acquisition process or manufacturer defects. If the RPCA/RTPC~\citep{candes2011robust,lu2016tensor} in Section~\ref{subsubsec: RPCA} is applied to the video with missing pixels,  the background and foreground are incomplete since RPCA/RTPCA cannot recover missing pixels. If matrix/tensor completion in Section~\ref{subsubsec: TC} is applied, the background and foreground cannot be separated because they cannot decompose the video data. If robust matrix/tensor completion in Section~\ref{subsubsec: RTC} is applied, the foreground recovery is not guaranteed. This is due to the fact that the foreground is represented by the sparse component. Specifically, if a part of the foreground is missing in the data, there is no way to recover the missing part since the optimization problem will set this part to zero because it is modeled by the sparse component. Therefore, this work seeks to address these research gaps  by devising a new smooth robust tensor completion (SRTC) model. The proposed model can be considered as separating background/foreground  together with video recovery by providing a new decomposition methodology with missing pixels.

\section{Proposed Method} \label{sec: proposed method}
% In this work, we consider the problem of Background/Foreground separation problem with missing pixels.   Traditional methods rely on robust matrix/tensor completion or robust PCA with missing pixels, which essentially combines lower-rank and sparse decomposition with missing pixels. These types of method can recover the background pretty well, however,  To improve from the current literature, this work use the spatial-temporal continuity to model the foreground so that the foreground can be recovered as much as possible. 
In Section \ref{subsec: proposed formulation},  the proposed smooth RTC for the background/foreground  separation with missing pixels is presented. Specifically, the low-rankness and spatio-temporal continuity are formulated by the Tucker decomposition and total variation (TV) regularization, respectively. In Section \ref{subsec: PAM algorithm}, an efficient algorithm based on proximal alternating minimization (PAM)~\citep{attouch2010proximal} is designed to solve the proposed model.

\subsection{Proposed Model} \label{subsec: proposed formulation}
Throughout this work,  it is focused on the video that can be represented as a third-order tensor $\BFcalX \coloneqq \{\BFX_1,\cdots,\BFX_T\} \in \mathbb{R}^{H\times W \times T}$, where each matrix $\BFX_t\in \mathbb{R}^{H\times W}$ represents $t$-th image frame $t=1,\dots,T$. $H$, $W$, and $T$ denote the height, width of an image frame, and the number of image frames, respectively. The three modes of  tensor $\BFcalX$ are height, width and time of the video.  In the static background, the image frames keep unchanged along with the time domain. This can be achieved by restricting  $\BFcalL$ to be a low-rank tensor in the time  domain. For the moving objects in the video foreground, they are   continuous spatially and temporally so that they can be represented as a smooth tensor $\BFcalS$.  

 \begin{figure}[!ht]
\centering
\includegraphics[width=1\textwidth]{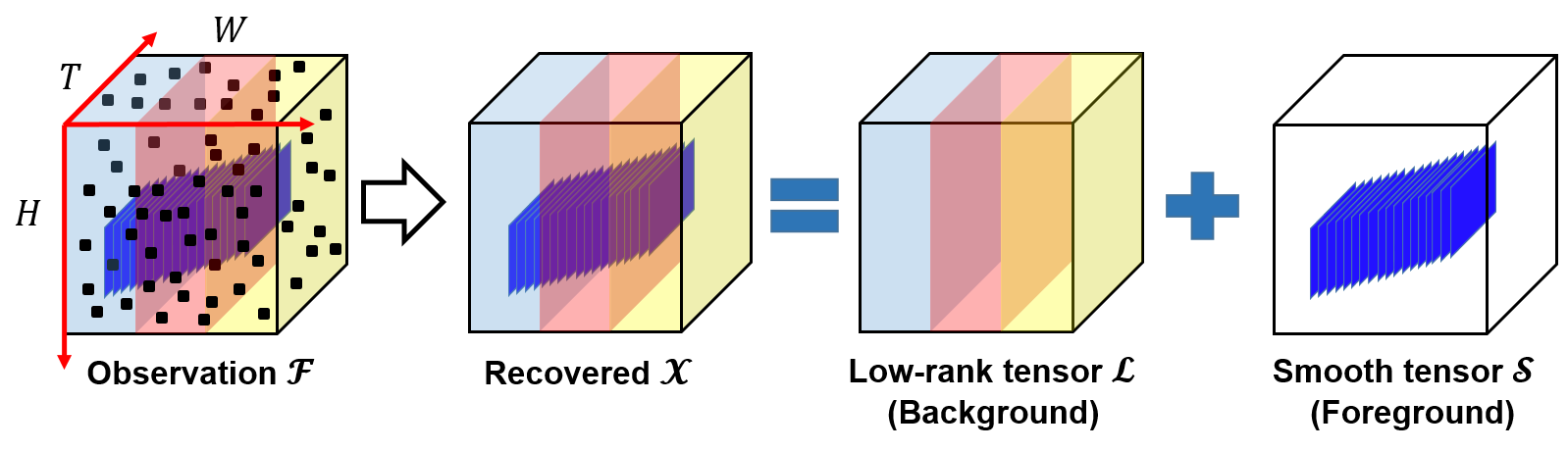}  
\caption{Illustration of the decomposition strategy of a video in the proposed method.}
    \label{fig: Illustration for the decomposition of video volume}
\end{figure} 
As discussed in Section \ref{subsec: reseach gap}, for a video with missing pixels, it is necessary to recover the video data  $\BFcalX$ and decompose it into the low-rank tensor $\BFcalL$ (the static video background), the smooth tensor $\BFcalS$ (the smooth moving objects in the foreground), respectively. In the static background, the image frames keep unchanged along with the time domain. This can be achieved by restricting  $\BFcalL$ to be a low-rank tensor in the time  domain. For the moving objects in the video foreground, they are   continuous spatially and temporally so that they can be represented as a smooth tensor $\BFcalS$.   An illustration of the video decomposition strategy for our proposed method is provided in  Figure~\ref{fig: Illustration for the decomposition of video volume}.  Specifically, it has the following form $\BFcalX = \BFcalL + \BFcalS$ as mentioned in Section~\ref{sec: intro}.

 To model the low-rankness, the static background $\BFcalL$ is approximated by the well-known Tucker decomposition~\citep{kolda2009tensor} with rank-$(r_1,r_2,r_3)$. Specifically, the Tucker decomposition has the following form
\begin{equation} \label{eq: tucker decomposition}
    \BFcalL=\BFcalC\times_1 \BFU^{(1)} \times_2 \BFU^{(2)} \times_3 \BFU^{(3)},
\end{equation} 
where  $\BFU^{(1)}\in \mathbb{R}^{H\times r_1},r_1<H$ and $\BFU^{(2)} \in \mathbb{R}^{W\times r_2},r_2<W $ are orthogonal factor matrices for two spatial domains,  $\BFU^{(3)} \in \mathbb{R}^{T\times r_3},r_3<T$ is the orthogonal factor matrix for the temporal domain, core tensor $\BFcalC \in \mathbb{R}^{r_1 \times r_2 \times r_3}$ interacts these factors.   By formulating the low-rank tensor $\BFcalL$ using Tucker decomposition, it can reconstruct a more accurate video background than the low-rank model based on matrices. Because the Tucker decomposition considers not only the spatial   but also  the temporal correlations  in the video background. 

The smooth tensor $\BFcalS$ (moving objects) is assumed to have the  spatio-temporal continuity property  such that the foreground moves smoothly and coherently in the spatial and temporal directions. In the literature, imposing the  spatio-temporal continuity constraints on moving objects in the foreground is well studied and proven to be effective~\citep{cao2015total,cao2016total}. To measure the sensitivity to change of a quantity function, the derivative is often applied in mathematics. For discrete functions, difference operators are the approximation to derivative. Given a third-order tensor $\BFcalS\in \mathbb{R}^{H\times W\times T}$,  $\BFcalS(x,y,t)$  indicates the intensity of position $(x, y)$ at time $t$, and 
\begin{equation*}
    \begin{aligned}
    \BFcalS_h(x,y,t) & = \BFcalS_h(x+1,y,t) - \BFcalS_h(x,y,t), \\
      \BFcalS_v(x,y,t) & = \BFcalS_v(x,y+1,t) - \BFcalS_v(x,y,t), \\
        \BFcalS_t(x,y,t) & = \BFcalS_t(x,y,t+1) - \BFcalS_t(x,y,t)
\end{aligned}
\end{equation*}
 denote three difference operation results of position $(x, y)$ at time $t$ with periodic boundary conditions along with the horizontal, vertical, and temporal directions, respectively. For simplicity of computation, all the entries of $\BFcalS$ can be stacked into a column vector $\bm{s}=\texttt{vec}(\BFcalS)$, in which $\texttt{vec}(\cdot)$ represents the vectorization operator.  $\BFD_h\bm{s}=\texttt{vec}(\BFcalS_h)$, $\BFD_v\bm{s}=\texttt{vec}(\BFcalS_v)$, and $\BFD_t\bm{s}=\texttt{vec}(\BFcalS_t)$ are used to denote the vectorizations of the three difference operation results, respectively, in which $\BFD_h$, $\BFD_v$, and $\BFD_t\in \mathbb{R}^{HWT  \times HWT }$. Furthermore, $\BFD\bm{s}=[\BFD_h\bm{s}^{\top},\BFD_v\bm{s}^{\top},\BFD_t\bm{s}^{\top}]^{\top}$ is used to represent the concatenated difference operation, in which $\BFD = [\BFD_h^{\top},\BFD_v^{\top},\BFD_t^{\top}]^{\top}\in \mathbb{R}^{3 HWT  \times HWT }$. Note that the $i$-th element in $\BFD_h\bm{s}$, $\BFD_v\bm{s}$, and $\BFD_t\bm{s}$ (namely, $[\BFD_h\bm{s}]_i$, $[\BFD_v\bm{s}]_i$, and $[\BFD_t\bm{s}]_i$) describes the intensity changes of $i$-th point in $\bm{s}$ along with the horizontal, vertical, and temporal directions,  respectively. To quantify the changes of intensity, any vector norm of  $\big[[\BFD_h\bm{s}]_i, [\BFD_v\bm{s}]_i, [\BFD_t\bm{s}]_i\big]$ can be applied. The commonly used vector norm is the $\ell_1$ norm. Specifically,  the anisotropic total variation norm is defined as
\begin{equation} \label{eq: definition of TV1}
\|\BFcalS\|_{TV1}=\sum_{i} (|[\BFD_h\bm{s}]_i|+|[\BFD_v\bm{s}]_i|+|[\BFD_t\bm{s}]_i|),
\end{equation}
% and the isotropic total variation norm as
% \begin{equation} \label{eq: definition of TV2}
% \|\BFcalS\|_{TV2}=\sum_{i} \sqrt{[\BFD_h\bm{s}]_i^2+[\BFD_v\bm{s}]_i^2+[\BFD_t\bm{s}]_i^2},
% \end{equation}
which is the $\ell_1$ norm of $[\BFD_h\bm{s},\BFD_v\bm{s},\BFD_t\bm{s}]^{\top}$. The total variation regularization has been widely used in image and video denoising and restoration~\citep{wang2017hyperspectral, zhang2019hyperspectral} due to its superiority in detecting discontinuous changes in image processing.

By combining the advantages of Tucker decomposition for the low-rank tensor and total variation regularizations for the smooth tensor, the proposed problem has the following formulation
\begin{equation} \label{eq: proposed formulation}
\begin{aligned}
\underset{\BFcalC,\BFU,\BFcalS,\BFcalX}{\min} \quad & \|\BFcalX-\BFcalS-\BFcalC\times_1 \BFU^{(1)} \times_2 \BFU^{(2)} \times_3 \BFU^{(3)}\|_F^2+ \lambda\|\BFcalS\|_{TV1} \\ % \lambda_1\|\BFcalS_{\Omega}\|_1 +
\textnormal{s.t.} \quad  & \BFcalP_{\Omega}(\BFcalX) = \BFcalP_{\Omega}(\BFcalF) , \\
                         &  \BFU^{(n)\top}\BFU^{(n)}=\BFI, n=1,2,3,
\end{aligned}
\end{equation}
where $\BFcalF$ is the partially observed tensor, $\Omega$ is the index set of the observed elements, the first term is the fitting error, and the second term is the regularization term to measure the spatio-temporal continuity of $\BFcalS$. For notational convenience, let the  core tensor $\BFcalC \in \mathbb{R}^{r_1\times r_2\times r_3}$. The set of multilinear subspaces, namely, $\BFU$, is defined as $\{(\BFU^{(1)},\BFU^{(2)},\BFU^{(3)}):\BFU^{(n)\top}\BFU^{(n)}=\BFI, n=1,2,3\}$.  $\lambda>0$ is the coefficient for the one regularization term in \eqref{eq: proposed formulation}. The optimization problem in \eqref{eq: proposed formulation} is based on the decision variables $\{\BFcalC,\BFU,\BFcalS,\BFcalX\}$.  The following proposition states that $\BFcalC$ can be projected out from the original formulation. 

\begin{proposition}\label{lemma: C closed-form}
Suppose $(\BFcalC^*,\BFU^*,\BFS^*,\BFcalX^*)$ is an optimal solution of the proposed formulation \eqref{eq: proposed formulation}, then
\begin{equation*}
    \BFcalC^*=(\BFcalX^* - \BFcalS^*)\times_1 \BFU^{(1)*\top} \times_2 \BFU^{(2)*\top} \times_3 \BFU^{(3)*\top}.
\end{equation*}
\end{proposition} 
\begin{proof}
Suppose the optimal $\BFU^*,\BFS^*,\BFcalX^*$ are given. Then the first-order optimality condition with respect to $\BFcalC$ is
\begin{equation*}%\label{eq: Z_m stationary}
     2\big[-(\BFcalX^* - \BFcalS^*)\times_1 \BFU^{(1)*\top} \times_2 \BFU^{(2)*\top} \times_3 \BFU^{(3)*\top}+\BFcalC\big]=\bm{0}.
\end{equation*} 
We must have
	\[\BFcalC^*=(\BFcalX^* - \BFcalS^*)\times_1 \BFU^{(1)*\top} \times_2 \BFU^{(2)*\top} \times_3 \BFU^{(3)*\top}.\] 
\end{proof}

Proposition~\ref{lemma: C closed-form} implies that there are only three types of decision variables $\BFU,\BFcalS, \BFcalX$ in the proposed formulation  \eqref{eq: proposed formulation}, namely, by projecting out variables $\BFcalC$, which can be simplified as%\todo{I delete the equation}
\begin{equation} \label{eq: proposed formulation-new}
\begin{aligned}
\underset{\BFU,\BFcalS,\BFcalX}{\min} \quad & \hat{F}(\BFU\BFU^\top,\BFcalS,\BFcalX) \\
\textnormal{s.t.} \quad  & \BFcalP_{\Omega}(\BFcalX) = \BFcalP_{\Omega}(\BFcalF) , \\
                         &  \BFU^{(n)\top}\BFU^{(n)}=\BFI, n=1,2,3,
\end{aligned}
\end{equation}
where $\hat{F}(\BFU\BFU^\top,\BFcalS,\BFcalX):= \|\BFcalX-\BFcalS-(\BFcalX-\BFcalS)\times_1 \BFU^{(1)}\BFU^{(1)\top} \times_2 \BFU^{(2)}\BFU^{(2)\top} \times_3 \BFU^{(3)}\BFU^{(3)\top}\|_F^2 + \lambda\|\BFcalS\|_{TV1}$ and $\BFU\BFU^\top$ is defined as $\{(\BFU^{(1)}\BFU^{(1)\top},\BFU^{(2)}\BFU^{(2)\top},\BFU^{(3)}\BFU^{(3)\top}):\BFU^{(n)\top}\BFU^{(n)}=\BFI, n=1,2,3\}$. 

The following proposition shows the representation of $\BFU\BFU^\top$ is unique but not for $\BFU$. 
	\begin{proposition} \label{lemma: projection matrices}
	Given $\BFW^{\top}\BFW=\BFV^{\top}\BFV=\BFI, \BFW,\BFV\in \mathbb{R}^{I\times r}, r< I $, $\|\BFW\BFW^{\top} - \BFV\BFV^{\top}\|_F^2=0$ if and only if there exists an orthogonal matrix $\BFR\in \mathbb{R}^{r\times r}$ such that $\BFW=\BFV\BFR$.
	\end{proposition}
\begin{proof}
    	if: This follows by the straightforward calculation.
	%	$\|\BFW\BFW^{\top} - \BFV\BFV^{\top}\|_F^2=0 \iff  \BFW\BFW^{\top} =\BFV\BFV^{\top}$ 
		
	only if: $\|\BFW\BFW^{\top} - \BFV\BFV^{\top}\|_F^2=0$ implies that $\BFW\BFW^{\top} =\BFV{\BFV}^{\top}$.   Since $\BFW^{\top}\BFW =\BFV^{\top}\BFV=\BFI$, we further have
\[	\BFW\BFW^{\top} \BFV  =\BFV\BFV^{\top} \BFV= \BFV, \]
namely, columns of $\BFV$ are distinct eigenvectors of $\BFW\BFW^{\top}$, where their corresponding eigenvalues are equal to 1. Therefore, $\BFW$ and $\BFV$ have the same column spaces~\citep{strang2016introduction}.
\end{proof}

\subsection{Tensor Proximal Alternating Minimization Algorithm} \label{subsec: PAM algorithm}
Note that \eqref{eq: proposed formulation-new} is a multivariate optimization problem. Alternating minimization  algorithm~\citep{attouch2013convergence} is commonly used to solve multivariate optimization problems due to its simplicity and efficiency. To enhance the theoretical convergence and numerical stability of the alternating minimization algorithm, proximal terms are suggested to add in sub-problems arising from the alternating minimization algorithm, which is called  proximal alternating minimization (PAM) algorithm. In this section, a tensor PAM (tenPAM) algorithm is developed for solving  \eqref{eq: proposed formulation-new}.

Given the solution from $k$-th iterations $(\BFU_k\BFU_k^{\top},\BFcalS^k,\BFcalX^k)$ for the problem \eqref{eq: proposed formulation-new}, then the PAM iterates the following three parts: (1) fix $(\BFcalS^k,\BFcalX^k)$, we solve the optimization problem over $\BFU$ in Section~\ref{subsec: optimization U} to obtain $\BFU_{k+1}$; (2) fix  $(\BFU_{k+1},\BFcalX^k)$, we solve the optimization problem over $\BFcalS$ in Section~\ref{subsec: optimization S} to obtain $\BFcalS^{k+1}$; (3) fix $(\BFU_{k+1},\BFcalS^{k+1})$, we solve the optimization problem over $\BFcalX$ in Section~\ref{subsec: optimization X} to obtain $\BFcalX^{k+1}$.

\subsubsection[alternative title goes here]{Optimization  over $\BFU$} \label{subsec: optimization U}
At iteration $k+1$, assuming that $\BFcalS^k,\BFcalX^k$ is fixed, we solve the below problem to obtain $\BFU^{(n)}_{k+1}$ in a sequential way ($n=1,2,3$):
 \begin{equation} \label{eq: Update U^(n)}
 \begin{aligned}
 & \underset{\BFU^{(n)}}{\text{min}}
 & &\hat{F}(\{\BFU^{(i)}_{k+1}\BFU_{k+1}^{(i)\top}\}_{i<n},\BFU^{(n)}\BFU^{(n)\top},\{\BFU^{(i)}_{k}\BFU_{k}^{(i)\top}\}_{i>n},\BFcalS^k,\BFcalX^k) \\
& & &+  \rho\|\BFU^{(n)}\BFU^{(n)\top}-\BFU^{(n)}_{k}\BFU^{(n)\top}_{k}\|_F^2 \\
& \text{s.t.}
& & \BFU^{(n)\top}\BFU^{(n)}=\BFI,
 \end{aligned}
 \end{equation}
where $\rho\|\BFU^{(n)}\BFU^{(n)\top}-\BFU^{(n)}_{k}\BFU^{(n)\top}_{k}\|_F^2$ is the proximal term, $\rho>0$ is the positive coefficient. Problem \eqref{eq: Update U^(n)} can be equivalently formulated as a standard trace optimization problem as
\begin{equation} \label{eq: Update U^(n) 1}
    \BFU^{(n)}_{k+1} \in \arg\max_{\BFU^{(n)}}\left\{ \mathrm{Tr}(\BFU^{(n)\top} \BFPsi_k^{(n)} \BFU^{(n)}) - \rho\|\BFU^{(n)}\BFU^{(n)\top}-\BFU^{(n)}_{k}\BFU^{(n)\top}_{k}\|_F^2:\BFU^{(n)\top}\BFU^{(n)}=\BFI\right\},
\end{equation}
where  $\BFPsi_{k}^{(n)}=(\BFX_{(n)}^k -\BFS_{(n)}^k) \cdot \BFU_{\BFPsi^{(n)}_k}\cdot {\BFU_{\BFPsi^{(n)}_k}}^{\top}   \cdot  (\BFX_{(n)}^k -\BFS_{(n)}^k)^{\top}$ and $\BFU_{\BFPsi^{(n)}_k}=\BFU^{(3)}_k\otimes \cdots \otimes \BFU^{(n+1)}_k \otimes \BFU^{(n-1)}_{k+1}\otimes \cdots \otimes\BFU^{(1)}_{k+1}$.
The following lemma shows that we can absorb the penalty term so that problem \eqref{eq: Update U^(n) 1} has a closed-form optimal solution by redefining matrices $\{\BFPsi_k^{(n)}\}_{n\in [3]}$ for all $k$.
\begin{lemma}\label{lemma: Equivalent Trace Optimization}
Problem \eqref{eq: Update U^(n) 1} is equivalent to
\begin{equation} \label{eq: Update U^(n) final}
    \BFU^{(n)}_{k+1}\in \arg\max_{\BFU^{(n)}} \left\{\mathrm{Tr}(\BFU^{(n)\top} \BFPhi_k^{(n)} \BFU^{(n)}):\BFU^{(n)\top}\BFU^{(n)}=\BFI\right\},
\end{equation}
where $\BFPhi^{(n)}_{k}= \BFPsi_k^{(n)} -2\rho (\BFI - \BFU^{(n)}_{k} {\BFU^{(n)\top}_{k}})$. In addition, problem~\eqref{eq: Update U^(n) final} is a standard eigen-decomposition problem, which has a closed-form solution. % *****.% and $\BFI_{P_n} - \BFU^{(n)}_{t} {\BFU^{(n)\top}_{t}}$ is tangent space of $ \BFU^{(n)}_{t}$
\end{lemma}
\begin{proof}
On the one hand, we have
\begin{equation} \label{eq: f-norm to trace}
    \begin{aligned}
& \|\BFU^{(n)} \BFU^{(n)\top} - \BFU^{(n)}_{k} \BFU^{(n)\top}_{k}\|_F^2\\
 = & \mathrm{Tr}\Big((\BFU^{(n)} \BFU^{(n)\top} - \BFU^{(n)}_{k} \BFU^{(n)\top}_{k})(\BFU^{(n)} \BFU^{(n)\top} - \BFU^{(n)}_{k} \BFU^{(n)\top}_{k})^{\top}\Big)  \\
%  = & \mathrm{Tr}(\BFU^{(n)}\BFU^{(n)\top} - 2\BFU^{(n)}_{t} \BFU^{(n)\top}_{t}\BFU^{(n)}\BFU^{(n)\top} + \BFU^{(n)}_{t}\BFU^{(n)\top}_{t}) \\
 = & 2\Big(r_n - \mathrm{Tr}(\BFU^{(n)}_{k}\BFU^{(n)\top}_{k}\BFU^{(n)}\BFU^{(n)\top})\Big).
\end{aligned}
\end{equation} 
On the other hand, we have
\begin{equation} \label{eq: trace to trace}
    	\begin{aligned}
& \mathrm{Tr}\Big({\BFU^{(n)\top}}(\BFI - \BFU^{(n)}_{k} {\BFU^{(n)\top}_{k}})\BFU^{(n)}\Big) \\
%  = &\mathrm{Tr}((\BFI_{I_n} - \BFU^{(n)}_{t} {\BFU^{(n)\top}_{t}})\BFU^{(n)}\BFU^{(n)\top})  \\
 =& \mathrm{Tr}\Big(\BFU^{(n)}\BFU^{(n)\top} - \BFU^{(n)}_{k} \BFU^{(n)\top}_{k}\BFU^{(n)}\BFU^{(n)\top}\Big) \\
 =& r_n - \mathrm{Tr}(\BFU^{(n)}_{k} \BFU^{(n)\top}_{k}\BFU^{(n)}\BFU^{(n)\top}).
\end{aligned}
\end{equation} According to the above two equations~\eqref{eq: f-norm to trace} and~\eqref{eq: trace to trace}, the following can be derived
\begin{equation}\label{eq: f-norm and trace}
    \|\BFU^{(n)} \BFU^{(n)\top} - \BFU^{(n)}_{k} \BFU^{(n)\top}_{k}\|_F^2=
2\mathrm{Tr}\Big({\BFU^{(n)\top}}(\BFI - \BFU^{(n)}_{k} {\BFU^{(n)\top}_{k}})\BFU^{(n)}\Big).
\end{equation}
Thus,
\begin{align*}
   & \mathrm{Tr}(\BFU^{(n)\top} \BFPsi_k^{(n)} \BFU^{(n)}) - \rho\|\BFU^{(n)}\BFU^{(n)\top}-\BFU^{(n)}_{k}\BFU^{(n)\top}_{k}\|_F^2 \\
    = &  \mathrm{Tr}(\BFU^{(n)\top} \BFPsi_k^{(n)} \BFU^{(n)}) - 2\rho \mathrm{Tr}\Big({\BFU^{(n)\top}}(\BFI- \BFU^{(n)}_{k} {\BFU^{(n)\top}_{k}})\BFU^{(n)}\Big) \\
    = & \mathrm{Tr}(\BFU^{(n)\top}\BFPhi^{(n)}_{k}\BFU^{(n)}),
\end{align*}
where the first equality is due to~\eqref{eq: f-norm and trace} and  $\BFPhi^{(n)}_{k}=\BFPsi_k^{(n)} -2\rho(\BFI - \BFU^{(n)}_{k} {\BFU^{(n)\top}_{k}})$.
\end{proof}
\begin{remark}
There are two reasons to use $\rho\|\BFU^{(n)}\BFU^{(n)\top}-\BFU^{(n)}_{k}\BFU^{(n)\top}_{k}\|_F^2$ as the proximal term instead of $\rho\|\BFU^{(n)}-\BFU^{(n)}_{k}\|_F^2$. 
\begin{enumerate}
  \item  Tucker decomposition of a tensor is not unique due to the possible orthogonal transformations of basis matrices. This is known as the  rotation indeterminacy of tensors. For example, the solutions $\BFU^{(n)}$  and $\BFU^{(n)}\BFA$ can achieve the same performance, where $\BFA$ is an orthogonal matrix. But the $\BFU^{(n)}\BFU^{(n)\top}$ is unique for Tucker decomposition as shown in Proposition~\ref{lemma: projection matrices}.
  \item If one uses $\rho\|\BFU^{(n)}-\BFU^{(n)}_{k}\|_F^2$ in \eqref{eq: Update U^(n)}, then it becomes a very hard optimization problem to solve. All current algorithms~\citep{chen2020proximal,wang2020riemannian} can guarantee the convergence to a critical point. However, we need to solve the optimization problem~\eqref{eq: Update U^(n)} to optimal so that all further theoretical results will hold.
\end{enumerate}
\end{remark}

\subsubsection[alternative title goes here]{Optimization  over $\BFcalS$}  \label{subsec: optimization S}
In this case, we  fix all $(\BFU_{k+1},\BFcalX^k)$. Here, we will show the step to update $\BFcalS^{k+1}$
\begin{equation} \label{eq: Update S}
\begin{aligned}
     \BFcalS^{k+1}\in \underset{\BFcalS}{\arg\min} \Big\{ & \|\BFcalX^k - \BFcalS - (\BFcalX^k - \BFcalS^k) \times_1 \BFU^{(1)}_{k+1}\BFU_{k+1}^{(1)\top} \times_2 \BFU^{(2)}_{k+1}\BFU_{k+1}^{(2)\top}  \times_3 \BFU^{(3)}_{k+1}\BFU_{k+1}^{(3)\top}\|_F^2\\
     & + \lambda\|\BFcalS\|_{TV1}+\rho\|\BFcalS -\BFcalS^k\|_F^2\Big\},
\end{aligned}
\end{equation}
where $\rho\|\BFcalS -\BFcalS^k\|_F^2$ is the proximal term. This problem \eqref{eq: Update S} is a  strictly convex
optimization problem with the strongly convex objective function, which possesses global and unique minimizers. There are many efficient solvers for finding global minimizers of~\eqref{eq: Update S}; see, e.g., the Bregman methods~\citep{goldstein2009split}, proximal splitting methods~\citep{combettes2011proximal}, and alternating direction of multiplier methods (ADMM)~\citep{hong2017linear}. We choose ADMM as a solver for the minimization problems~\eqref{eq: Update S}, since the convergence of ADMM for~\eqref{eq: Update S} is theoretically guaranteed~\citep{hong2017linear}.

The optimization problem~\eqref{eq: Update S} is equivalent to the following equality constrained  one:
 \begin{equation} \label{eq: Update S New1}
 \begin{aligned}
 & \underset{\BFcalS,\bm{f}}{\text{min}}
 & & \|\BFcalS - \BFcalE_{k+1}\|_F^2 + \frac{\lambda}{1+\rho}\|\bm{f}\|_1 \\
& \text{s.t.}
& &  \bm{f} = \BFD\texttt{vec}(\BFcalS),
 \end{aligned}
 \end{equation}
where $\BFcalE_{k+1} = \frac{\BFcalX^k - (\BFcalX^k - \BFcalS^k) \times_1 \BFU^{(1)}_{k+1}\BFU_{k+1}^{(1)\top} \times_2 \BFU^{(2)}_{k+1}\BFU_{k+1}^{(2)\top}  \times_3 \BFU^{(3)}_{k+1}\BFU_{k+1}^{(3)\top}+\rho\BFcalS^k}{1+\rho}$. Note that the objective function in \eqref{eq: Update S New1} is the summation of two single variable functions with $\BFcalS$ and $\bm{f}$ as their individual variables. Thus, ADMM is applicable. The augmented Lagrangian function of problem~\eqref{eq: Update S New1} can be written as:
\begin{equation}\label{eq: Lagrangian of Update S}
         L_A(\BFcalS,\bm{f}) =  \|\BFcalS - \BFcalE_{k+1}\|_F^2 + \frac{\lambda}{1+\rho}\|\bm{f}\|_1 -\langle \BFlambda^{\bm{f}}, \bm{f} - \BFD\texttt{vec}(\BFcalS)\rangle +\frac{\beta^{\bm{f}}}{2}\|\bm{f} - \BFD\texttt{vec}(\BFcalS)\|_2^2,
\end{equation}
where  $\BFlambda^{\bm{f}}$ is the Lagrange multiplier vector, $\beta^{\bm{f}}$ is positive scalars. The optimization problem of $L_A$ in~\eqref{eq: Lagrangian of Update S} with respect to each variable can be solved by the following sub-problems:

1) $\bm{f}$ sub-problem: the sub-problem of $L_A$ with respect to $\bm{f}$ can be rewritten as 
\begin{equation*}
    \underset{\bm{f}}{\min} \ \frac{\lambda}{1+\rho}\|\bm{f}\|_1 +\frac{\beta^{\bm{f}}}{2}\|\bm{f} - (\BFD\texttt{vec}(\BFcalS)+\frac{\BFlambda^{\bm{f}}}{\beta^{\bm{f}}})\|_2^2,
\end{equation*}
 where the well-known soft-thresholding operator~\citep{donoho1995noising} can be applied to solve this sub-problem as follows
\begin{equation} \label{eq: ADMM f step}    
\bm{f}=\textnormal{soft}(\BFD\texttt{vec}(\BFcalS)+\frac{\BFlambda^{\bm{f}}}{\beta^{\bm{f}}},\frac{\lambda}{(1+\rho)\beta^{\bm{f}}}),
\end{equation}
where the soft-thresholding operator $\textnormal{soft}(\BFcalA,\tau) = \textnormal{sign}(\BFcalA)\cdot \max(|\BFcalA|-\tau,0)$ is performed element-wisely.

2) $\BFcalS$ sub-problem:  the sub-problem of $L_A$ with respect to $\BFcalS$ can be solved by the following linear system:
\begin{equation*}
    (2\BFI+\beta^{\bm{f}}\BFD^*\BFD)\texttt{vec}(\BFcalS) = 2\texttt{vec}(\BFcalE_{k+1})+\BFD^*(\beta^{\bm{f}}\bm{f}-\BFlambda^{\bm{f}}),
\end{equation*}
 where $\BFD^*$ indicates the adjoint of $\BFD$. Let $\BFcalC=\texttt{ten}\big(2\texttt{vec}(\BFcalE_{k+1})+\BFD^*(\beta^{\bm{f}}\bm{f}-\BFlambda^{\bm{f}})\big)$. Thanks to the block-circulant
structure of the matrix corresponding to the operator $\BFD^*\BFD$, it can be diagonalized by the 3D FFT matrix. Therefore, $\BFcalS$ can
be fast computed by
\begin{equation} \label{eq: ADMM S step}
     \texttt{ifftn}\Big(\frac{\texttt{fftn}(\BFcalC)}{2\cdot\bm{\mathrm{1}}+\beta^{\bm{f}}( |\texttt{fftn}(\BFD_h)|^2+|\texttt{fftn}(\BFD_v)|^2+|\texttt{fftn}(\BFD_t)|^2)}\Big),
\end{equation}
where $\texttt{fftn}(\cdot)$ and $\texttt{ifftn}(\cdot)$  indicate fast 3D Fourier transform and its inverse transform, respectively.  Note that the denominator in the equation can be pre-calculated outside the main loop, avoiding the extra computational cost.

3) Updating Multipliers: According to the ADMM, the multipliers associated with $L_A$ are updated by the following formulas:
\begin{equation} \label{eq: ADMM lambda f update}
   \BFlambda^{\bm{f}}    \leftarrow  \BFlambda^{\bm{f}} - \gamma\beta^{\bm{f}}\big(\bm{f} - \BFD\texttt{vec}(\BFcalS)\big)
\end{equation}
where $\gamma>0$ is a parameter associated with convergence rate, and the penalty parameters $\beta^{\bm{f}}$  follow an adaptive updating scheme  as suggested in~\citep{chan2011augmented}. Take 
$\beta^{\bm{f}}$ as an example,
\begin{equation} \label{eq: ADMM beta f update}
  \beta^{\bm{f}} = \left\{
  \begin{aligned}
     c_1 \beta^{\bm{f}}, &\textnormal{  if  } \textnormal{Err}(\bm{f}^{\textnormal{iter}+1}) \geq c_2\textnormal{Err}(\bm{f}^\textnormal{iter})\\
     \beta^{\bm{f}}, &\textnormal{  otherwise.}
  \end{aligned}
  \right.
\end{equation}
where $\textnormal{Err}(\bm{f}^\textnormal{iter}) =\|\bm{f}^\textnormal{iter} - \BFD\texttt{vec}(\BFcalS)\|_2$. As suggested in~\citep{cao2016total}, $\gamma=1.1$, and $c_1$, $c_2$ can be taken as 1.15 and 0.95, respectively.

Repeating the iteration process~\eqref{eq: ADMM f step},~\eqref{eq: ADMM S step},  and~\eqref{eq: ADMM beta f update}  sufficiently many
times, the so-obtained $\BFcalS$ is taken as an iterative solution to $\BFcalS^{k+1}$. Note that among these many rounds of iteratively repeating~\eqref{eq: ADMM f step},~\eqref{eq: ADMM S step}, and~\eqref{eq: ADMM beta f update}, the first round requires initial guesses of $\BFcalS$ and $\BFlambda^{\bm{f}}$, for which we set $\BFcalS^k$ as the initial guess of $\BFcalS$ and
$\bm{0}$ (zero vector) as the initial guess of $\BFlambda^{\bm{f}}$.
\begin{algorithm}[!ht]
		\caption{ADMM algorithm for solving~\eqref{eq: Update S New1}} \label{alg: ADMM algorithm for updating S}
		 \textbf{Input:}  $(\BFU_{k+1},\BFcalX^k)$, $\rho=0.001$; The algorithm parameter:  $\lambda$. \\
		 \textbf{Initialization:}   $\BFcalS=\BFcalS^{k}$; $\beta^{\bm{f}}$ is initialized by $\frac{1e^{+1}}{\textnormal{mean}(\BFcalX)}$; other variables are initialized by $\bm{0}$.                      
		\begin{algorithmic}[1]
		\While {$\frac{\|\BFcalS_{t}-\BFcalS_{t-1}\|_F}{\max\{1,\|\BFcalS_{t-1}\|_F\}}>10^{-6}$ \& $\textnormal{iter}\leq 100$}
        \State Updating $\bm{f}$ via \eqref{eq: ADMM f step};
        \State Updating $\BFcalS$ via \eqref{eq: ADMM S step};
        \State Updating multipliers and the related parameters via \eqref{eq: ADMM lambda f update} and \eqref{eq: ADMM beta f update};  
        \State $\textnormal{iter}= \textnormal{iter}+1$;
       \EndWhile\\
       \textbf{Output:} $\BFcalS^{k+1} = \BFcalS$
		\end{algorithmic} 
	\end{algorithm} 
\subsubsection[alternative title goes here]{Optimization  over $\BFcalX$}  \label{subsec: optimization X}
In this case, we  fix all $(\BFU_{k+1},\BFcalS^{k+1})$. Here, we will show the step to update $\BFcalX^{k+1}$
\begin{equation} \label{eq: Update X}
\begin{aligned}
     \BFcalX^{k+1}\in \underset{\BFcalX}{\arg\min} \Big\{ & \|\BFcalX - \BFcalS^{k+1} - (\BFcalX^k - \BFcalS^{k+1}) \times_1 \BFU^{(1)}_{k+1}\BFU_{k+1}^{(1)\top} \times_2 \BFU^{(2)}_{k+1}\BFU_{k+1}^{(2)\top}  \times_3 \BFU^{(3)}_{k+1}\BFU_{k+1}^{(3)\top}\|_F^2\\
     & +\rho\| \BFcalX- \BFcalX^{k}\|_F^2:\BFcalP_{\Omega}(\BFcalX) = \BFcalP_{\Omega}(\BFcalF)\Big\},
\end{aligned}
\end{equation}
 Now let $\BFcalF_{k+1} \coloneqq \BFcalS^{k+1} + (\BFcalX^k - \BFcalS^{k+1}) \times_1 \BFU^{(1)}_{k+1}\BFU_{k+1}^{(1)\top} \times_2 \BFU^{(2)}_{k+1}\BFU_{k+1}^{(2)\top}  \times_3 \BFU^{(3)}_{k+1}\BFU_{k+1}^{(3)\top}$. This problem has closed-form solution as follows
 
\begin{equation}
    \BFcalX^{k+1} = \BFcalP_{\Bar{\Omega}}(\frac{\BFcalF_{k+1}+\rho\BFcalX^{k}}{1+\rho})+ \BFcalP_{\Omega}(\BFcalF),
\end{equation}
where $\Bar{\Omega}$ is the complement set of $\Omega$. Now the detailed implementation of tenPAM algorithm is summarized in Algorithm~\ref{alg: PAM}. 
\begin{algorithm}[!ht]
		\caption{tenPAM Algorithm}\label{alg: PAM}
		\textbf{Input:} Partial observed tensor $\BFcalF_{\Omega}$, $\lambda>0$ \\
		\textbf{Initialization:} $r_1=\textnormal{ceil}(0.8\times H)$, $r_2=\textnormal{ceil}(0.8\times W)$, $r_3=1$; $\gamma = 1.1$; $c_1=1.15,c_2=0.95$; $\BFcalX^0=\BFcalP(\BFcalF)$; $\{\BFU^{(n)}_0\}_{n\in [3]}$ is initialized by $(r_1,r_2,r_3)$-Tucker decomposition of  $\BFcalX^0$; $\BFcalS=\BFcalX^0-\BFcalL^0$; $\rho=0.001$; Other variables are initialized by $\bm{0}$.                     
		\begin{algorithmic}[1]
				\While {$\frac{\|\BFcalP_{\Omega}(\BFcalX^{k}-\BFcalF)\|_F}{\max\{1,\|\BFcalF\|_F\}}>10^{-6}$ \& $\textnormal{iter}\leq 50$}
			\For{$n= 1$ to $3$}
			\State Get $\BFU^{(n)}_{k+1}$ by solving problem \eqref{eq: Update U^(n) final}
			\EndFor
            \State Get $\BFcalS^{k+1}$ by solving problem \eqref{eq: Update S}  using Algorithm~\ref{alg: ADMM algorithm for updating S}
         	\State Get $\BFcalX^{k+1}$ by solving problem \eqref{eq: Update X}
			\State $k=k+1$ 
		\EndWhile\\
			\textbf{Output:} $\{\BFU_k\BFU_k^{\top},\BFcalS^k,\BFcalX^k\}_{k\geq 0}$
		\end{algorithmic}
	\end{algorithm}
\subsection{Implementation Details} \label{subsec: implementation details} 
In the SRTC model~\eqref{eq: proposed formulation}, there exist four parameters, namely, $r_1$, $r_2$, $r_3$, and $\lambda$, respectively, where $r_1$ and $r_2$ control the complexity of spatial redundancy, $r_3$ controls the complexity of temporal redundancy, and $\lambda$  handles a trade-off between noise and foreground modeling. In all experiments,  $r_1$ and $r_2$ are set to values of $\textnormal{ceil}(0.80\times H)$ and $\textnormal{ceil}(0.80\times W)$, respectively,  where $\textnormal{ceil}(\cdot)$ is the operator to round the element to the nearest integer greater than or equal to that element. By doing so, the accumulation energy ratio of top normalized singular values (AccEgyR) ) attains a ratio over 0.9 for various natural images,  as reported in~\citep{cao2016total}. For $r_3$, it takes the value 1 for all experiments so that each image frame in $\BFcalL$ is the same~\citep{sobral2016lrslibrary}. In terms of $\lambda$, it needs to be carefully tuned based on the data. Specifically, $\lambda$ is taken in the range $[0.2,1]$. %Empirically,  the proposed algorithm can achieve satisfactory performance with any  $\lambda\in [0.2,1]$. Therefore, practitioners can apply the proposed method without extensive parameter tuning.

\section{Convergence Analysis of tenPAM Algorithm} \label{sec: convergence analysis}
In this section, we will prove the global convergence of tenPAM algorithm~\ref{alg: PAM}. For notional convenience in our analysis, set $I_1=H$, $I_2=W$, $I_3=T$, and $N=3$. We first note that the proposed formulation~\eqref{eq: proposed formulation-new} can be reformulated as an equivalent unconstrained optimization problem as follows:
\begin{equation}\label{eq: proposed formulation unconstrained}
\min_{\BFU\BFU^\top,\BFcalS,\BFcalX} G(\BFU\BFU^\top,\BFcalS,\BFcalX)=\hat{F}(\BFU\BFU^\top,\BFcalS,\BFcalX)+\delta_{S}(\BFcalX)+ \sum_{n \in [3]} \delta_{S_n}(\BFU^{(n)}\BFU^{(n)\top}),
\end{equation}
where set $S_n \coloneqq \{\BFX \in \mathbb{R}^{I_n \times I_n}:\mathrm{rank}(\BFX)=r_n, \BFX=\BFX^{\top},\mathrm{eigenvalue} \ \mathrm{of} \ \BFX \mathrm{\ is \ either \ 1 \ or \ 0 } \}$, and $S\coloneqq \{\BFcalX \in \mathbb{R}^{H\times W \times T}:\BFcalP_{\Omega}(\BFcalX) = \BFcalP_{\Omega}(\BFcalF)\}$. For a given set $A$, its characteristic function is defined as
\begin{equation*}
\delta_{A}(\bm x)=\left\{
\begin{array}{lcr}
0 & &  \bm x \in A\\
+\infty & & \textrm{otherwise }
\end{array} \right.,
\end{equation*}
which is a proper and lower semicontinuous (PLSC) function. Therefore, $G(\cdot)$ is PLSC function. $\partial G(\BFU\BFU^\top,\BFcalS,\BFcalX)$ is called Subdifferential of $G$ at $\{\BFU\BFU^\top,\BFcalS,\BFcalX\}$, which has the following definition. 
\begin{definition}[\normalfont\textit{Subdifferentials}~\citep{attouch2009convergence, attouch2010proximal}]\label{def: Subdifferentials}~Assume that $f:~\mathbb{R}^d \to (-\infty,+\infty)$ is a proper and lower semicontinuous function.
	\begin{enumerate}%[label=(\roman*)]
	\item The domain of $f$ is defined and denoted by $\textnormal{dom}f\coloneqq \{\bm{x}\in\mathbb{R}^n:f(\bm{x})<+\infty\}$
	\item For a given $\bm{x}\in \textnormal{dom}f$, the Fréchet subdifferential of $f$ at $\bm{x}$, written $\hat{\partial}f(\bm{x})$, is the set of all vectors $\bm{u}\in  \mathbb{R}^d $ that satisfy
\[\lim\limits_{\bm{y} \neq \bm{x}} \inf_{\bm{y} \to \bm{x}} \frac{f(\bm{y})-f(\bm{x}) - \langle \bm{u} , \bm{y}-\bm{x}\rangle}{\|\bm{y}-\bm{x}\|}\geq 0.\]
\item The limiting-subdifferential, or simply the subdifferential, of $f$ at $\bm{x}$, written $\partial f(\bm{x})$ is defined through the following closure process
\[\partial f(\bm{x}):=\{\bm{u}\in  \mathbb{R}^d: \exists \bm{x}^k \to \bm{x},f(\bm{x}^k) \to f(\bm{x}) \textnormal{ and } \bm{u}^k \in \hat{\partial}f(\bm{x}^k) \to \bm{u} \textnormal{ as } k\to \infty \}.\]
	\end{enumerate}
\end{definition}

Before introducing our key results, the following proposition is needed to build our main results.
\begin{proposition} \label{proposition: S}
The following optimization problem has a closed-form solution
\begin{equation}  \label{prop: optimal}
       \BFcalS \in \underset{\BFcalY}{\arg\min}\ \|\BFcalX - \BFcalS - (\BFcalX - \textcolor{red}{\BFcalY}) \times_1 \BFU^{(1)}\BFU^{(1)\top} \times_2 \dots \times_N \BFU^{(N)}\BFU^{(N)\top}\|_F^2. 
\end{equation}
\end{proposition}
\begin{proof}
   The objective function in \eqref{prop: optimal} is a convex function of $\BFcalY$. Therefore, based on the first-order optimality condition, we have
\begin{equation*}
    \begin{aligned}
    \bm{0}=&   2(\BFcalY-\BFcalX) \times_1 \BFU^{(1)}\BFU^{(1)\top} \times_2 \dots \times_N \BFU^{(N)}\BFU^{(N)\top} \\ 
    -&  2(\BFcalS-\BFcalX) \times_1 \BFU^{(1)}\BFU^{(1)\top} \times_2 \dots \times_N \BFU^{(N)}\BFU^{(N)\top}\\
    =&2(\BFcalY - \BFcalS)\times_1 \BFU^{(1)}\BFU^{(1)\top} \times_2 \dots \times_N \BFU^{(N)}\BFU^{(N)\top}.
    \end{aligned}
\end{equation*}
Thus, the statement in~\eqref{prop: optimal} is valid.
\end{proof}

Now our first main lemma about sufficient decrease property of the iterative sequence  $\{\BFU_k\BFU_k^{\top},\BFcalS^k,\BFcalX^k\}_{k\geq 0}$ from Algorithm~\ref{alg: PAM} is ready to be introduced.
\begin{lemma}[Sufficient decrease property]\label{lemma: Sufficient decrease property}~Given that $0<\rho<\infty$, $\{\BFU_k\BFU_k^{\top},\BFcalS^k,\BFcalX^k\}_{k\geq 0}$ is the sequence generated from the proposed Algorithm~\ref{alg: PAM}, then the sequence satisfies
\begin{equation} \label{eq: sufficient decrease}
\begin{aligned}
    & \rho(\|\BFU_{k+1}\BFU_{k+1}^{\top}-\BFU_k\BFU_k^{\top}\|_F^2+\|\BFcalS^{k+1} -\BFcalS^k\|_F^2+\|\BFcalX^{k+1} -\BFcalX^k\|_F^2)\\
    \leq & G(\BFU_k\BFU_k^{\top},\BFcalS^k,\BFcalX^k) -  G(\BFU_{k+1}\BFU_{k+1}^{\top},\BFcalS^{k+1},\BFcalX^{k+1}).  
\end{aligned}
\end{equation}
\end{lemma}
 \begin{proof}%[Proof of Lemma~\ref{lemma: Sufficient decrease property}]
    At $(k+1)$-th iteration, due to that we can get the optimal solution for \eqref{eq: Update U^(n) final} for all $n=1,2,3$, thus 
\begin{equation} \label{ineq: sufficient 0}
\begin{aligned}
       &  G(\{\BFU^{(i)}_{k+1}\BFU_{k+1}^{(i)\top}\}_{i\leq n},\{\BFU^{(i)}_{k}\BFU_{k}^{(i)\top}\}_{i>n},\BFcalS^k,\BFcalX^k) \\
        \leq   &   G(\{\BFU^{(i)}_{k+1}\BFU_{k+1}^{(i)\top}\}_{i< n},\{\BFU^{(i)}_{k}\BFU_{k}^{(i)\top}\}_{i\geq n},\BFcalS^k,\BFcalX^k) - \rho \|\BFU_k^{(n)}\BFU_k^{(n)\top}-\BFU_{k+1}^{(n)}\BFU_{k+1}^{(n)\top}\|_F^2.
\end{aligned}
\end{equation}
 The above inequality~\eqref{ineq: sufficient 0} implies the following
\begin{equation} \label{ineq: sufficient 1}
    \begin{aligned}
           & G(\BFU_{k+1}\BFU_{k+1}^{\top},\BFcalS^k,\BFcalX^k) -     G(\BFU_k\BFU_k^{\top},\BFcalS^k,\BFcalX^k) \\
           = & \sum_{n=1}^3  \Big(G(\{\BFU^{(i)}_{k+1}\BFU_{k+1}^{(i)\top}\}_{i\leq n},\{\BFU^{(i)}_{k}\BFU_{k}^{(i)\top}\}_{i>n},\BFcalS^k,\BFcalX^k)\\
           & - G(\{\BFU^{(i)}_{k+1}\BFU_{k+1}^{(i)\top}\}_{i< n},\{\BFU^{(i)}_{k}\BFU_{k}^{(i)\top}\}_{i\geq n},\BFcalS^k,\BFcalX^k)\Big) \\
           \leq & -\sum_{n=1}^3\rho \|\BFU_k^{(n)}\BFU_k^{(n)\top}-\BFU_{k+1}^{(n)}\BFU_{k+1}^{(n)\top}\|_F^2. 
    \end{aligned}
\end{equation}
The fact that $\BFcalS^{k+1}$ is the optimal solution for problem \eqref{eq: Update S} shows,
\begin{equation} \label{ineq: sufficient 2}
    \begin{aligned}
          & \|\BFcalX^k - \BFcalS^{k+1} - (\BFcalX^k - \BFcalS^k) \times_1 \BFU^{(1)}_{k+1}\BFU_{k+1}^{(1)\top} \times_2 \BFU^{(2)}_{k+1}\BFU_{k+1}^{(2)\top}  \times_3 \BFU^{(3)}_{k+1}\BFU_{k+1}^{(3)\top}\|_F^2\\
     & + \lambda\|\BFcalS^{k+1}\|_{TV1}+\rho\|\BFcalS^{k+1} -\BFcalS^k\|_F^2  \leq G(\BFU_{k+1}\BFU_{k+1}^{\top},\BFcalS^k,\BFcalX^k).
    \end{aligned}
\end{equation}
 Based on  Proposition~\ref{proposition: S}, the following holds
\begin{equation}  \label{ineq: sufficient 3}
    \begin{aligned}
    & G(\BFU_{k+1}\BFU_{k+1}^{\top},\BFcalS^{k+1},\BFcalX^k)\\
    =  & \underset{\BFcalS}{\min} \ \Big(\|\BFcalX^k - \BFcalS^{k+1} - (\BFcalX^k - \textcolor{red}{\BFcalS}) \times_1 \BFU^{(1)}_{k+1}\BFU_{k+1}^{(1)\top} \times_2 \BFU^{(2)}_{k+1}\BFU_{k+1}^{(2)\top}  \times_3 \BFU^{(3)}_{k+1}\BFU_{k+1}^{(3)\top}\|_F^2 \\
     & + \lambda\|\BFcalS^{k+1}\|_{TV1} \Big).
    \end{aligned}
\end{equation}
Combine \eqref{ineq: sufficient 2} and \eqref{ineq: sufficient 3}, we have
\begin{equation}\label{ineq: sufficient 4}
    G(\BFU_{k+1}\BFU_{k+1}^{\top},\BFcalS^{k+1},\BFcalX^k) \leq G(\BFU_{k+1}\BFU_{k+1}^{\top},\BFcalS^k,\BFcalX^k) - \rho\|\BFcalS^{k+1} -\BFcalS^k\|_F^2.
\end{equation}
 	   Since $\BFcalX^{k+1}$ is the optimal solution for \eqref{eq: Update X}, the following holds
 	   \begin{equation} \label{ineq: sufficient 5}
    \begin{aligned}
          & \|\BFcalX^{k+1} - \BFcalS^{k+1} - (\BFcalX^k - \BFcalS^{k+1}) \times_1 \BFU^{(1)}_{k+1}\BFU_{k+1}^{(1)\top} \times_2 \BFU^{(2)}_{k+1}\BFU_{k+1}^{(2)\top}  \times_3 \BFU^{(3)}_{k+1}\BFU_{k+1}^{(3)\top}\|_F^2\\
     & + \lambda\|\BFcalS^{k+1}\|_{TV1}+\rho\|\BFcalX^{k+1} -\BFcalX^k\|_F^2  \leq G(\BFU_{k+1}\BFU_{k+1}^{\top},\BFcalS^{k+1},\BFcalX^k).
    \end{aligned}
\end{equation}
Again, the following can be obtained through the Proposition~\ref{proposition: S},
\begin{equation}  \label{ineq: sufficient 6}
    \begin{aligned}
    & G(\BFU_{k+1}\BFU_{k+1}^{\top},\BFcalS^{k+1},\BFcalX^{k+1})\\
    =  & \underset{\BFcalX}{\min} \ \Big(\|\BFcalX^{k+1} - \BFcalS^{k+1} - ( \textcolor{red}{\BFcalX} -\BFcalS^{k+1}) \times_1 \BFU^{(1)}_{k+1}\BFU_{k+1}^{(1)\top} \times_2 \BFU^{(2)}_{k+1}\BFU_{k+1}^{(2)\top}  \times_3 \BFU^{(3)}_{k+1}\BFU_{k+1}^{(3)\top}\|_F^2 \\
     & + \lambda\|\BFcalS^{k+1}\|_{TV1} \Big).
    \end{aligned}
\end{equation}
Combine \eqref{ineq: sufficient 5} and \eqref{ineq: sufficient 6}, we have
\begin{equation}\label{ineq: sufficient 7}
    G(\BFU_{k+1}\BFU_{k+1}^{\top},\BFcalS^{k+1},\BFcalX^{k+1}) \leq G(\BFU_{k+1}\BFU_{k+1}^{\top},\BFcalS^{k+1},\BFcalX^k) - \rho\|\BFcalX^{k+1} -\BFcalX^k\|_F^2.
\end{equation}
Sum \eqref{ineq: sufficient 1}, \eqref{ineq: sufficient 4}, and \eqref{ineq: sufficient 7} together, we can obtain our result
\begin{equation*}
\begin{aligned}
    & \rho(\|\BFU_{k+1}\BFU_{k+1}^{\top}-\BFU_k\BFU_k^{\top}\|_F^2+\|\BFcalS^{k+1} -\BFcalS^k\|_F^2+\|\BFcalX^{k+1} -\BFcalX^k\|_F^2)\\
    \leq & G(\BFU_k\BFU_k^{\top},\BFcalS^k,\BFcalX^k) -  G(\BFU_{k+1}\BFU_{k+1}^{\top},\BFcalS^{k+1},\BFcalX^{k+1}),     
\end{aligned}
\end{equation*}
which is the sufficient decrease property.
 \end{proof}
Our second main lemma is about the subgradient lower bound. To build the lemma, the following twos propositions are needed: (1)  Proposition~\ref{proposition: X,S bound} shows the iterative sequence $\{\BFU_k\BFU_k^{\top},\BFcalS^k,\BFcalX^k\}_{k\geq 0}$ from our Algorithm~\ref{alg: PAM} is bounded; (2)  Proposition~\ref{proposition: Lipschitz continuous} shows the Lipschitz continuity of the gradient of $\hat{F}(\cdot)$.
\begin{proposition}\label{proposition: X,S bound}
$\{\BFcalS^k,\BFcalX^k\}_{k\geq 0}$ are bounded sequence, where the bounds are determined by the initial values of $(\BFU_0\BFU_0^{\top},\BFcalS^0,\BFcalX^0)$. Specifically, it follows
\begin{equation} \label{ineq: X,S bound}
    \begin{aligned}
           \|\BFcalS^k\|_F & \leq \|\BFcalS^0\|_F +G(\BFU_0\BFU_0^{\top},\BFcalS^0,\BFcalX^0)/\rho \\
           \|\BFcalX^k\|_F & \leq \|\BFcalX^0\|_F+G(\BFU_0\BFU_0^{\top},\BFcalS^0,\BFcalX^0)/\rho.
    \end{aligned}
\end{equation}
\end{proposition}
\begin{proof} To start with,
\begin{subequations}
 \begin{align}
%  \begin{aligned}
          \|\BFcalS^k\|_F & = \|\sum_{i=1}^k(\BFcalS^i-\BFcalS^{i-1}) + \BFcalS^0\|_F  \nonumber \\ 
          & \leq \sum_{i=1}^k \|\BFcalS^i-\BFcalS^{i-1}\| + \|\BFcalS^0\|_F  \label{ineq: X,S bound 1}\\
          & \leq \sum_{i=1}^k \Big(G(\BFU_{i-1}\BFU_{i-1}^{\top},\BFcalS^{i-1},\BFcalX^{i-1}) - G(\BFU_i\BFU_i^{\top},\BFcalS^i,\BFcalX^i)\Big)/\rho  + \|\BFcalS^0\|_F  \label{ineq: X,S bound 2}\\
          & = \Big(G(\BFU_0\BFU_0^{\top},\BFcalS^0,\BFcalX^0) - G(\BFU_k\BFU_k^{\top},\BFcalS^k,\BFcalX^k)\Big)/\rho +\|\BFcalS^0\|_F \nonumber \\
          & \leq G(\BFU_0\BFU_0^{\top},\BFcalS^0,\BFcalX^0)/\rho + \|\BFcalS^0\|_F \label{ineq: X,S bound 3},
%  \end{aligned}
 \end{align}
 \end{subequations}
 where the inequality~\eqref{ineq: X,S bound 1} comes from triangle inequality, the inequality~\eqref{ineq: X,S bound 2} comes from \eqref{eq: sufficient decrease} in Lemma~\ref{lemma: Sufficient decrease property}, and the last inequality~\eqref{ineq: X,S bound 3}  is due to the fact that  $G(\BFU_k\BFU_k^{\top},\BFcalS^k,\BFcalX^k)\geq 0$. The same proof can also be applied to $\BFcalX^k$. In practice, we can set $\BFcalS^0=0$, $\BFcalX^0=\BFcalP(\BFcalF)$, $\BFU_0\BFU_0^{\top}$ is the HOSVD of  $\BFcalX^0$. By doing so, $G(\BFU_0\BFU_0^{\top},\BFcalS^0,\BFcalX^0)\leq \|\BFcalX^0\|_F^2$, which is bounded by the input data.
\end{proof}

\begin{proposition}\label{proposition: Lipschitz continuous}
For  bounded $\|\BFcalS\|_F^2$ and $\|\BFcalX\|_F^2$, there exists a constant $L \coloneqq (\|\BFcalX\|_F^2 + \|\BFcalS\|_F^2)\frac{2^{N}(\prod_{i=1}^Nr_i)}{\sqrt{N-1}}$  such that for any pair of 
	$\hat{\BFU}\hat{\BFU}{}^{\top} , \tilde{\BFU}\tilde{\BFU}{}^{\top},\forall \ n \in [N] $
	\begin{equation}\label{eq: L-smooth}
	    \|- \nabla_{\BFU^{(n)}\BFU^{(n)\top}} \hat{F}(\hat{\BFU}\hat{\BFU}{}^{\top},\BFcalS,\BFcalX)  + \nabla_{\BFU^{(n)}\BFU^{(n)\top}} \hat{F}(\tilde{\BFU}\tilde{\BFU}{}^{\top},\BFcalS,\BFcalX) \|_F\leq L\|\hat{\BFU}\hat{\BFU}{}^{\top} - \tilde{\BFU}\tilde{\BFU}{}^{\top}\|_F.
	\end{equation}
\end{proposition}
\begin{proof} Define $\|\cdot\|$ is the 2-operator norm.
 \begin{subequations}
	\begin{align}
	  &\|- \nabla_{\BFU^{(n)}\BFU^{(n)\top}} \hat{F}(\hat{\BFU}\hat{\BFU}{}^{\top},\BFcalS,\BFcalX)  + \nabla_{\BFU^{(n)}\BFU^{(n)\top}} \hat{F}(\tilde{\BFU}\tilde{\BFU}{}^{\top},\BFcalS,\BFcalX) \|_F   \nonumber\\
	       =& \| (\BFX_{(n)} -\BFS_{(n)}) \cdot \hat{\BFU}_{\BFPsi^{(n)}}\cdot {\hat{\BFU}_{\BFPsi^{(n)}}}^{\top}   \cdot  (\BFX_{(n)} -\BFS_{(n)}){}^{\top}    \nonumber \\
           &- (\BFX_{(n)}-\BFS_{(n)}) \cdot \tilde{\BFU}_{\BFPsi^{(n)}}\cdot {\tilde{\BFU}_{\BFPsi^{(n)}}}^{\top}   \cdot  (\BFX_{(n)} -\BFS_{(n)}){}^{\top}\|_F    \nonumber \\
           =& \|(\BFX_{(n)} -\BFS_{(n)}) \cdot (\hat{\BFU}_{\BFPsi^{(n)}}\cdot {\hat{\BFU}_{\BFPsi^{(n)}}}^{\top} -\tilde{\BFU}_{\BFPsi^{(n)}}\cdot {\tilde{\BFU}_{\BFPsi^{(n)}}}^{\top})   \cdot  (\BFX_{(n)} -\BFS_{(n)}){}^{\top} \|_F   \nonumber  \\
               \leq & \|(\BFX_{(n)} -\BFS_{(n)})(\BFX_{(n)} -\BFS_{(n)})^{\top}\| \cdot\|\hat{\BFU}_{\BFPsi^{(n)}}\cdot {\hat{\BFU}_{\BFPsi^{(n)}}}^{\top} -\tilde{\BFU}_{\BFPsi^{(n)}}\cdot {\tilde{\BFU}_{\BFPsi^{(n)}}}^{\top}\|_F  \label{eq: L-smooth 1}\\
            \leq & L\|\hat{\BFU}\hat{\BFU}{}^{\top} - \tilde{\BFU}\tilde{\BFU}{}^{\top}\|_F. \label{eq: L-smooth 2}
	\end{align}
	\end{subequations}
	where the inequality~\eqref{eq: L-smooth 1}  is due to the Frobenius norm and operator norm inequality, the  inequality~\eqref{eq: L-smooth 2} is based on the definition of $\hat{\BFU}_{\BFPsi^{(n)}}$,  $\tilde{\BFU}_{\BFPsi^{(n)}}$, and $L \coloneqq (\|\BFcalX\|_F^2 + \|\BFcalS\|_F^2)\frac{2^{N}(\prod_{i=1}^Nr_i)}{\sqrt{N-1}}$. 
	
	Next, we will show how we get the value for $L$.
	\begin{subequations} \label{ineq: UU bound}
	    \begin{align}
	         &\|\hat{\BFU}_{\BFPsi^{(n)}}\cdot {\hat{\BFU}_{\BFPsi^{(n)}}}^{\top} -\tilde{\BFU}_{\BFPsi^{(n)}}\cdot {\tilde{\BFU}_{\BFPsi^{(n)}}}^{\top}\|_F   \nonumber \\
	         = &\prod_{i\neq n} \|\hat{\BFU}^{(i)}\hat{\BFU}^{(i)\top} - \tilde{\BFU}^{(i)}\tilde{\BFU}^{(i)\top}\|_F  \label{ineq: UU bound 0} \\
	         = &\frac{1}{N-1}\sum_{j\neq n}\prod_{i\neq n} \|\hat{\BFU}^{(i)}\hat{\BFU}^{(i)\top} - \tilde{\BFU}^{(i)}\tilde{\BFU}^{(i)\top}\|_F \nonumber\\
	         = &\frac{2^{N-1}}{N-1}\sum_{j\neq n}(\prod_{i\neq n,j}r_i) \|\hat{\BFU}^{(j)}\hat{\BFU}^{(j)\top} - \tilde{\BFU}^{(j)}\tilde{\BFU}^{(j)\top}\|_F \nonumber \\
	         \leq  &\frac{2^{N-1}(\prod_{i=1}^Nr_i)}{N-1}\sum_{j\neq n} \|\hat{\BFU}^{(j)}\hat{\BFU}^{(j)\top} - \tilde{\BFU}^{(j)}\tilde{\BFU}^{(j)\top}\|_F  \label{ineq: UU bound 1}\\
	         \leq & \frac{2^{N-1}(\prod_{i=1}^Nr_i)}{N-1}\sqrt{N-1}\sqrt{\sum_{j\neq n}\|\hat{\BFU}^{(j)}\hat{\BFU}^{(j)\top} - \tilde{\BFU}^{(j)}\tilde{\BFU}^{(j)\top}\|_F^2} \label{ineq: UU bound 2}\\
	         \leq & \frac{2^{N-1}(\prod_{i=1}^Nr_i)}{\sqrt{N-1}}\|\hat{\BFU}\hat{\BFU}{}^{\top} - \tilde{\BFU}\tilde{\BFU}{}^{\top}\|_F, \label{ineq: UU bound 3}
	    \end{align}
	\end{subequations}
	where the equality~\eqref{ineq: UU bound 0} is based on the definition of $\hat{\BFU}_{\BFPsi^{(n)}}$, the inequality~\eqref{ineq: UU bound 1} is due to the fact that $r_i\geq 1$, the inequality~\eqref{ineq: UU bound 2} is from the Cauchy–Schwarz inequality, and the inequality~\eqref{ineq: UU bound 3} comes from $j\neq n$.
There is another inequality that needs to prove
	\begin{subequations} \label{ineq: X-S bound}
	    \begin{align}
	          & \|(\BFX_{(n)} -\BFS_{(n)})(\BFX_{(n)} -\BFS_{(n)})^{\top}\|   \nonumber \\
	         \leq & \|(\BFX_{(n)} -\BFS_{(n)})(\BFX_{(n)} -\BFS_{(n)})^{\top}\|_F \label{ineq: X-S bound 1} \\
	     =    & \|\BFX_{(n)} -\BFS_{(n)}\|_F^2  \nonumber\\
	        \leq   &2(\|\BFcalX\|_F^2 + \|\BFcalS\|_F^2),  \label{ineq: X-S bound 2}
	    \end{align}
	\end{subequations}
	where the inequality~\eqref{ineq: X-S bound 1} is based on the definition of  the 2-operator norm and the inequality~\eqref{ineq: X-S bound 2} is due to the Cauchy–Schwarz inequality. In addition, $\|\BFcalX\|_F^2 + \|\BFcalS\|_F^2$ is bounded in our algorithm based on~\eqref{ineq: X,S bound}. 
	  Based on the above two inequalities~\eqref{ineq: UU bound} and~\eqref{ineq: X-S bound}, $L \coloneqq (\|\BFcalX\|_F^2 + \|\BFcalS\|_F^2)\frac{2^{N}(\prod_{i=1}^Nr_i)}{\sqrt{N-1}}$ can be derived.
\end{proof}
Now our second main lemma is ready to be presented.
 \begin{lemma}[Subgradient lower bound]\label{lemma: subgradient lower bound}
Supposed that $0<\rho<\infty$,~$\{\BFU_k\BFU_k^{\top},\BFcalS^k,\BFcalX^k\}_{k\geq 0}$ is the sequence generated from the proposed Algorithm~\ref{alg: PAM}, then the sequence satisfies
\begin{equation*}\label{eq: subgradient lower bound}
\|\omega_{k+1}\|\leq\rho_1 \sqrt{\|\BFU_{k+1}\BFU_{k+1}^{\top}-\BFU_{k}\BFU_{k}^{\top}\|_F^2+\|\BFcalX^{k+1}-\BFcalX^k\|_F^2+\|\BFcalS^{k+1}-\BFcalS^k\|_F^2},
\end{equation*}
where $\rho_1 = \max \{2\sqrt{3}\rho+3L,4+2\rho+\kappa,  2+2\rho+\kappa\}$ with $L=\frac{2^{N}(\prod_{i=1}^Nr_i)}{\sqrt{N-1}}\Big(2G(\BFU_0\BFU_0^{\top},\BFcalS^0,\BFcalX^0)/\rho + \|\BFcalS^0\|_F+\|\BFcalX^0\|_F\Big)^2$ and $\kappa = 4G(\BFU_0\BFU_0^{\top},\BFcalS^0,\BFcalX^0)/\rho + 2\|\BFcalS^0\|_F+2\|\BFcalX^0\|_F$.
\end{lemma}
%  	(ii)\textit{A subgradient lower bound for the iterates gap}: %\todo{I add more explanation} 
 \begin{proof}%[Proof of Lemma~\ref{lemma: subgradient lower bound}]
     	According to the first-order optimality condition for each sub-problem~\eqref{eq: Update U^(n)} in $k$-th iteration of the proposed algorithm, we have
 	\begin{align*}
 	    	\textbf{0} \in & \nabla_{\BFU^{(n)}\BFU^{(n)\top}}  \hat{F}(\{\BFU^{(i)}_{k+1}\BFU_{k+1}^{(i)\top}\}_{i\leq n},\{\BFU^{(i)}_{k}\BFU_{k}^{(i)\top}\}_{i>n},\BFcalS^k, \BFcalX^k) \\ 
 	   + &	 2\rho(\BFU^{(n)}_{k+1}\BFU^{(n)\top}_{k+1} - \BFU^{(n)}_k\BFU^{(n)\top}_k) + \partial_{\BFU^{(n)}\BFU^{(n)\top}} \delta_{S_n}(\BFU^{(n)}_{k+1}\BFU^{(n)\top}_{k+1}),
 	\end{align*}
 which can be rewritten as
 \begin{equation} \label{eq: subgradient Un}
     \begin{aligned}
            & -2\rho(\BFU^{(n)}_{k+1}\BFU^{(n)\top}_{k+1} - \BFU^{(n)}_k\BFU^{(n)\top}_k) - \nabla_{\BFU^{(n)}\BFU^{(n)\top}}  \hat{F}(\{\BFU^{(i)}_{k+1}\BFU_{k+1}^{(i)\top}\}_{i\leq n},\{\BFU^{(i)}_{k}\BFU_{k}^{(i)\top}\}_{i>n},\BFcalS^k, \BFcalX^k)\\
  &+ \nabla_{\BFU^{(n)}\BFU^{(n)\top}} \hat{F}(\BFU_{k+1}\BFU_{k+1}^{\top},\BFcalS^k,\BFcalX^k)- \nabla_{\BFU^{(n)}\BFU^{(n)\top}} \hat{F}(\BFU_{k+1}\BFU_{k+1}^{\top},\BFcalS^k,\BFcalX^k) \\
  & + \nabla_{\BFU^{(n)}\BFU^{(n)\top}} \hat{F}(\BFU_{k+1}\BFU_{k+1}^{\top},\BFcalS^{k+1},\BFcalX^k)- \nabla_{\BFU^{(n)}\BFU^{(n)\top}} \hat{F}(\BFU_{k+1}\BFU_{k+1}^{\top},\BFcalS^{k+1},\BFcalX^k) \\
  &+ \nabla_{\BFU^{(n)}\BFU^{(n)\top}} \hat{F}(\BFU_{k+1}\BFU_{k+1}^{\top},\BFcalS^{k+1},\BFcalX^{k+1})\\
  &\in \nabla_{\BFU^{(n)}\BFU^{(n)\top}} \hat{F}(\BFU_{k+1}\BFU_{k+1}^{\top},\BFcalS^{k+1},\BFcalX^{k+1}) + \partial_{\BFU^{(n)}\BFU^{(n)\top}} \delta_{S_n}(\BFU^{(n)}_{k+1}\BFU^{(n)\top}_{k+1})\\
 &:=\partial_{\BFU^{(n)}\BFU^{(n)\top}} G(\BFU_{k+1}\BFU_{k+1}^{\top},\BFcalS^{k+1},\BFcalX^{k+1}).  
     \end{aligned}
 \end{equation}
	Based on the Definition~\ref{def: Subdifferentials}, we have the following proposition.
\begin{proposition}[\normalfont\textit{Subdifferentiability property}~\citep{bolte2014proximal}]~Given that $\Psi(\bm{x},\bm{y})=H(\bm{x},\bm{y}) +f(\bm{x}) +g(\bm{y})$, if $H$ is continuously differentiable, then for all $(\bm{x},\bm{y})$ we have
\begin{equation*} 
           \partial \Psi(\bm{x},\bm{y}) = (\nabla_{\bm{x}} H(\bm{x},\bm{y})+\partial f(\bm{x}),\nabla_{\bm{y}} H(\bm{x},\bm{y})+\partial g(\bm{y})).
\end{equation*}
\end{proposition}
Accordingly, based on the first-order optimality condition (well-known Fermat’s rule) of \eqref{eq: Update S},
\begin{equation}  \label{eq: subgradient S 1}
    \begin{aligned}
         \textbf{0} & \in 2\Big(\BFcalS^{k+1}-\BFcalX^k+(\BFcalX^k - \BFcalS^k) \times_1 \BFU^{(1)}_{k+1}\BFU_{k+1}^{(1)\top} \times_2 \BFU^{(2)}_{k+1}\BFU_{k+1}^{(2)\top}  \times_3 \BFU^{(3)}_{k+1}\BFU_{k+1}^{(3)\top}\Big)\\
     & + \lambda\partial\|\BFcalS^{k+1}\|_{TV1}+2\rho(\BFcalS^{k+1} -\BFcalS^k).
    \end{aligned}
\end{equation}
In addition,
\begin{equation} \label{eq: subgradient S 2}
\begin{aligned}
     &  \partial_{\BFcalS} G(\BFU_{k+1}\BFU_{k+1}^{\top},\BFcalS^{k+1},\BFcalX^{k+1}) \\
        = &  2(\BFcalS^{k+1}-\BFcalX^{k+1})+2(\BFcalX^{k+1} - \BFcalS^{k+1}) \times_1 \BFU^{(1)}_{k+1}\BFU_{k+1}^{(1)\top} \times_2 \BFU^{(2)}_{k+1}\BFU_{k+1}^{(2)\top}  \times_3 \BFU^{(3)}_{k+1}\BFU_{k+1}^{(3)\top}\\
     & + \lambda\partial\|\BFcalS^{k+1}\|_{TV1}.
\end{aligned}
\end{equation}
Combine with \eqref{eq: subgradient S 1} and \eqref{eq: subgradient S 2},
\begin{equation} \label{eq: subgradient S final}
    \begin{aligned}
    & -2(\BFcalX^{k+1}-\BFcalX^k)+2(\BFcalX^{k+1}-\BFcalX^k)\times_1 \BFU^{(1)}_{k+1}\BFU_{k+1}^{(1)\top} \times_2 \BFU^{(2)}_{k+1}\BFU_{k+1}^{(2)\top}  \times_3 \BFU^{(3)}_{k+1}\BFU_{k+1}^{(3)\top}\\
          & + 2(\BFcalS^k- \BFcalS^{k+1}) \times_1 \BFU^{(1)}_{k+1}\BFU_{k+1}^{(1)\top} \times_2 \BFU^{(2)}_{k+1}\BFU_{k+1}^{(2)\top}  \times_3 \BFU^{(3)}_{k+1}\BFU_{k+1}^{(3)\top} + 2\rho(\BFcalS^k- \BFcalS^{k+1})\\
           &  \in \partial_{\BFcalS} G(\BFU_{k+1}\BFU_{k+1}^{\top},\BFcalS^{k+1},\BFcalX^{k+1}).
    \end{aligned}
\end{equation}
Accordingly, based on the first order optimality condition (well-known Fermat’s rule) of \eqref{eq: Update X},
\begin{equation}  \label{eq: subgradient X 1}
    \begin{aligned}
         \textbf{0} & \in 2(\BFcalX^{k+1} - \BFcalS^{k+1}) - 2(\BFcalX^k - \BFcalS^{k+1}) \times_1 \BFU^{(1)}_{k+1}\BFU_{k+1}^{(1)\top} \times_2 \BFU^{(2)}_{k+1}\BFU_{k+1}^{(2)\top}  \times_3 \BFU^{(3)}_{k+1}\BFU_{k+1}^{(3)\top}\\
     & +2\rho(\BFcalX^{k+1}- \BFcalX^{k})+\partial\delta_{S}(\BFcalX^{k+1}).
    \end{aligned}
\end{equation}
In addition,
\begin{equation} \label{eq: subgradient X 2}
\begin{aligned}
     &  \partial_{\BFcalX} G(\BFU_{k+1}\BFU_{k+1}^{\top},\BFcalS^{k+1},\BFcalX^{k+1}) \\
        = &  2(\BFcalX^{k+1}-\BFcalS^{k+1})-2(\BFcalX^{k+1} - \BFcalS^{k+1}) \times_1 \BFU^{(1)}_{k+1}\BFU_{k+1}^{(1)\top} \times_2 \BFU^{(2)}_{k+1}\BFU_{k+1}^{(2)\top}  \times_3 \BFU^{(3)}_{k+1}\BFU_{k+1}^{(3)\top}\\
     & +\partial\delta_{S}(\BFcalX^{k+1}).
\end{aligned}
\end{equation}
Combine with \eqref{eq: subgradient X 1} and \eqref{eq: subgradient X 2}, we have
\begin{equation}  \label{eq: subgradient X final}
    \begin{aligned}
    &-2(\BFcalX^{k+1}-\BFcalX^k)\times_1 \BFU^{(1)}_{k+1}\BFU_{k+1}^{(1)\top} \times_2 \BFU^{(2)}_{k+1}\BFU_{k+1}^{(2)\top}  \times_3 \BFU^{(3)}_{k+1}\BFU_{k+1}^{(3)\top}+ 2\rho(\BFcalX^k- \BFcalX^{k+1})\\
             \in & \partial_{\BFcalX} G(\BFU_{k+1}\BFU_{k+1}^{\top},\BFcalS^{k+1},\BFcalX^{k+1}).
    \end{aligned}
\end{equation}
Thus 
\begin{subequations} \label{ineq: outcome 2}
    \begin{align}
           & \|\nabla_{\BFU^{(n)}\BFU^{(n)\top}} \hat{F}(\BFU_{k+1}\BFU_{k+1}^{\top},\BFcalS^k,\BFcalX^k)- \nabla_{\BFU^{(n)}\BFU^{(n)\top}} \hat{F}(\BFU_{k+1}\BFU_{k+1}^{\top},\BFcalS^{k+1},\BFcalX^k)\|_F \nonumber \\
           = & \|(\BFX_{(n)}^k -\BFS_{(n)}^k) \cdot \BFU_{\BFPsi^{(n)}_{k+1}}\cdot {\BFU_{\BFPsi^{(n)}_{k+1}}}^{\top}   \cdot  (\BFX_{(n)}^k -\BFS_{(n)}^k)^{\top}  \nonumber \\
           &- (\BFX_{(n)}^k -\BFS_{(n)}^{k+1}) \cdot \BFU_{\BFPsi^{(n)}_{k+1}}\cdot {\BFU_{\BFPsi^{(n)}_{k+1}}}^{\top}   \cdot  (\BFX_{(n)}^k -\BFS_{(n)}^{k+1})^{\top}\|_F \nonumber \\
           \leq &  \|(\BFX_{(n)}^k -\BFS_{(n)}^k) \cdot \BFU_{\BFPsi^{(n)}_{k+1}}\cdot {\BFU_{\BFPsi^{(n)}_{k+1}}}^{\top}   \cdot  (\BFX_{(n)}^k -\BFS_{(n)}^k)^{\top}  \label{ineq: outcome 2-1} \\
           &- (\BFX_{(n)}^k -\BFS_{(n)}^{k}) \cdot \BFU_{\BFPsi^{(n)}_{k+1}}\cdot {\BFU_{\BFPsi^{(n)}_{k+1}}}^{\top}   \cdot  (\BFX_{(n)}^k -\BFS_{(n)}^{k+1})^{\top}\|_F \nonumber\\
           &+\|(\BFX_{(n)}^k -\BFS_{(n)}^k) \cdot \BFU_{\BFPsi^{(n)}_{k+1}}\cdot {\BFU_{\BFPsi^{(n)}_{k+1}}}^{\top}   \cdot  (\BFX_{(n)}^k -\BFS_{(n)}^{k+1})^{\top}   \nonumber \\
           &- (\BFX_{(n)}^k -\BFS_{(n)}^{k+1}) \cdot \BFU_{\BFPsi^{(n)}_{k+1}}\cdot {\BFU_{\BFPsi^{(n)}_{k+1}}}^{\top}   \cdot  (\BFX_{(n)}^k -\BFS_{(n)}^{k+1})^{\top}\|_F \nonumber \\
           = &  \|(\BFX_{(n)}^k -\BFS_{(n)}^k) \cdot \BFU_{\BFPsi^{(n)}_{k+1}}\cdot {\BFU_{\BFPsi^{(n)}_{k+1}}}^{\top}   \cdot  (\BFS_{(n)}^{k+1} -\BFS_{(n)}^k)^{\top}\|_F  \nonumber \\
           &+ \|(\BFX_{(n)}^k -\BFS_{(n)}^{k+1
           }) \cdot \BFU_{\BFPsi^{(n)}_{k+1}}\cdot {\BFU_{\BFPsi^{(n)}_{k+1}}}^{\top}   \cdot  (\BFS_{(n)}^{k+1} -\BFS_{(n)}^k)^{\top}\|_F  \nonumber \\
         \leq  &  \Big(\|(\BFX_{(n)}^k -\BFS_{(n)}^k) \cdot \BFU_{\BFPsi^{(n)}_{k+1}}\cdot {\BFU_{\BFPsi^{(n)}_{k+1}}}^{\top}\|+\|(\BFX_{(n)}^k -\BFS_{(n)}^{k+1}) \cdot \BFU_{\BFPsi^{(n)}_{k+1}}\cdot {\BFU_{\BFPsi^{(n)}_{k+1}}}^{\top}\|\Big) \|\BFcalS^{k+1}-\BFcalS^k\|_F \label{ineq: outcome 2-2} \\
           \leq  &  \Big(\|(\BFX_{(n)}^k -\BFS_{(n)}^k) \cdot \BFU_{\BFPsi^{(n)}_{k+1}}\cdot {\BFU_{\BFPsi^{(n)}_{k+1}}}^{\top}\|_F+\|(\BFX_{(n)}^k -\BFS_{(n)}^{k+1}) \cdot \BFU_{\BFPsi^{(n)}_{k+1}}\cdot {\BFU_{\BFPsi^{(n)}_{k+1}}}^{\top}\|_F\Big) \|\BFcalS^{k+1}-\BFcalS^k\|_F \label{ineq: outcome 2-3} \\
          \leq & \Big( \|\BFX_{(n)}^k -\BFS_{(n)}^k\|_F +\|\BFX_{(n)}^k -\BFS_{(n)}^{k+1}\|_F \Big)\|\BFcalS^{k+1}-\BFcalS^k\|_F \label{ineq: outcome 2-4}\\
          \leq & (2\|\BFcalX^k\|_F+\|\BFcalS^k\|_F+\|\BFcalS^{k+1}\|_F)\|\BFcalS^{k+1}-\BFcalS^k\|_F  \label{ineq: outcome 2-5} \\
          \leq & \kappa\|\BFcalS^{k+1}-\BFcalS^k\|_F, \label{ineq: outcome 2-6} 
    \end{align}
\end{subequations}
where the inequalities~\eqref{ineq: outcome 2-1} and~\eqref{ineq: outcome 2-5}   are because of the triangle inequality, the inequality~\eqref{ineq: outcome 2-2} comes from the Frobenius norm and operator norm inequality, the  inequality~\eqref{ineq: outcome 2-3} is based on the definition of  the 2-operator norm, and the  inequality~\eqref{ineq: outcome 2-4} is because $\BFU_{\BFPsi^{(n)}_{k+1}}$ is an semi-orthogonal matrix. In~\eqref{ineq: outcome 2-6}, $\kappa = 4G(\BFU_0\BFU_0^{\top},\BFcalS^0,\BFcalX^0)/\rho + 2\|\BFcalS^0\|_F+2\|\BFcalX^0\|_F \geq 2\|\BFcalX^k\|_F+\|\BFcalS^k\|_F+\|\BFcalS^{k+1}\|_F$ is dependent on the upper bound for $\BFcalX,\BFcalS$. 

In the end, combining \eqref{eq: subgradient Un}, \eqref{eq: subgradient S final}, and \eqref{eq: subgradient X final}, we have 
\begin{subequations} \label{ineq: outcome 3} 
     \begin{align} 
           & \textrm{d}(\textbf{0},\partial G(\BFU_{k+1}\BFU_{k+1}^{\top},\BFcalS^{k+1},\BFcalX^{k+1})  \nonumber\\
            \leq &\sum_{n=1}^3 \|-2\rho(\BFU^{(n)}_{k+1}\BFU^{(n)\top}_{k+1} - \BFU^{(n)}_k\BFU^{(n)\top}_k)  \label{ineq: outcome 3-1}\\
            & - \nabla_{\BFU^{(n)}\BFU^{(n)\top}}  \hat{F}(\{\BFU^{(i)}_{k+1}\BFU_{k+1}^{(i)\top}\}_{i\leq n},\{\BFU^{(i)}_{k}\BFU_{k}^{(i)\top}\}_{i>n},\BFcalS^k, \BFcalX^k)  \nonumber\\
  &+ \nabla_{\BFU^{(n)}\BFU^{(n)\top}} \hat{F}(\BFU_{k+1}\BFU_{k+1}^{\top},\BFcalS^k,\BFcalX^k)- \nabla_{\BFU^{(n)}\BFU^{(n)\top}} \hat{F}(\BFU_{k+1}\BFU_{k+1}^{\top},\BFcalS^k,\BFcalX^k)   \nonumber\\
  & + \nabla_{\BFU^{(n)}\BFU^{(n)\top}} \hat{F}(\BFU_{k+1}\BFU_{k+1}^{\top},\BFcalS^{k+1},\BFcalX^k)- \nabla_{\BFU^{(n)}\BFU^{(n)\top}} \hat{F}(\BFU_{k+1}\BFU_{k+1}^{\top},\BFcalS^{k+1},\BFcalX^k)   \nonumber \\
  &+ \nabla_{\BFU^{(n)}\BFU^{(n)\top}} \hat{F}(\BFU_{k+1}\BFU_{k+1}^{\top},\BFcalS^{k+1},\BFcalX^{k+1})\|_F  \nonumber\\
    & +\|-2(\BFcalX^{k+1}-\BFcalX^k)+2(\BFcalX^{k+1}-\BFcalX^k)\times_1 \BFU^{(1)}_{k+1}\BFU_{k+1}^{(1)\top} \times_2 \BFU^{(2)}_{k+1}\BFU_{k+1}^{(2)\top}  \times_3 \BFU^{(3)}_{k+1}\BFU_{k+1}^{(3)\top}  \nonumber\\
          & + 2(\BFcalS^k- \BFcalS^{k+1}) \times_1 \BFU^{(1)}_{k+1}\BFU_{k+1}^{(1)\top} \times_2 \BFU^{(2)}_{k+1}\BFU_{k+1}^{(2)\top}  \times_3 \BFU^{(3)}_{k+1}\BFU_{k+1}^{(3)\top} + 2\rho(\BFcalS^k- \BFcalS^{k+1})\|_F  \nonumber\\
          & +\|-2(\BFcalX^{k+1}-\BFcalX^k)\times_1 \BFU^{(1)}_{k+1}\BFU_{k+1}^{(1)\top} \times_2 \BFU^{(2)}_{k+1}\BFU_{k+1}^{(2)\top}  \times_3 \BFU^{(3)}_{k+1}\BFU_{k+1}^{(3)\top}+ 2\rho(\BFcalX^k- \BFcalX^{k+1})\|_F  \nonumber\\
          \leq & \sum_{n=1}^3 (2\rho\|\BFU^{(n)}_{k+1}\BFU^{(n)\top}_{k+1} - \BFU^{(n)}_k\BFU^{(n)\top}_k\|_F+L\|\BFU_{k+1}\BFU_{k+1}^{\top}-\BFU_{k}\BFU_{k}^{\top}\|_F) \label{ineq: outcome 3-2} \\
          & + \|\nabla_{\BFU^{(n)}\BFU^{(n)\top}} \hat{F}(\BFU_{k+1}\BFU_{k+1}^{\top},\BFcalS^k,\BFcalX^k)- \nabla_{\BFU^{(n)}\BFU^{(n)\top}} \hat{F}(\BFU_{k+1}\BFU_{k+1}^{\top},\BFcalS^{k+1},\BFcalX^k)\|_F  \nonumber\\
          & +\|\nabla_{\BFU^{(n)}\BFU^{(n)\top}} \hat{F}(\BFU_{k+1}\BFU_{k+1}^{\top},\BFcalS^{k+1},\BFcalX^k)-\nabla_{\BFU^{(n)}\BFU^{(n)\top}} \hat{F}(\BFU_{k+1}\BFU_{k+1}^{\top},\BFcalS^{k+1},\BFcalX^{k+1})\|_F  \nonumber\\
          & + (4+2\rho)\|\BFcalX^{k+1}-\BFcalX^k\|_F+(2+2\rho)\|\BFcalS^{k+1}-\BFcalS^k\|_F  \nonumber\\
          \leq & (2\sqrt{3}\rho+3L)\|\BFU_{k+1}\BFU_{k+1}^{\top}-\BFU_{k}\BFU_{k}^{\top}\|_F  \label{ineq: outcome 3-3}\\
          & + (4+2\rho+\kappa)\|\BFcalX^{k+1}-\BFcalX^k\|_F+(2+2\rho+\kappa)\|\BFcalS^{k+1}-\BFcalS^k\|_F,  \nonumber
    \end{align}
\end{subequations}
where the inequality~\eqref{ineq: outcome 3-1} and the inequality~\eqref{ineq: outcome 3-2} are due to the triangle inequality, the inequality~\eqref{ineq: outcome 3-3} is from~\eqref{eq: L-smooth} in Proposition~\ref{proposition: Lipschitz continuous} with $L=\frac{2^{N}(\prod_{i=1}^Nr_i)}{\sqrt{N-1}}\Big(2G(\BFU_0\BFU_0^{\top},\BFcalS^0,\BFcalX^0)/\rho + \|\BFcalS^0\|_F+\|\BFcalX^0\|_F\Big)^2$ and~\eqref{ineq: outcome 2}. 

Therefore,  the following holds
\begin{equation} \label{ineq: outcome3}
    \begin{aligned}
           &  \textrm{d}(\textbf{0},\partial G(\BFU_{k+1}\BFU_{k+1}^{\top},\BFcalS^{k+1},\BFcalX^{k+1}))\\
              \leq & (2\sqrt{3}\rho+3L)\|\BFU_{k+1}\BFU_{k+1}^{\top}-\BFU_{k}\BFU_{k}^{\top}\|_F\\
          & + (4+2\rho+\kappa)\|\BFcalX^{k+1}-\BFcalX^k\|_F+(2+2\rho+\kappa)\|\BFcalS^{k+1}-\BFcalS^k\|_F\\
   \leq &\sqrt{3} \rho_1 \sqrt{\|\BFU_{k+1}\BFU_{k+1}^{\top}-\BFU_{k}\BFU_{k}^{\top}\|_F^2+\|\BFcalX^{k+1}-\BFcalX^k\|_F^2+\|\BFcalS^{k+1}-\BFcalS^k\|_F^2}, \\
    \end{aligned}
\end{equation}
where the first inequality is~\eqref{ineq: outcome 3} and the second inequality is because of  the Cauchy–Schwarz inequality with $\rho_1 = \max \{2\sqrt{3}\rho+3L,4+2\rho+\kappa,  2+2\rho+\kappa\}$. Thus, there exist~$\omega_{k+1}\in \partial G(\BFU_{k+1}\BFU_{k+1}^{\top},\BFcalS^{k+1},\BFcalX^{k+1})$,
\[ \|\omega_{k+1}\|\leq\rho_1 \sqrt{\|\BFU_{k+1}\BFU_{k+1}^{\top}-\BFU_{k}\BFU_{k}^{\top}\|_F^2+\|\BFcalX^{k+1}-\BFcalX^k\|_F^2+\|\BFcalS^{k+1}-\BFcalS^k\|_F^2}.\]
which is the property of subgradient lower bound.
 \end{proof}
 
 \begin{definition}[\normalfont \textit{Critical point}~\citep{attouch2009convergence, attouch2010proximal}]~A necessary condition for $\bm{x}$ to be a minimizer of a proper and lower semicontinuous (PLSC) function  $f$ is that
 \begin{equation} \label{eq: critical point}
     \bm{0} \in \partial f(\bm{x}).
 \end{equation}
 A point that satisfies \eqref{eq: critical point} is called limiting-critical or simply critical.
\end{definition}
Based on Lemmas~\ref{lemma: Sufficient decrease property} and~\ref{lemma: subgradient lower bound}, the following theorem summarizes the theoretical property of iterative sequence $\{\BFU_k\BFU_k^{\top},\BFcalS^k,\BFcalX^k\}_{k\geq 0}$ from our Algorithm~\ref{alg: PAM}.
 	\begin{theorem} \label{thm: key1}
 Let $\{\BFU_k\BFU_k^{\top},\BFcalS^k,\BFcalX^k\}_{k\geq 0} $ denote the sequence generated from Algorithm~\ref{alg: PAM} with $w(\BFU_0\BFU_0^{\top},\BFcalS^0,\BFcalX^0)$ denoting the set of all its limit points, and let set $\textrm{crit} G=\{(\BFU\BFU^{\top},\BFcalS,\BFcalX): (\BFU\BFU^{\top},\BFcalS,\BFcalX)\text{ is a critical point of \eqref{eq: proposed formulation unconstrained}}\}$. Then 
	\begin{enumerate}[label=(\roman*)]
	\item The sequence  $\{G(\BFU_k\BFU_k^{\top},\BFcalS^k,\BFcalX^k)\}_{k \geq 0 }$ is nonincreasing;
	\item  $\sum_{k\geq 0}  \| (\BFU_{k+1}\BFU_{k+1}^{\top},\BFcalS^{k+1},\BFcalX^{k+1}) - (\BFU_k\BFU_k^{\top},\BFcalS^k,\BFcalX^k)\|_F^2 < +\infty $;
	\item $w(\BFU_0\BFU_0^{\top},\BFcalS^0,\BFcalX^0) \subseteq \textrm{crit}G$; 
	\item $w(\BFU_0\BFU_0^{\top},\BFcalS^0,\BFcalX^0)$ is a nonempty compact connected set, and 
	\begin{equation*}
	    \begin{aligned}
	         &  \mathrm{d}\Big((\BFU_k\BFU_k^{\top},\BFcalS^k,\BFcalX^k),w(\BFU_0\BFU_0^{\top},\BFcalS^0,\BFcalX^0)\Big) \\
	       \coloneqq   &  \inf_{(\BFU\BFU^{\top},\BFcalS,\BFcalX) \in w(\BFU_0\BFU_0^{\top},\BFcalS^0,\BFcalX^0)} \|(\BFU\BFU^{\top},\BFcalS,\BFcalX) - (\BFU_k\BFU_k^{\top},\BFcalS^k,\BFcalX^k)\|_F\to 0, \textnormal{as } k \to +\infty;
	    \end{aligned}
	\end{equation*}
	\item $G(\cdot)$ is finite and constant on $w(\BFU_0\BFU_0^{\top},\BFcalS^0,\BFcalX^0)$.
\end{enumerate}
	\end{theorem}
\begin{proof}
The proof is split into five parts.

\textit{(i)} It comes from Lemma~\ref{lemma: Sufficient decrease property};

\textit{(ii)} Let $f_k \coloneqq G(\BFU_{k+1}\BFU_{k+1}^{\top},\BFcalS^{k+1},\BFcalX^{k+1})$. According to the \textit{Monotone convergence theorem},  $\lim\limits_{k \to \infty }{f_k}$ exist since $f_k$ is bounded from below, namely, $f_k \geq 0$. Let $\varepsilon_k \coloneqq f_{k-1} - f_k$
\begin{equation*}
    f_k=f_0-\sum_{i=1}^{k} \varepsilon_i.
\end{equation*}
Based on \eqref{eq: sufficient decrease},
\begin{equation*}
   \rho\|(\BFU_{k+1}\BFU_{k+1}^{\top},\BFcalS^{k+1},\BFcalX^{k+1}) - (\BFU_k\BFU_k^{\top},\BFcalS^k,\BFcalX^k)\|_F^2 \leq \varepsilon_{k+1},
\end{equation*} which implies that
\begin{equation*}
\sum_{k\geq 0}\|(\BFU_{k+1}\BFU_{k+1}^{\top},\BFcalS^{k+1},\BFcalX^{k+1}) - (\BFU_k\BFU_k^{\top},\BFcalS^k,\BFcalX^k)\|_F^2 \leq (\sum_{k\geq 0}\varepsilon_{k+1})/\rho<+\infty
\end{equation*}
\begin{equation*} %\label{eq: subsequence}
    \lim\limits_{k \to \infty }{ \|(\BFU_{k+1}\BFU_{k+1}^{\top},\BFcalS^{k+1},\BFcalX^{k+1}) - (\BFU_k\BFU_k^{\top},\BFcalS^k,\BFcalX^k)\|_F}=0;
\end{equation*}

\textit{(iii)} Based on \eqref{ineq: outcome3} in Lemma~\ref{lemma: subgradient lower bound},
\begin{equation*}
  \lim\limits_{k \to \infty}{ \textrm{d}(\textbf{0},\partial G(\BFU_{k+1}\BFU_{k+1}^{\top},\BFcalS^{k+1},\BFcalX^{k+1})} = 0;
\end{equation*}

\textit{(iv)} This can be derived from part \textit{(ii)};

\textit{(v)} This result is based on part \textit{(i)} and function $G(\cdot)$ is nonnegative. 
\end{proof}

However, Theorem~\ref{thm: key1} cannot guarantee the global convergence of the iterative sequence $\{\BFU_k\BFU_k^{\top},\BFcalS^k,\BFcalX^k\}_{k\geq 0}$, which has the following definition.
\begin{definition}[\normalfont \textit{Global convergence}~\citep{petrovai2017global,xu2018convergence}] %lanckriet2009convergence,
Any iterative algorithm for solving an optimization problem over a set $X$, is said to be \textbf{globally convergent} if for any starting point $\bm x_0 \in X$, the sequence generated by the algorithm always has an accumulation critical point.
\end{definition}
To build the global convergence of our iterative sequence $\{\BFU_k\BFU_k^{\top},\BFcalS^k,\BFcalX^k\}_{k\geq 0}$ based on Theorem~\ref{thm: key1},   the function $G(\cdot)$ needs to have  K\L~property as follows
\begin{definition}[\normalfont \textit{K\L~property}~\citep{attouch2010proximal,attouch2013convergence,bolte2014proximal, xu2018convergence}]\label{df: kl}~A real function $f: \mathbb{R}^d \to (-\infty,+\infty]$ has the  Kurdyka \L{}ojasiewicz (K\L) property, namely, for any point $\bar{\bm{u}}\in \mathbb{R}^d$, in a neighborhood $N(\bar{\bm{u}},\sigma)$, there exists a desingularizing function $\phi(s)=cs^{1-\theta}$ for some $c>0$ and $\theta \in [0,1)$ such that
\begin{equation*}
    \phi'(|f(\bm{u})-f(\bar{\bm{u}})|)\mathrm{d}(0,\partial f(\bm{u}))\geq 1 \text{ for any } \bm{u} \in N(\bar{\bm{u}},\sigma) \text{ and } f(\bm{u})\neq f(\bar{\bm{u}}).
\end{equation*}
\end{definition}
The semi-algebraic set and semi-algebraic function are related to Kurdyka \L{}ojasiewicz (K\L) property, which are introduced below.
\begin{definition}[\normalfont\textit{Semi-algebraic}~\citep{attouch2009convergence,bolte2014proximal}] \label{df: semi-algerba} 
	A subset $S$ of $\mathbb{R}^d$ is a real \textbf{semi-algebraic set} if there exist  a finite number of real
	polynomial functions $g_{ij},h_{ij}$:  $\mathbb{R}^d \to \mathbb{R}$ such that
	\[S=\cup_{j=1}^q\cap_{i=1}^p\{\bm{u}\in \mathbb{R}^d:g_{ij}(\bm{u})=0 \textrm{ and }h_{ij}(\bm{u})<0 \}.\]
In addition, a function $h:\mathbb{R}^{d+1} \to \mathbb{R}\cup{+\infty}$ is called 	\textbf{semi-algebraic} if its graph
\[\{(\bm{u}, t)\in \mathbb{R}^{d+1}: h(\bm{u})=t \} \]
is a real semi-algebraic set.
\end{definition}

% \begin{definition} \label{df: semi-algerba}
% 	A subset $S$ of $\mathbb{R}^d$ is a real \textbf{semi-algebraic set} if there exist  a finite number of real
% 	polynomial functions $g_{ij},h_{ij}$:  $\mathbb{R}^d \to \mathbb{R}$ such that
% 	\[S=\cup_{j=1}^q\cap_{i=1}^p\{\BFu\in \mathbb{R}^d:g_{ij}(\BFu)=0 \textrm{ and }h_{ij}(\BFu)<0 \}.\]
% In addition, a function $h:\mathbb{R}^{d+1} \to \mathbb{R}\cup{+\infty}$ is called 	\textbf{semi-algebraic} if its graph
% \[\{(\textbf{u},t)\in \mathbb{R}^{d+1}: h(\textbf{u})=t \} \]
% is a real semi-algebraic set.
% \end{definition}

After introducing these two definitions, the following lemma shows that the objective function $G(\BFU\BFU^\top,\BFcalS,\BFcalX)$ of~\eqref{eq: proposed formulation unconstrained} has the so-called K\L~property.
\begin{lemma} \label{lemma: G KL property}
Function $G(\cdot)$ has the K\L~property.
\end{lemma}  
\begin{proof} % I_n should be changed based on the problem
    We will first prove that sets $\{S_n\}_{n\in [3]}$ are semi-algebraic sets.
In~\citep{bolte2014proximal,lewis2008alternating}, the authors showed that the set of all matrices with the same rank is semi-algebraic.
Therefore, for each $n\in [N]$, set
\[T_n \coloneqq \{\BFX \in \mathbb{R}^{I_n \times I_n}: \textrm{rank}(\BFX)=r_n\}\]  
is semi-algebraic.
In addition, we observe that
\begin{align*}
    K_n&:=\{\BFX \in \mathbb{R}^{I_n \times I_n} : \mathrm{eigenvalue} \ \mathrm{of} \ \BFX \mathrm{\ is \ either \ 1 \ or \ 0 } ,  \BFX=\BFX^{\top} \}\\
    &=\{\BFX \in \mathbb{R}^{I_n \times I_n} : \BFX\BFX^{\top}=\BFX,\BFX^{\top}\BFX=\BFX\},
\end{align*}
where all the equalities are quadratic. Clearly, set $K_n$ is semi-algebraic.
Since $S_n=T_n\cap K_n$ and intersection of two semi-algebraic sets is still semi-algebraic~\citep{bolte2014proximal}, thus $S_n$ is a semi-algebraic set.
Recall that
\[G(\BFU\BFU^\top,\BFcalS,\BFcalX)=\hat{F}(\BFU\BFU^\top,\BFcalS,\BFcalX)+\delta_{S}(\BFcalX)+ \sum_{n \in [3]} \delta_{S_n}(\BFU^{(n)}\BFU^{(n)\top}).\]
 $\hat{F}(\BFU\BFU^\top,\BFcalS,\BFcalX)$  is a function of summation of polynomial functions of all the elements in $\{\BFU\BFU^\top,\BFcalS,\BFcalX\}$. % and $M$ different $\|\cdot\|_0$ norm functions. The polynomial functions and $\|\cdot\|_0$ norm are semi-algebraic functions as exemplified in \textit{example 2,3}~\citet{bolte2014proximal}. 
 In addition, the characteristic function $\delta_{A}(\cdot)$ is a semi-algebraic functions if set $A$ is semi-algebraic~\citep{attouch2010proximal, bolte2014proximal}. Any finite sum of semi-algebraic functions is semi-algebraic, thus function $G(\BFU\BFU^\top,\BFcalS,\BFcalX)$ is semi-algebraic.  A semi-algebraic real-valued function 
is a K\L~function based on the work of~\citep{bolte2007lojasiewicz,bolte2014proximal}. 
\end{proof}

% According to Lemma~\ref{lemma: G KL property} and Section 3.2 in~\citet{bolte2014proximal},  $\{\BFU_{t} {\BFU^{\top}_{t}},\BFS_{t}\}_{t\geq 0}$, it is sufficient to check the following two conditions.%$G(\BFX)$ has the so-called K\L{ }property ~\citetp{attouch2010proximal, attouch2013convergence, bolte2014proximal,xu2013block}. The global convergence can be obtained if the following claim holds ~\citetp{bolte2014proximal}:
% 	\begin{enumerate}[label=(\roman*)]
%  \item \textit{Sufficient decrease property}: There exists a positive constant $\rho_1$ for $\forall\ t=0,1,\dots$ such that
%  \begin{equation} \label{eq: sufficient}
%      \rho_1 (\|\BFU_t\BFU_t^{\top}-\BFU_{t+1}\BFU_{t+1}^{\top}\|_F^2+\|\BFS_t - \BFS_{t+1}\|_F^2)\leq G(\BFU_t\BFU_t^{\top},\BFS_t)-G(\BFU_{t+1}\BFU_{t+1}^{\top},\BFS_{t+1}),
%  \end{equation}
%  \item \textit{A subgradient lower bound for the iterates gap}: There exists the other positive constant $\rho_2$ for $\forall\ t=0,1,\dots$ such  that
%  \begin{equation} \label{eq: subgradient}
%      \|\omega_{t+1}\|\leq \rho_2 \sqrt{\|\BFU_t\BFU_t^{\top}-\BFU_{t+1}\BFU_{t+1}^{\top}\|_F^2+\|\BFS_t - \BFS_{t+1}\|_F^2} , \omega_{t+1}\in \partial G(\BFU_{t+1}\BFU_{t+1}^{\top},\BFS_{t+1}).
%  \end{equation}
%  	\end{enumerate}
% Conditions (i) and (ii), together with the K\L~property of function imply that $\{\BFU_t\BFU_t^{\top},\BFS_t\}_{t\geq 0}$ is a Cauchy sequence, and thus is convergent~\citet{bolte2014proximal}.% where $\BFU_t\BFU_t^{\top}=(\BFU^{(1)}_t{\BFU^{(1)\top}_t},\dots,\BFU^{(N)}_t{\BFU^{(N)\top}_t})$. 

Based on Lemmas~\ref{lemma: Sufficient decrease property},~\ref{lemma: subgradient lower bound}, and~\ref{lemma: G KL property} and conclusions in~\citep[Section 3.2]{bolte2014proximal}, the following main Theorem can be obtained.
\begin{theorem}[Global Convergence] \label{thm: convergence}
$\{\BFU_k\BFU_k^{\top},\BFcalS^k,\BFcalX^k\}_{k\geq 0}$ is the sequence generated from the proposed Algorithm~\ref{alg: PAM} with any  initial point so that $\BFcalS^0,\BFcalX^0$, and $G(\BFU_0\BFU_0^{\top},\BFcalS^0,\BFcalX^0)$ are bounded. Then, there exists $(\BFU_*\BFU_*^{\top},\BFcalS^*,\BFcalX^*)$  such that
	\begin{enumerate}[label=(\roman*)]
	\item  $(\BFU_k\BFU_k^{\top},\BFcalS^k,\BFcalX^k) \to (\BFU_*\BFU_*^{\top},\BFcalS^*,\BFcalX^*)$;
	\item $\bm{0}\in\partial G(\BFU_*\BFU_*^{\top},\BFcalS^*,\BFcalX^*)$;
	\item  $\{\BFU_k\BFU_k^{\top},\BFcalS^k,\BFcalX^k\}_{k\geq 0}$ has a finite length, namely,
	\[\sum_{k=0}^{+\infty}\|(\BFU_{k+1}\BFU_{k+1}^{\top},\BFcalS^{k+1},\BFcalX^{k+1})- (\BFU_k\BFU_k^{\top},\BFcalS^k,\BFcalX^k)\|_F < + \infty.\]
	\end{enumerate}
\end{theorem}

\section{Numerical Studies} \label{sec: numerical study}
To evaluate the performance of the proposed SRTC,  its performance on open-sourced video data is presented in this section. In Section~\ref{subsec: convergence analysis}, the empirical convergence of the proposed algorithm is illustrated to verify our theoretical results. The performances of the proposed algorithm for background subtraction and foreground detection are presented in Sections~\ref{subsec: background subtraction} and~\ref{subsec: foreground detection}, respectively.  In Sections~\ref{subsec: background subtraction} and~\ref{subsec: foreground detection}, MCOS\footnote{\url{https://github.com/ZihengLi6321/MCOS}}~\citep{li2021matrix}, BFMNM\footnote{\href{https://1drv.ms/u/s!Aur7I-sQwed-gjE_15tuofYp2n7x?e=xN57mw}{One drive}}~\citep{shang2017bilinear}, HQ-ASD\footnote{\url{https://github.com/he1c/robust-matrix-completion}}~\citep{he2019robust}, RTRC\footnote{\url{https://github.com/HuyanHuang/Robust-Low-rank-Tensor-Ring-Completion}}~\citep{huang2021robust},  and HQ-TCASD\footnote{\url{https://github.com/he1c/robust-tensor-completion}}~\citep{he2020robust} are selected as benchmarks for comparison with the proposed SRTC, which are state-of-the-art methods in the related area.  The benchmarks have two categories:   
\begin{enumerate}
\item MCOS, BFMNM, and HQ-ASD are the most advanced Robust Matrix Completion algorithms in the literature;
\item RTRC and HQ-TCASD are state-of-the-art Robust Tensor Completion algorithms.
\end{enumerate} 
All results in this section are the average results of 20 repetitions for comparison.  The codes of SRTC are implemented in Matlab 2021a. The CPU of the computer to conduct experiments in this paper is an Intel\textsuperscript{\textregistered} Xeon\textsuperscript{\textregistered} Processor E-2286M (8-cores 2.40-GHz Turbo, 16 MB).

\textbf{Performance evaluation indices and parameter tuning:} For the task of background subtraction, the peak signal-to-noise ratio (PSNR) and the structural similarity index (SSIM) are used to measure the recovery accuracy. PSNR and SSIM commonly measure the similarity of two
images in intensity and structure, respectively. Specifically, PSNR is defined as: $\textnormal{PSNR} = 10\times \log_{10}{\frac{255^2}{\|\BFI - \hat{\BFI}\|_F^2}}$, where  $\BFI$ and $\hat{\BFI}$ are respectively the original and recovered  background. SSIM measures the structural similarity of two images; see~\citep{wang2004image} for details. Average PSNR and SSIM over all image frames in the video are used to evaluate recovery performance of video background. For the task of foreground detection, F-measure is applied to assess the foreground  detection performance. % F-measure is  defined as $\textnormal{F-measure}=2\frac{\textnormal{precision}\cdot\textnormal{recall}}{\textnormal{precision} + \textnormal{recall}}$, where 
% \begin{equation*}
%     \begin{aligned}
%       \textnormal{precision} & = \frac{\textnormal{\#correctly classified foreground pixels}}{\textnormal{\#pixels classified as foreground}}, \\
%       \textnormal{recall} & = \frac{\textnormal{\#correctly classified foreground pixels}}{\textnormal{\#foreground pixels in ground truth}}.
%     \end{aligned}
% \end{equation*}
Average F-measure over all image frames in the video is applied  to  evaluate  the detection  performance  of video  foreground. Therefore, 20 repetitions are sufficient  to represent the performance of each method since each repetition is the average performance of multiple image frames. For these performance indices PSNR, SSIM, and F-measure, higher values   indicate the better performance.

% For the proposed method as well as the benchmark methods, the parameter tuning is performed by searching from 100 sets of parameters sampled by the maximin Latin hypercube design~\citep{joseph2008orthogonal} such that the average PSNR and F-measure achieve the best value for background subtraction and foreground detection, respectively.   %\textcolor{orange}{Need to specify the range of tuning parameters}.  

\subsection{Convergence Analysis} \label{subsec: convergence analysis}
The video data set \textit{Caviar2} from SBI data set\footnote{\url{https://sbmi2015.na.icar.cnr.it/SBIdataset.html}}~\citep{maddalena2015towards} is used in this subsection. In total, this video data set has 460 image frames, where the size of each grayscale image is  $256\times 384$. For simplicity, the first 80 image frames in the sequence are used for experiments. Therefore, the tensor size is $256\times 384 \times 80$. One image frame from \textit{Caviar2}  in this experiment is shown in Figure~\ref{subfig: Caviar2}. In the video, there are people entering and leaving a store, standing only for few frames. For each image, a ratio of pixels are randomly selected as missing pixels, and the positions of the missing pixels are unknown (one example with $50\%$ missing pixels is shown in Figure~\ref{subfig: missing Caviar2}). To evaluate the convergence of the proposed algorithm, the relative change $\textbf{\textnormal{relChgA}}=\frac{\|\BFcalA^k-\BFcalA^{k-1}\|_F}{\max(1,\|\BFcalA^{k-1}\|_F)}$ and the relative error $\textbf{\textnormal{relErrA}}=\frac{\|\BFcalA^k-\BFcalA^*\|_F}{\max(1,\|\BFcalA^*\|_F)}$ are applied as the assessment indices of algorithm convergence, where $\BFcalA^k$ is the result in $k$-th iteration and $\BFcalA^*$ is the ground truth.  The ground truth of the static video background is provided in the first column of Figure~\ref{fig: background visualization}. For the case of orthogonal matrices $\BFU$, the relative change has the following representation $\textbf{\textnormal{relChgU}}=\frac{\|\BFU_k\BFU_k^{\top}-\BFU_{k-1}\BFU_{k-1}^{\top}\|_F}{\max(1,\|\BFU_{k-1}\BFU_{k-1}^{\top}\|_F)}$.

\begin{figure}[!htbp]
	\centering
	\subfloat[Objective value]{\includegraphics[width=0.25\textwidth]{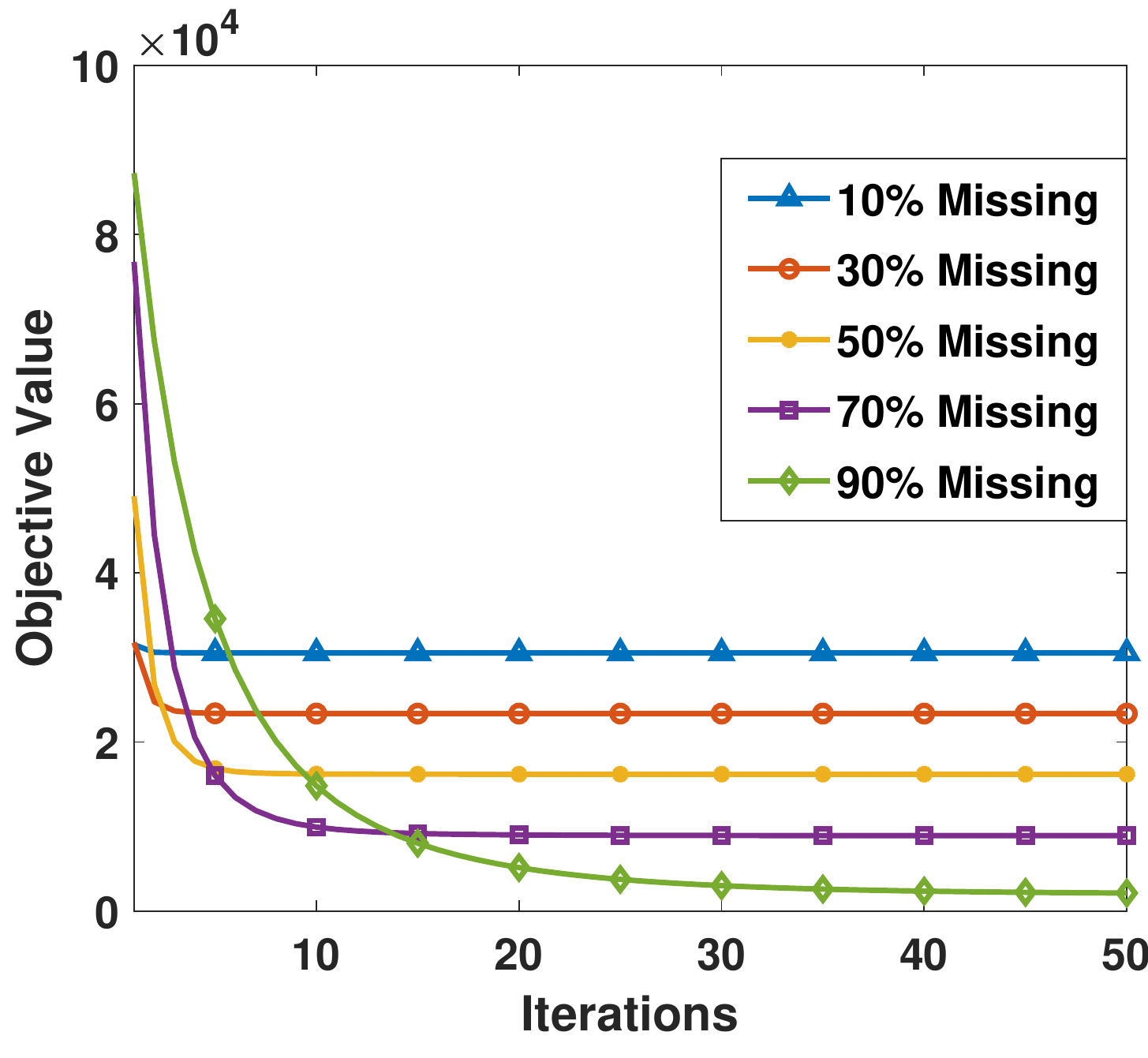}%
	\label{subfig: objective value}}
	\subfloat[Relative change of $\BFcalX$]{\includegraphics[width=0.24\textwidth]{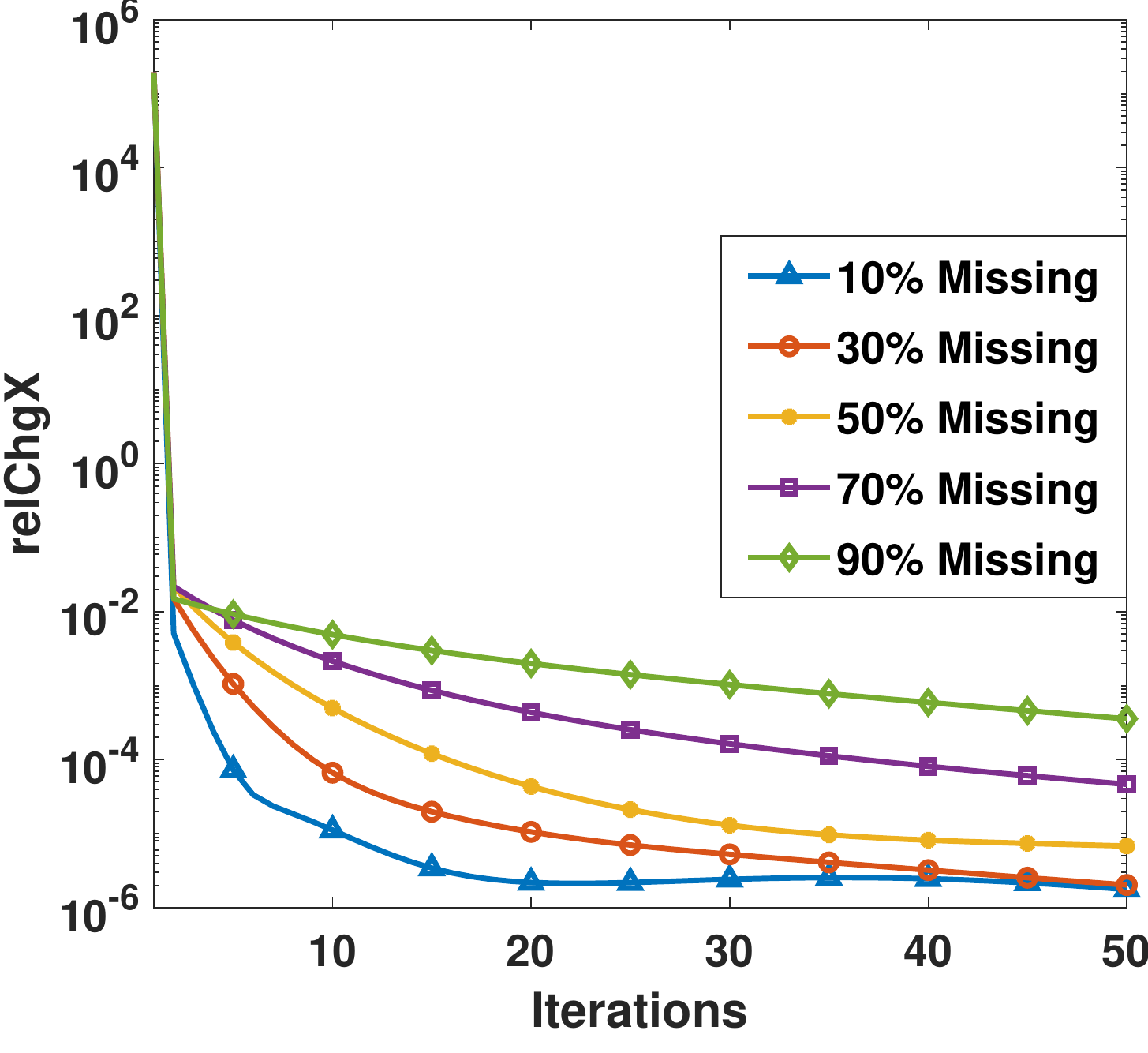}%
	  \label{subfig: relative change X}}
	  	\subfloat[Relative change of $\BFcalL$]{\includegraphics[width=0.24\textwidth]{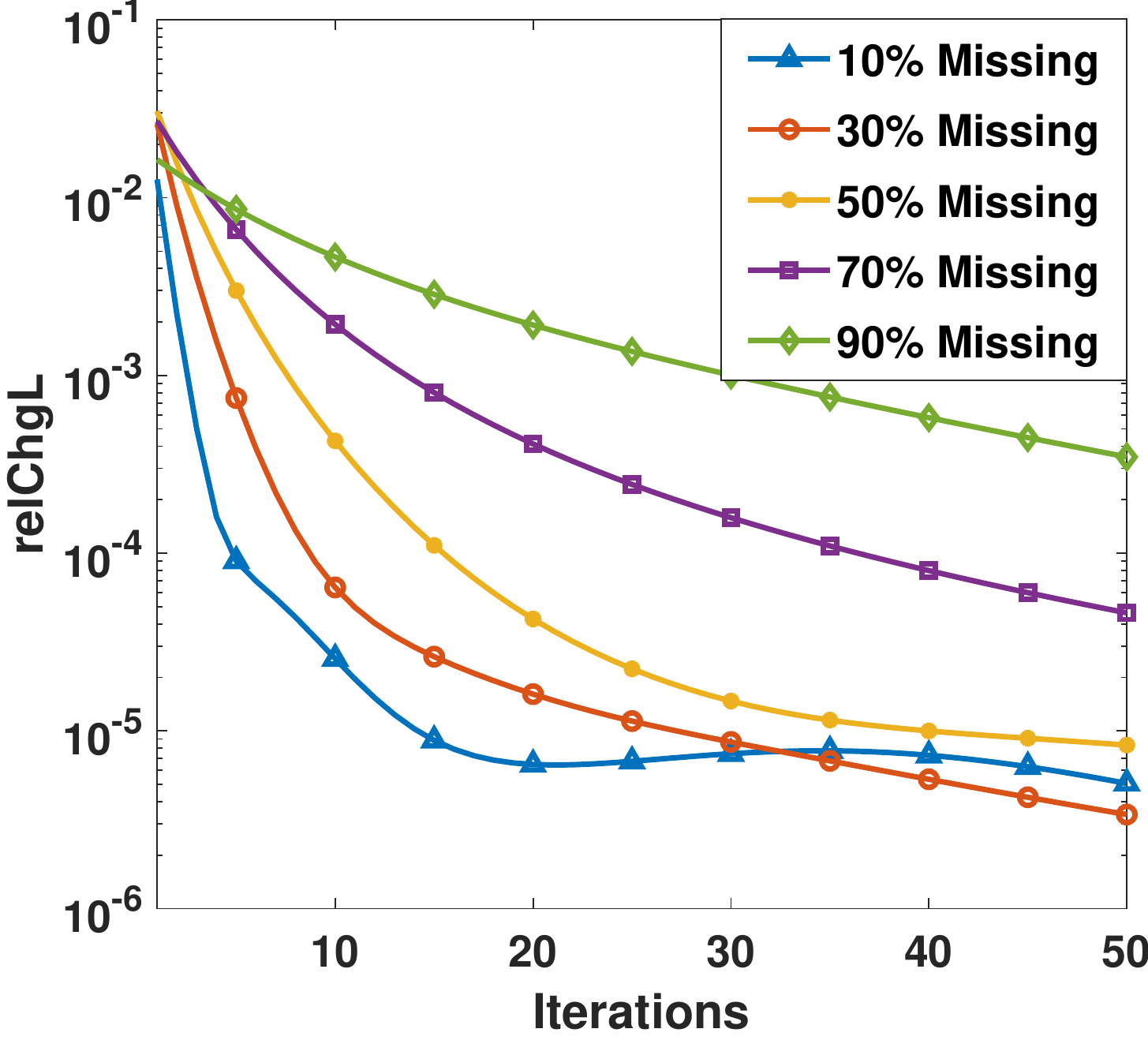}%
	  \label{subfig: relative change L}}
	  	  	\subfloat[Relative change of $\BFcalS$]{\includegraphics[width=0.25\textwidth]{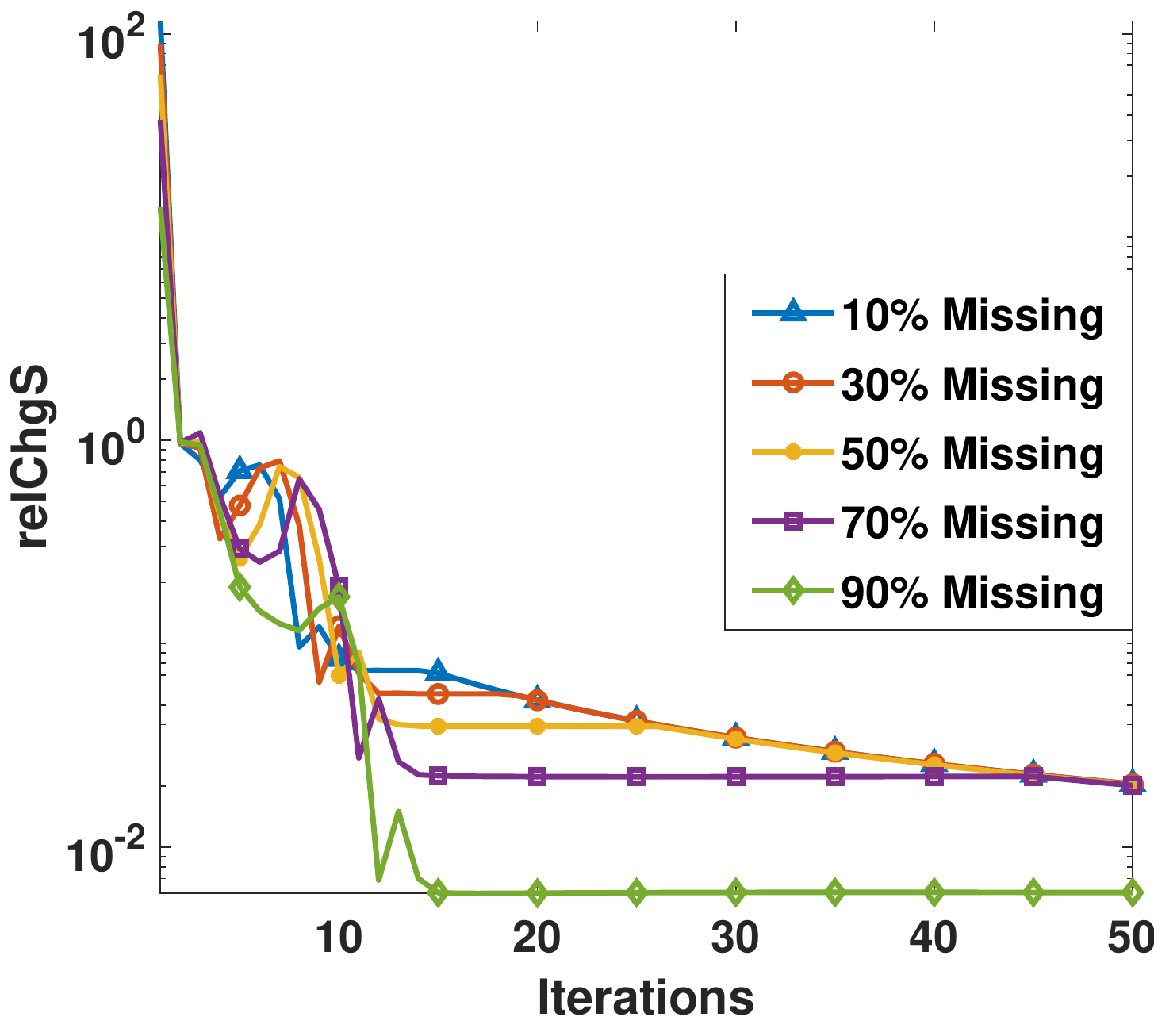}%
	  \label{subfig: relative change S}}
\caption{The empirical convergence analysis of tenPAM algorithm with different ratios of missing pixels (a) Objective value in \eqref{eq: proposed formulation}; (b) Relative change of $\BFcalX$; (c) Relative change of $\BFcalL$; (d) Relative change of $\BFcalS$.}
\label{fig: convergence 1}
\end{figure}
In this experiment, parameter $\lambda$ is set to the values of $0.5$ and the ratio of missing pixels can be selected from $\{10\%,30\%,50\%,70\%,90\%\}$. The curves of the objective value in \eqref{eq: proposed formulation}, the relative change of the full video  $\BFcalX$, the relative change of the video background $\BFcalL$, and the relative change of the video foreground $\BFcalS$ are shown in Figure~\ref{fig: convergence 1}. Figure~\ref{subfig: objective value} illustrates the monotone decreasing trends for the curve of the objective value in \eqref{eq: proposed formulation}, which verify the theoretical results in Theorem~\ref{thm: key1}.  \begin{figure}[!htb]
	\centering
	\subfloat[Relative change of $\BFU^{(1)}$]{\includegraphics[width=0.33\textwidth]{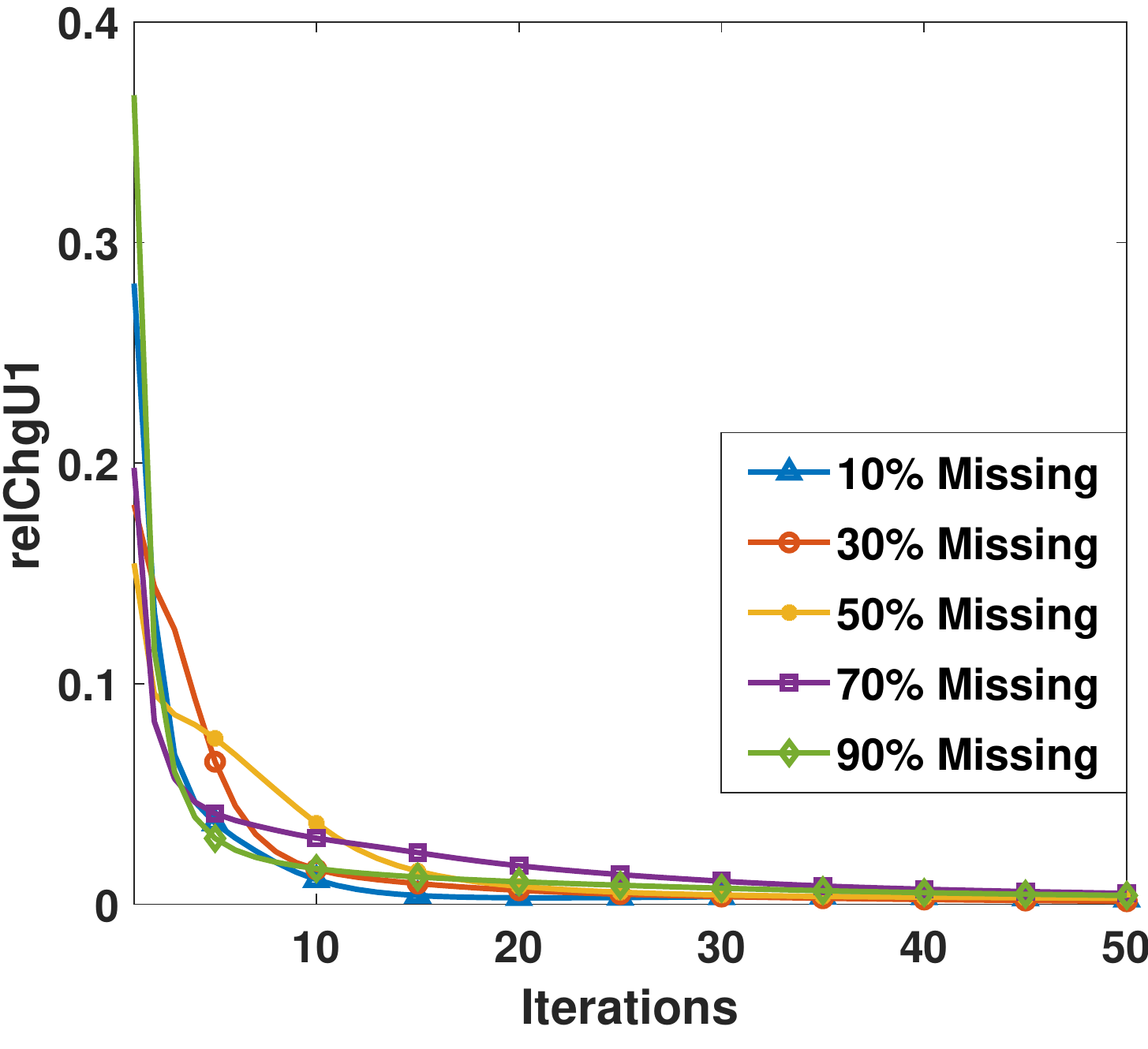}%
	  \label{subfig: relative change U1}}
	  	\subfloat[Relative change of $\BFU^{(2)}$]{\includegraphics[width=0.33\textwidth]{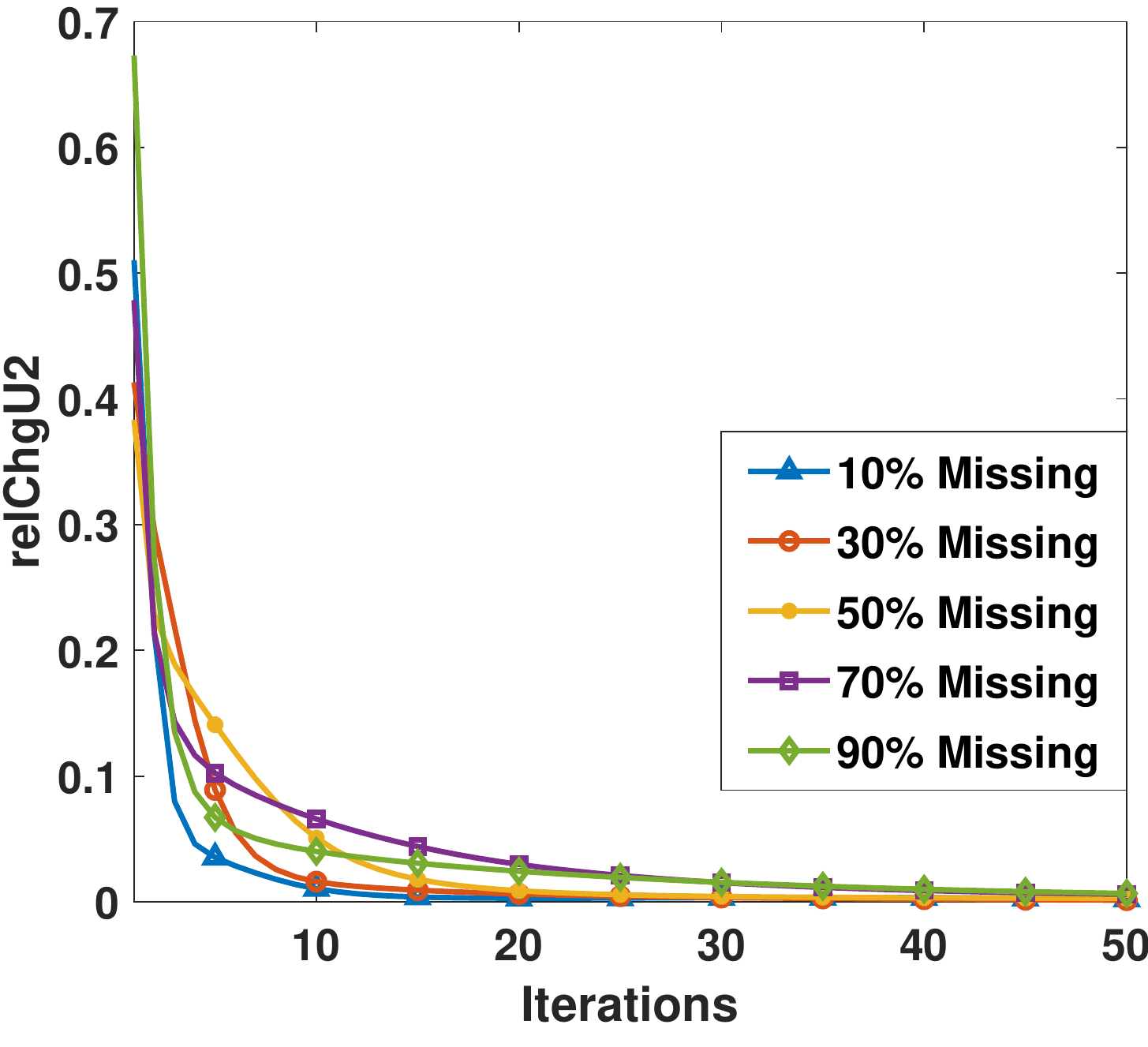}%
	  \label{subfig: relative change U2}}
	 	\subfloat[Relative change of $\BFU^{(3)}$]{\includegraphics[width=0.34\textwidth]{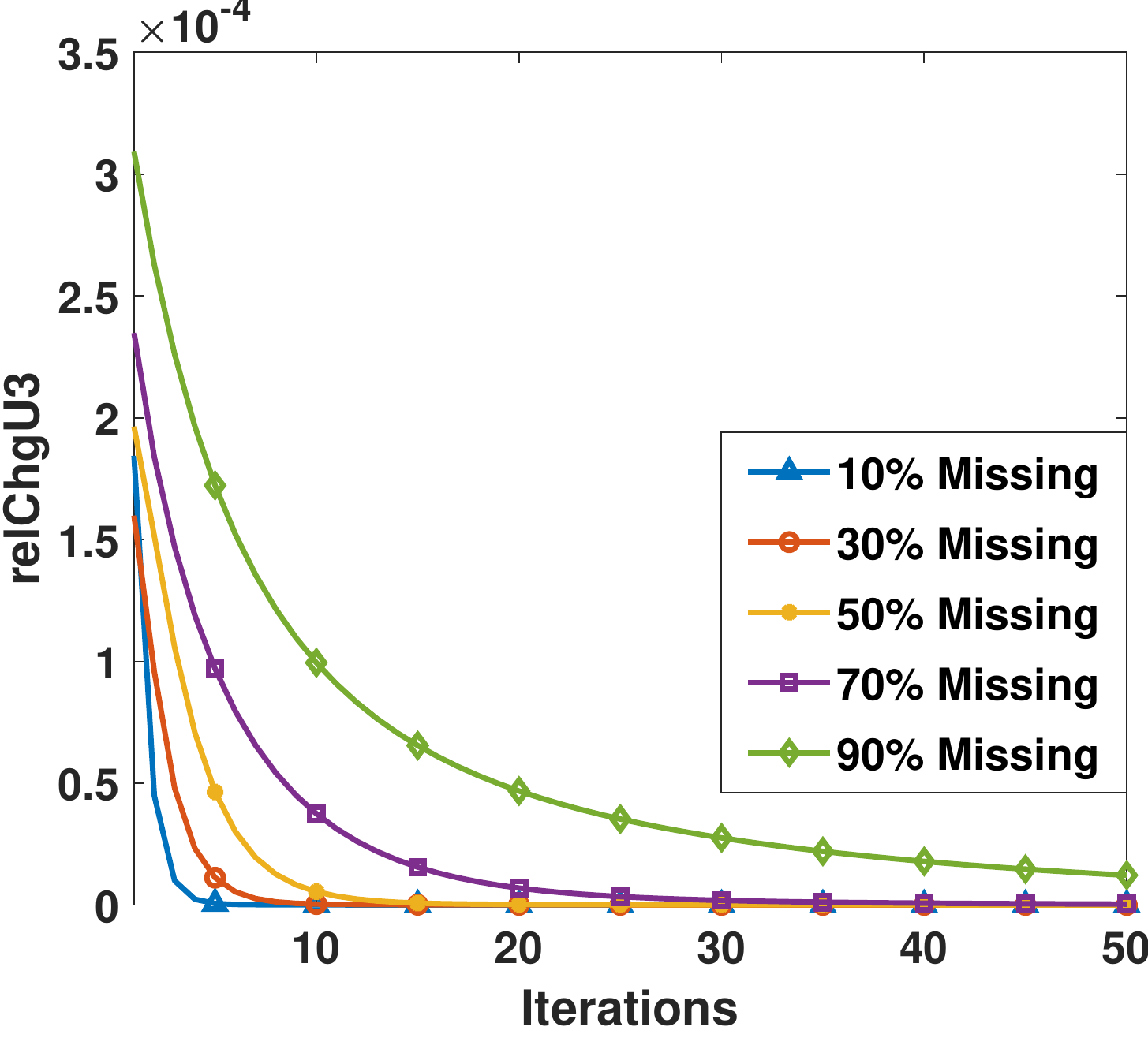}%
	  \label{subfig: relative change U3}}
\caption{The empirical convergence analysis of tenPAM algorithm with different ratios of missing pixels: (a) Relative change of $\BFU^{(1)}$; (b) Relative change of $\BFU^{(2)}$; (c) Relative change of $\BFU^{(3)}$.} 
\label{fig: convergence 2}
\end{figure} Meanwhile, it also shows that a bigger ratio of missing pixels implies a smaller objective value but slower convergence speed because it has less data to learn. From Figures~\ref{subfig: relative change X}, \ref{subfig: relative change L}, and~\ref{subfig: relative change S}, the relative changes of $\BFcalX$, $\BFcalL$, and $\BFcalS$ converge to zero when the number of iterations keeps increasing. Figure~\ref{fig: convergence 2} illustrates that the relative changes of $\BFU^{(1)}$, $\BFU^{(2)}$, and $\BFU^{(3)}$  converge to zero very fast for different ratios of missing pixels. Figures~\ref{fig: convergence 1} and~\ref{fig: convergence 2}  demonstrate that the convergence results in Theorems~\ref{thm: key1} and~\ref{thm: convergence}   are empirically verified. 
 \begin{figure}[!htb]
	\centering
	\subfloat[Relative error of $\BFcalX$]{\includegraphics[width=0.5\textwidth]{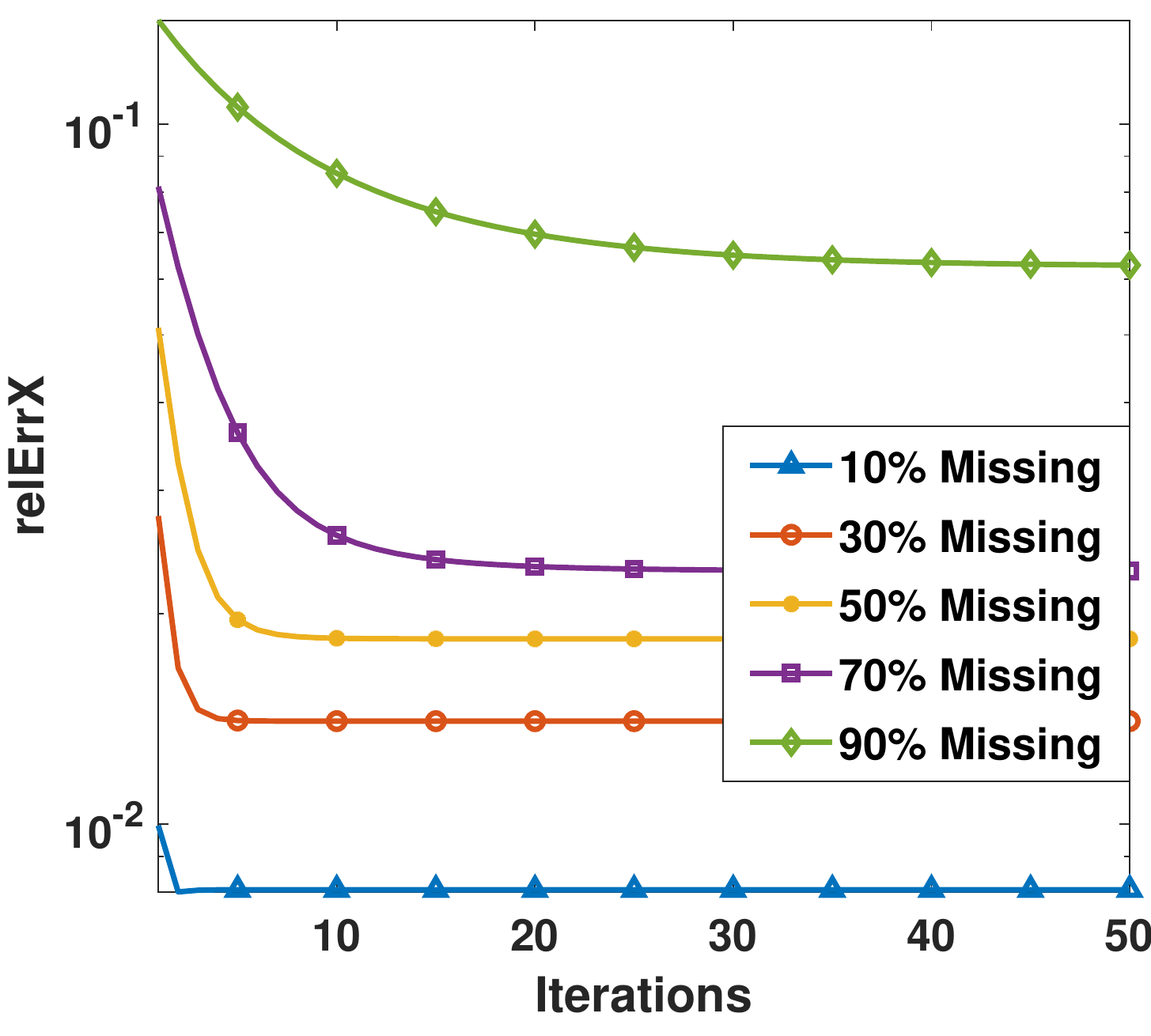}%
	  \label{subfig: relative error X}}
	  	\subfloat[Relative change of $\BFcalL$]{\includegraphics[width=0.5\textwidth]{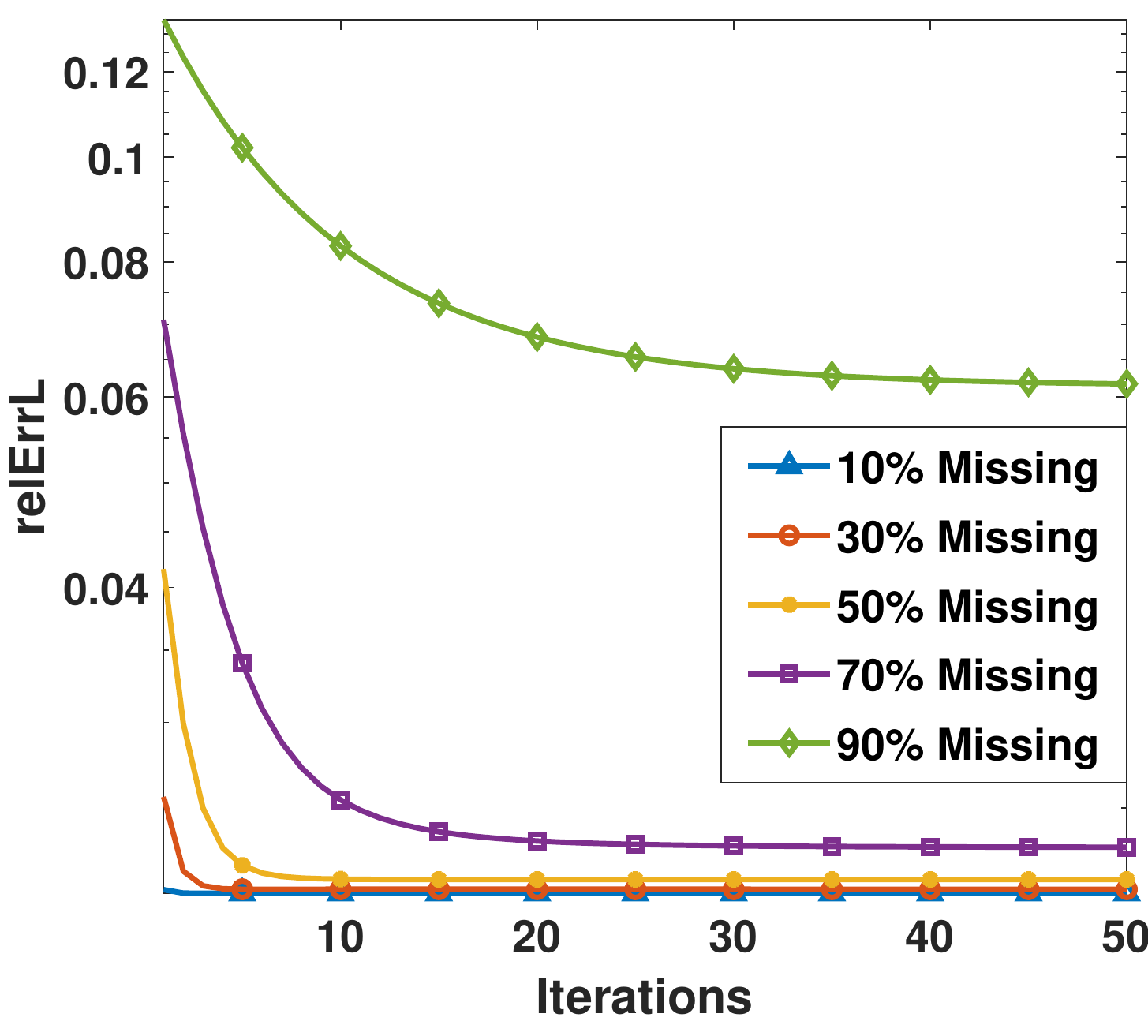}%
	  \label{subfig: relative error L}}
\caption{The empirical convergence analysis of tenPAM algorithm with different ratios of missing pixels: (a) Relative error of $\BFcalX$; (b) Relative error of $\BFcalL$.}
\label{fig: convergence 3}
\end{figure}

In addition, the ground truth of full video $\BFcalX$ and video background $\BFcalL$ is known to us, the curves of the relative error of $\BFcalX$ and $\BFcalA$ are shown in Figure~\ref{fig: convergence 3}. The curve of the relative error of the video foreground $\BFcalS$ is not provided since the ground truth video foreground for real data is unknown. From Figures~\ref{subfig: relative error X} and~\ref{subfig: relative error L}, the relative errors of the full video $\BFcalX$ and video background $\BFcalL$ gradually decrease to a stable value. In general, the results in this experiment show that  our proposed tenPAM algorithm~\ref{alg: PAM} can converge within 50 iterations.

\subsection{Background Subtraction} \label{subsec: background subtraction}
In this subsection, the proposed method is applied to background subtraction. The video data set \textit{Caviar2}  used in Section~\ref{subsec: convergence analysis}, \textit{Candela}, and \textit{Caviar1}  from SBI data set~\citep{maddalena2015towards} are used for the experiments.   \begin{figure}[!htbp] 
	\centering 
	\subfloat[\textit{Candela}]{\includegraphics[width=0.33\textwidth]{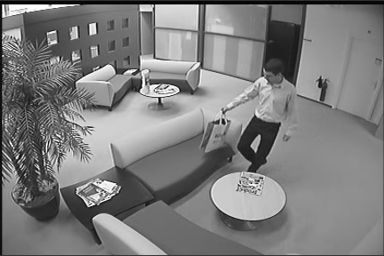}%
	\label{subfig: Candela}}
	\subfloat[\textit{Caviar1}]{ \includegraphics[width=0.33\textwidth]{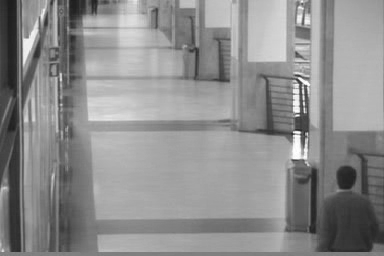}%
	  \label{subfig: Caviar1}}
	\subfloat[\textit{Caviar2}]{ \includegraphics[width=0.33\textwidth]{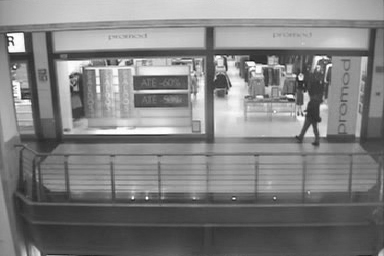}%
	 \label{subfig: Caviar2}} \\
	 	\subfloat[\textit{Candela} with $50\%$ missing pixels]{\includegraphics[width=0.33\textwidth]{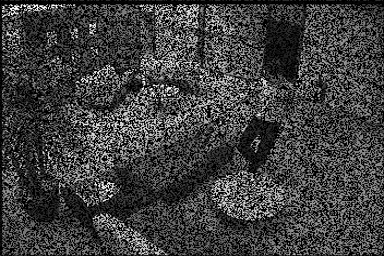}%
	\label{subfig: missing Candela}}
	\subfloat[\textit{Caviar1} with $50\%$ missing pixels]{ \includegraphics[width=0.33\textwidth]{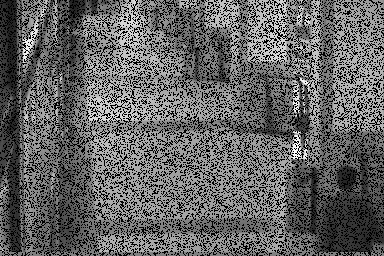}%
	  \label{subfig: missing Caviar1}}
	\subfloat[\textit{Caviar2} with $50\%$ missing pixels]{ \includegraphics[width=0.33\textwidth]{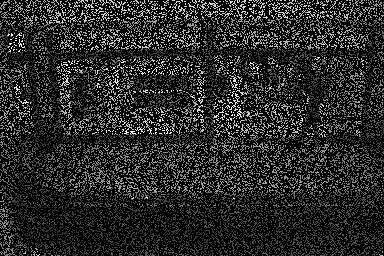}%
	 \label{subfig: missing Caviar2}}
\caption{Video data sets  for background subtraction: (a) \textit{Candela}; (b) \textit{Caviar1}; (c) \textit{Caviar2}; (d)  \textit{Candela} with  $50\%$ missing pixels;  (e) \textit{Caviar1} with $50\%$ missing pixels; (f)  \textit{Caviar2} with  $50\%$ missing pixels.} 
\label{fig: SBI images}
\end{figure}  Figures~\ref{subfig: Candela},~\ref{subfig: Caviar1}, and~\ref{subfig: Caviar2} show three image frames from the three video data sets. In the data set of \textit{Candela}, there is a man entering and leaving room, abandoning a bag. In the data set of \textit{Caviar1}, there are people slowly walking along a corridor, with mild shadows. For all videos, the first 80 image frames are used for experiments.  Thus, the tensor data size  is $256\times 384\times 80$. The background in each video data set is static, which is provided as the ground truth for comparison. There are people walking in the background, which are treated as the smooth foreground. In each video, a ratio of pixels in each image are set as missing pixels, and the positions of the missing pixels are unknown. The corresponding images with 50\% missing pixels from the three video data sets are shown in the second row of Figure~\ref{fig: SBI images}.  \begin{table}[!htb] 
\caption{Background subtraction results comparison on different video data sets with different missing ratios} 
\centering
 \resizebox{\textwidth}{!}{  
\begin{tabular}{lclrrrrrr}
\hline\hline
\multicolumn{1}{c}{Videos} & \begin{tabular}[c]{@{}c@{}}Missing\\Ratio\end{tabular}              & \multicolumn{1}{c}{Indices} & \multicolumn{1}{c}{MCOS} & \multicolumn{1}{c}{BFMNM} & \multicolumn{1}{c}{HQ-ASD} & \multicolumn{1}{c}{RTRC} & \multicolumn{1}{c}{HQ-TCASD} & \multicolumn{1}{c}{Proposed}  \\ 
\hline
\multirow{10}{*}{\textit{Candela}}  & \multirow{2}{*}{10\%} & PSNR                        & 33.10                    & 40.25                     & 38.55                      & 26.47                    & 24.90                        & \textbf{\textcolor{red}{43.23}}           \\
                           &                       & SSIM                        & 0.8342                   & 0.8784                    & 0.8824    & 0.7879                   & 0.5656                       & \textbf{\textcolor{red}{0.9011}}                        \\ 
\cline{2-9}
                           & \multirow{2}{*}{30\%} & PSNR                        & 33.06                    & 40.23                     & 37.38   & 26.25                    & 24.68                        & \textbf{\textcolor{red}{43.33}}                         \\
                           &                       & SSIM                        & 0.8241                   & 0.8761                    & 0.8573                     & 0.7459                   & 0.5258                       & \textbf{\textcolor{red}{0.9034}}        \\ 
\cline{2-9}
                           & \multirow{2}{*}{50\%} & PSNR                        & 32.90                    & 40.10                     & 35.27                      & 25.69                    & 23.95                        & \textbf{\textcolor{red}{43.34}}                         \\
                           &                       & SSIM                        & 0.8062                   & 0.8714                    & 0.8097                     & 0.6840                   & 0.4669                       & \textbf{\textcolor{red}{0.9048}}                        \\ 
\cline{2-9}
                           & \multirow{2}{*}{70\%} & PSNR                        & 17.65                    & 39.35                     & 31.82                      & 24.54                    & 22.46                        & \textbf{\textcolor{red}{43.29}}                         \\
                           &                       & SSIM                        & 0.3225                   & 0.8596                    & 0.7164                     & 0.5863                   & 0.3839                       & \textbf{\textcolor{red}{0.9044}}                        \\ 
\cline{2-9}
                           & \multirow{2}{*}{90\%} & PSNR                        & 8.74                     & 15.95                     & 23.07                      & 21.75                    & 19.18                        & \textbf{\textcolor{red}{41.14}}                         \\
                           &                       & SSIM                        & 0.0942                   & 0.3442                    & 0.4568                     & 0.3857                   & 0.2337                       & \textbf{\textcolor{red}{0.8555}}                        \\ 
\hline
\multirow{10}{*}{\textit{Caviar1}}  & \multirow{2}{*}{10\%} & PSNR                        & 29.07                    & 29.91                     & 31.83                      & 24.92                    & 24.63                        & \textbf{\textcolor{red}{36.68}}                         \\
                           &                       & SSIM                        & 0.7769                   & 0.7703                    & 0.7762                     & 0.7201                   & 0.6898                       & \textbf{\textcolor{red}{0.7992}}                        \\ 
\cline{2-9}
                           & \multirow{2}{*}{30\%} & PSNR                        & 23.04                    & 34.46                     & 31.61                      & 25.03                    & 24.65                        & \textbf{\textcolor{red}{36.74}}                         \\
                           &                       & SSIM                        & 0.3736                   & 0.7779                    & 0.7584                     & 0.7225                   & 0.6288                       & \textbf{\textcolor{red}{0.8028}}                        \\ 
\cline{2-9}
                           & \multirow{2}{*}{50\%} & PSNR                        & 17.88                    & 35.32                     & 30.64                      & 25.14                    & 24.60                        & \textbf{\textcolor{red}{36.75}}                         \\
                           &                       & SSIM                        & 0.2032                   & 0.7762                    & 0.7296                     & 0.7211                   & 0.5490                       & \textbf{\textcolor{red}{0.8051}}                        \\ 
\cline{2-9}
                           & \multirow{2}{*}{70\%} & PSNR                        & 15.54                    & 31.22                     & 28.98                      & 25.23                    & 24.24                        & \textbf{\textcolor{red}{36.78}}                         \\
                           &                       & SSIM                        & 0.1681                   & 0.7497                    & 0.6488                     & 0.7072                   & 0.4560                       & \textbf{\textcolor{red}{0.8068}}                        \\ 
\cline{2-9}
                           & \multirow{2}{*}{90\%} & PSNR                        & 8.27                     & 14.84                     & 22.20                      & 24.89                    & 22.21                        & \textbf{\textcolor{red}{36.62}}                         \\
                           &                       & SSIM                        & 0.0463                   & 0.1869                    & 0.3636                     & 0.6079                   & 0.3025                       & \textbf{\textcolor{red}{0.7901}}                        \\ 
\hline
\multirow{10}{*}{\textit{Caviar2}}  & \multirow{2}{*}{10\%} & PSNR                        & 36.28                    & 43.72                     & 43.34                      & 33.87                    & 31.49                        & \textbf{\textcolor{red}{44.65}}                         \\
                           &                       & SSIM                        & 0.9443                   & 0.9358                    & 0.9485                     & 0.8884                   & 0.8721                       & \textbf{\textcolor{red}{0.9493}}                        \\ 
\cline{2-9}
                           & \multirow{2}{*}{30\%} & PSNR                        & 27.04                    & 43.64                     & 41.03                      & 34.09                    & 31.11                        & \textbf{\textcolor{red}{44.79}}                         \\
                           &                       & SSIM                        & 0.6157                   & 0.9357                    & 0.9300                     & 0.8921                   & 0.8312                       & \textbf{\textcolor{red}{0.9509}}                        \\ 
\cline{2-9}
                           & \multirow{2}{*}{50\%} & PSNR                        & 21.59                    & 43.46                     & 36.66                      & 33.66                    & 29.52                        & \textbf{\textcolor{red}{44.88}}                         \\
                           &                       & SSIM                        & 0.4452                   & 0.9338                    & 0.8941                     & 0.8915                   & 0.7460                       & \textbf{\textcolor{red}{0.9518}}                        \\ 
\cline{2-9}
                           & \multirow{2}{*}{70\%} & PSNR                        & 18.55                    & 42.48                     & 31.62                      & 31.82                    & 26.96                        & \textbf{\textcolor{red}{44.85}}                         \\
                           &                       & SSIM                        & 0.3681                   & 0.9276                    & 0.8241                     & 0.8752                   & 0.6541                       & \textbf{\textcolor{red}{0.9508}}                        \\ 
\cline{2-9}
                           & \multirow{2}{*}{90\%} & PSNR                        & 9.24                     & 14.43                     & 22.44                      & 26.42                    & 22.49                        & \textbf{\textcolor{red}{43.31}}                         \\
                           &                       & SSIM                        & 0.0914                   & 0.3248                    & 0.5415                     & 0.7419                   & 0.4838                       & \textbf{\textcolor{red}{0.9362}}                        \\
\hline\hline
\multicolumn{9}{l}{Note: The bold red numbers are the best performance for each case.}
\end{tabular}
}\label{tab: background} 
\end{table}

In this experiment, the cases of $10\%$, $30\%$, $50\%$, $70\%$, and $90\%$ missing pixels are studied to show the performance of background subtraction under different missing ratios. The quantitative results of all benchmark methods with different missing ratios on simulated \textit{Candela}, \textit{Caviar1}, and \textit{Caviar2} are summarized in Table~\ref{tab: background} regarding PSNR and SSIM, respectively. For all cases, our method can achieve the best performance in terms of PSNR and SSIM. When the missing ratio increases, our proposed SRTC is the most consistent one among all the benchmark methods. Specifically,  our approach has a very small variation for different missing ratios, which is the only method that performs well for the case of 90\% missing pixels. Meanwhile, the performances of MCOS, BFMNM, and HQ-ASD degrade significantly when the missing ratio increases. RTRC and HQ-TCASD perform poorly for all cases.  These results demonstrate that the proposed method has the best performance in terms of accuracy due to the advantage of low-rank Tucker decomposition for the static background model over the nuclear norm.

To show the visualization  result, the background subtraction results from the case of $50\%$ missing ratio  on all three video data sets are demonstrated. The visualizations of the recovered video background from each video data set for different methods are shown in Figure~\ref{fig: background visualization}. \begin{figure}[!htbp]
\centering
\includegraphics[width=1\textwidth]{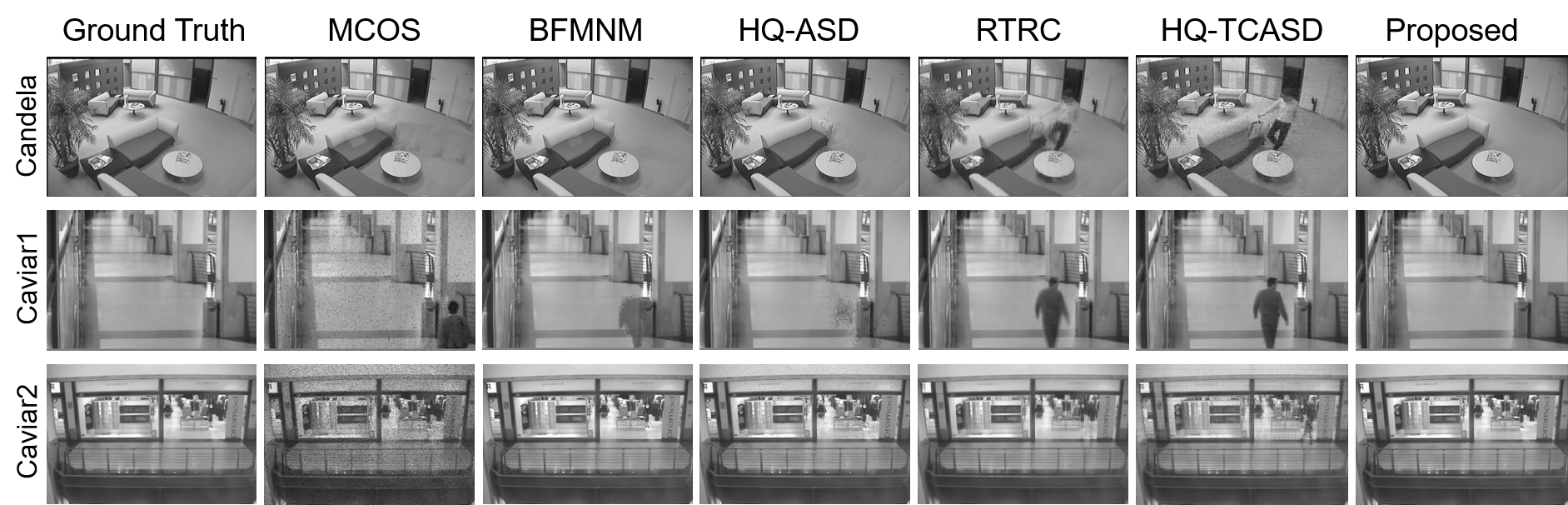}
\caption{Visualization results of different methods on a frame of \textit{Candela}, \textit{Caviar1}, and \textit{Caviar2} for background subtraction with 50\% missing pixels.}
 \label{fig: background visualization}
\end{figure}  Our proposed approach generally produces the cleanest background for all three video data sets. For the backgrounds generated using RTRC and HQ-TCASD, people still remain in the background even though the video even though the image is recovered.  For the results from BFMNM and HQ-ASD, their performance is very close to the proposed method. But they perform poorly for the \textit{Caviar1}, where there is shadow of people left in the background. For MOCS, it cannot recover the video background where there are still missing pixels 
in \textit{Caviar1} and \textit{Caviar2}.

\begin{figure}[!htbp]
	\centering
	\subfloat[\textit{Candela} with $50\%$ missing pixels]{\includegraphics[width=0.333\textwidth]{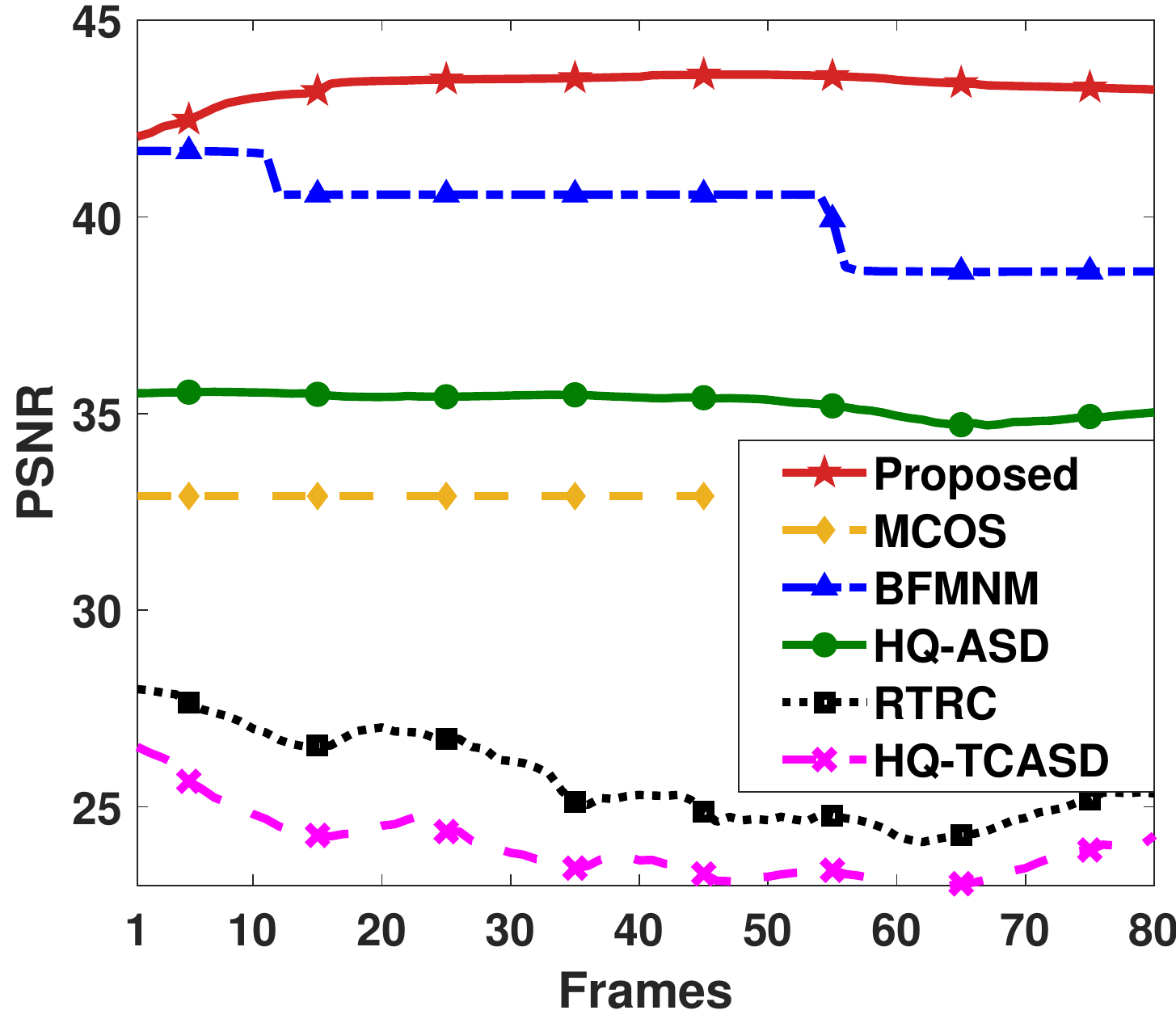}%
	\label{subfig: plot Candela}}
	\subfloat[\textit{Caviar1} with $50\%$  missing pixels]{ \includegraphics[width=0.333\textwidth]{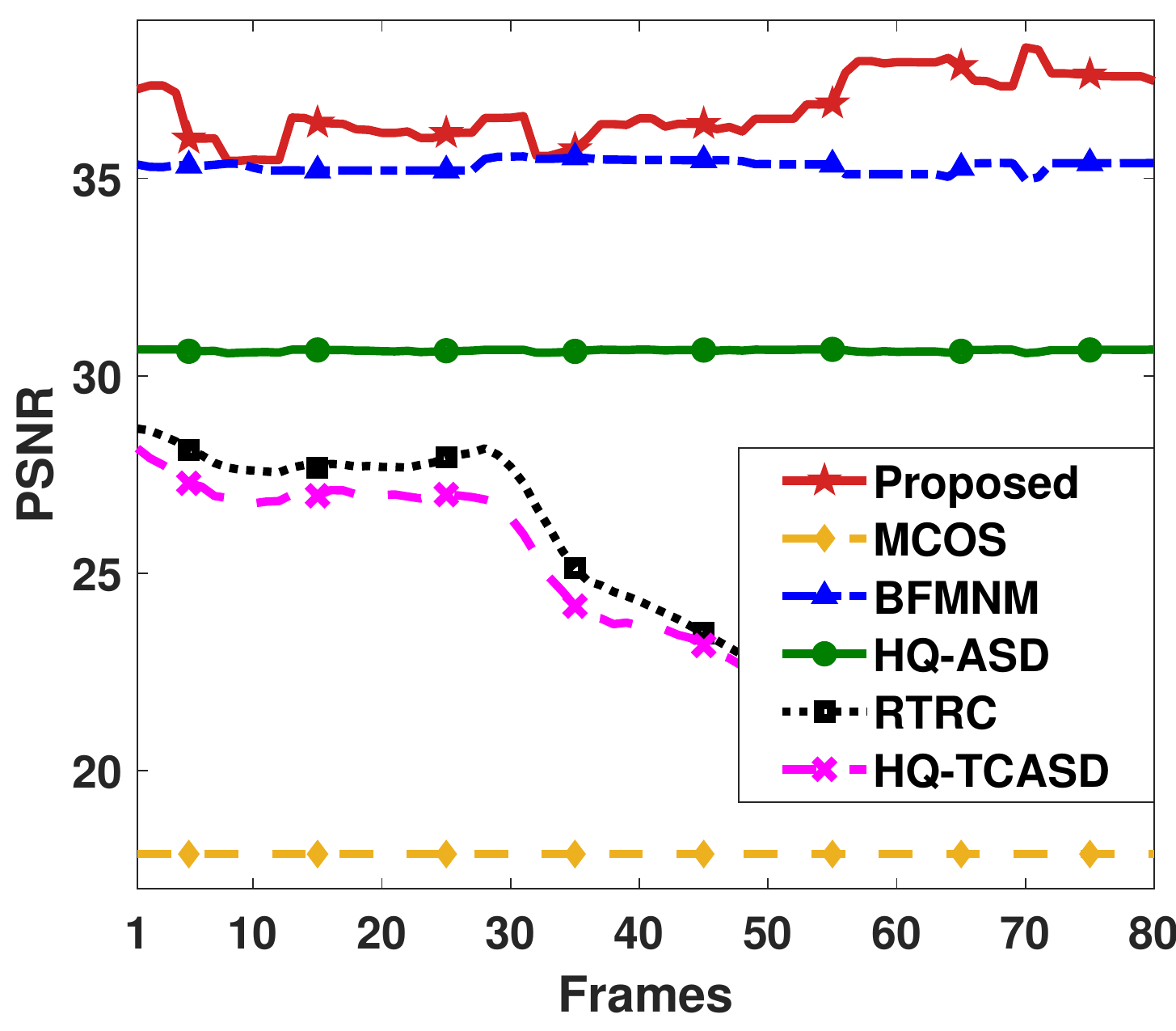}%
	  \label{subfig: plot CAVIAR1}}
	\subfloat[\textit{Caviar2} with $50\%$  missing pixels]{ \includegraphics[width=0.333\textwidth]{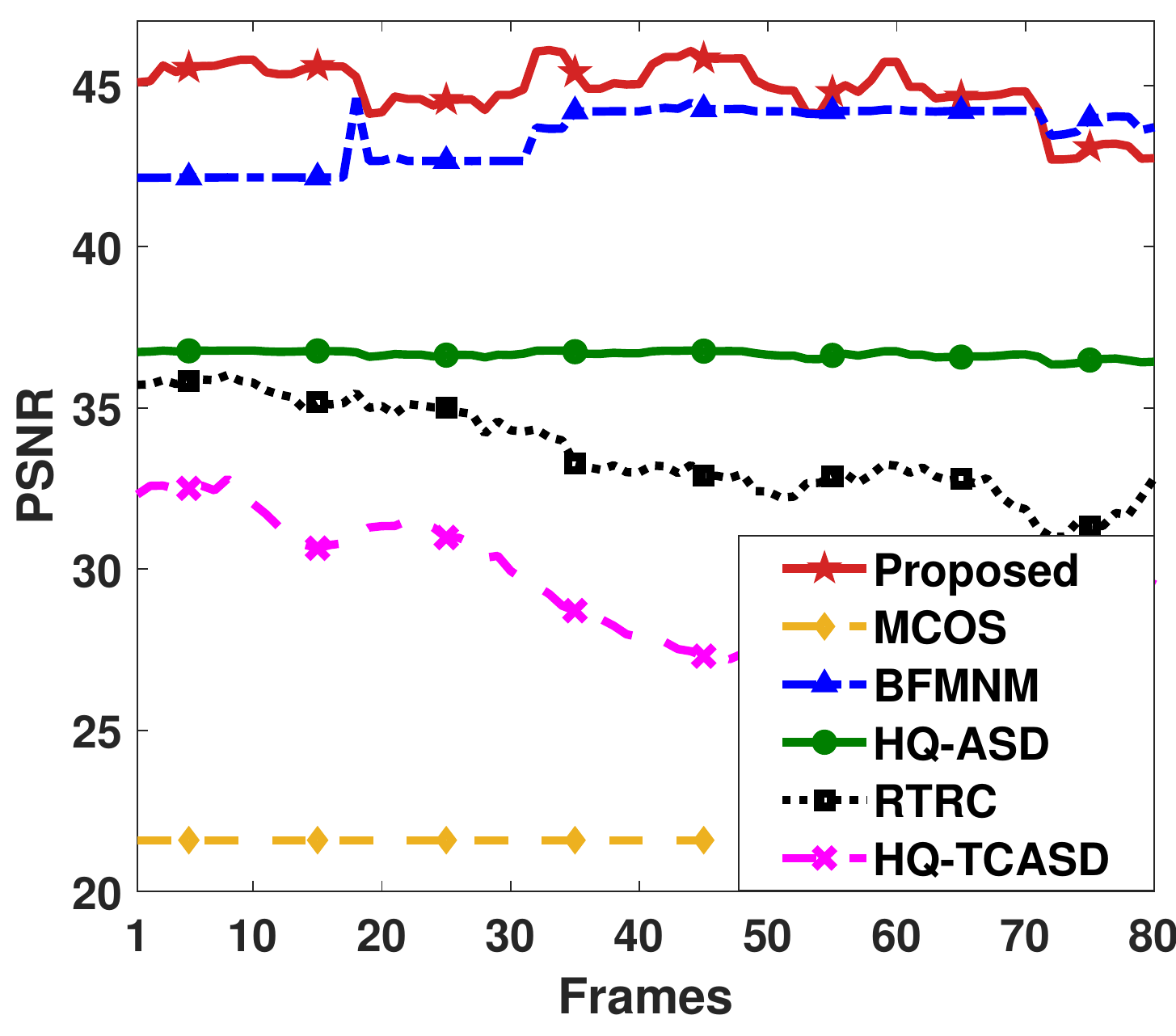}%
	  \label{subfig: plot CAVIAR2}}  
\caption{Background subtraction performance comparison by frames for different methods: (a) \textit{Candela} with $50\%$  missing pixels; (b) \textit{Caviar1} $50\%$  missing pixels; (c) \textit{Caviar2} $50\%$  missing pixels.} 
 \label{fig: Background vs Frame}
\end{figure}
Since the quantitative results in Table~\ref{tab: background} are the average performance of each algorithm for 80 frames in the video,   the performance variability for each frame is demonstrated.   For the case of  $50\%$ missing ratio, Figure~\ref{fig: Background vs Frame} shows PSNR of each frame in videos of \textit{Candela}, \textit{Caviar1}, and \textit{Caviar2}. This result shows that the proposed method is very stable since the variation among different frames is quite small. Besides, our proposed method can achieve the best performance for almost every frame, except for several frames in \textit{Caviar2}. BFMNM is the only comparable method, which has the second best performance, while other benchmark methods perform very poorly.

\subsection{Foreground Detection} \label{subsec: foreground detection}
In this subsection, our proposed method is applied to foreground detection. The video data sets, namely, \textit{Highway}, \textit{Office}, and \textit{Pedestrians} from CDnet data set\footnote{\url{http://jacarini.dinf.usherbrooke.ca/dataset2014/}.}~\citep{wang2014cdnet}, are used for experiments.\begin{figure}[!htb] 
	\centering 
	\subfloat[\textit{Highway}]{\includegraphics[width=0.33\textwidth]{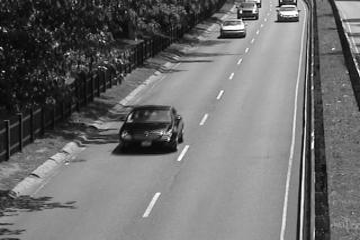}%
	\label{subfig: Highway}}
	\subfloat[\textit{Office}]{ \includegraphics[width=0.33\textwidth]{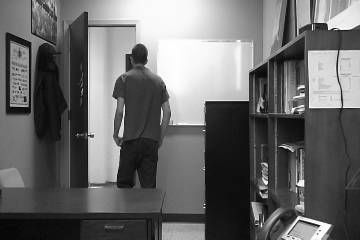}%
	  \label{subfig: Office}}
	\subfloat[\textit{Pedestrians}]{ \includegraphics[width=0.33\textwidth]{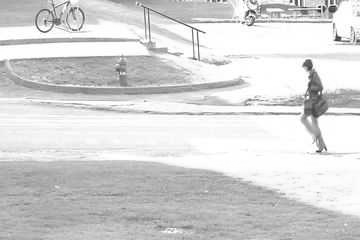}%
	 \label{subfig: Pedestrains}} \\
	 	\subfloat[\textit{Highway} with $70\%$ missing pixels]{\includegraphics[width=0.33\textwidth]{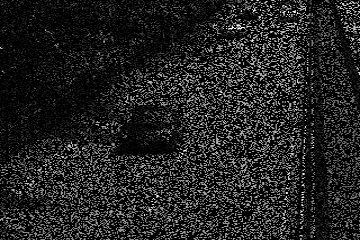}%
	\label{subfig: missing Highway}}
	\subfloat[\textit{Office} with $70\%$ missing pixels]{ \includegraphics[width=0.33\textwidth]{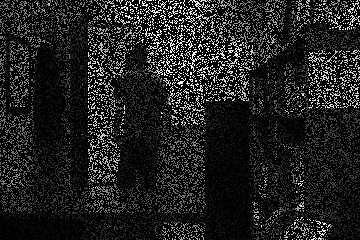}%
	  \label{subfig: missing Office}}
	\subfloat[\textit{Pedestrians} with $70\%$ missing pixels]{ \includegraphics[width=0.33\textwidth]{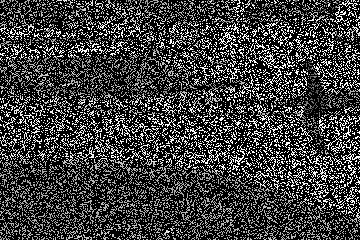}%
	 \label{subfig: missing Pedestrains}}
\caption{Video data sets  for foreground detection: (a) \textit{Highway}; (b) \textit{Office}; (c) \textit{Pedestrians}; (d)  \textit{Highway} with $70\%$ missing pixels;  (e) \textit{Office} with $70\%$ missing pixels; (f)  \textit{Pedestrians} with $70\%$ missing pixels.} 
\label{fig: CDnet images}
\end{figure}  Figures~\ref{subfig: Highway},~\ref{subfig: Office}, and~\ref{subfig: Pedestrains}  show three image frames from the three video data sets, respectively. In the data set of \textit{Highway}, there are vehicles moving along with the highway, where the size of each grayscale image  is $240\times320$. In the data set of \textit{Office}, there is a man entering, staying, and leaving the room, where the size of each grayscale image  is $240\times360$. In the data set of \textit{Pedestrians}, there are people walking on the road, where the size of each grayscale image  is $240\times360$. For all the videos, the first 80 image frames are used for experiments for the sake of computational time. The background in each video data set is static. The binary masks for the video foreground are provided as ground truth for comparison. In each video, a ratio of pixels in each image are set as missing pixels, and the positions of the missing pixels are unknown. In each video, a ratio of pixels in each image are set as missing pixels, and the positions of the missing pixels are unknown. The corresponding images with 70\% missing pixels from the three video data sets are shown in the second row of Figure~\ref{fig: CDnet images}. % In the experiment, we apply two pretrained deep learning models\footnote{\url{https://github.com/matlab-deep-learning}} on our data sets directly:  a Mask R-CNN~\citep{he2017mask} model trained on the COCO 2014 data set and a DeepLab~\citep{chen2018encoder} model trained on the PASCAL VOC data set.
\begin{table}[!htbp] 
\caption{Foreground detection results comparison on different video data sets with different missing ratios} 
\centering
 \resizebox{\textwidth}{!}{  
\begin{tabular}{cclrrrrrr} 
\hline\hline
Videos                        & \begin{tabular}[c]{@{}c@{}}Missing\\Ratio\end{tabular}           & Indexs    & \multicolumn{1}{c}{MCOS} & \multicolumn{1}{c}{BFMNM} & \multicolumn{1}{c}{HQ-ASD} & \multicolumn{1}{c}{RTRC} & \multicolumn{1}{c}{HQ-TCASD} & \multicolumn{1}{c}{Proposed}  \\ 
\hline
\multirow{15}{*}{\textit{Highway}}     & \multirow{3}{*}{10\%} & Precision & 0.5966                   & \textbf{\textcolor{red}{0.7881}}                    & 0.6944                     & 0.5276                   & 0.1917                       & 0.7810                        \\
                              &                       & Recall    & 0.7968                   & 0.2350                    & 0.8089                     & 0.6418                   & 0.2237                       & \textbf{\textcolor{red}{0.9443}}                        \\
                              &                       & F-measure    & 0.6794                   & 0.3557                    & 0.7457                     & 0.5786                   & 0.2044                       & \textbf{\textcolor{red}{0.8546}}                        \\ 
\cline{2-9}
                              & \multirow{3}{*}{30\%} & Precision & 0.5901                   & 0.7733                    & 0.5507                     & 0.4219                   & 0.1476                       & \textbf{\textcolor{red}{0.7864}}                        \\
                              &                       & Recall    & 0.7965                   & 0.1029                    & 0.7928                     & 0.5968                   & 0.1657                       & \textbf{\textcolor{red}{0.9453}}                        \\
                              &                       & F-measure    & 0.6756                   & 0.1791                    & 0.6480                     & 0.4934                   & 0.1539                       & \textbf{\textcolor{red}{0.8583}}                        \\ 
\cline{2-9}
                              & \multirow{3}{*}{50\%} & Precision & 0.4871                   & 0.4646                    & 0.3986                     & 0.3099                   & 0.1085                       & \textbf{\textcolor{red}{0.7919}}                        \\
                              &                       & Recall    & 0.2953                   & 0.0891                    & 0.7371                     & 0.5319                   & 0.0988                       & \textbf{\textcolor{red}{0.9456}}                        \\
                              &                       & F-measure    & 0.3644                   & 0.1453                    & 0.5155                     & 0.3906                   & 0.0987                       & \textbf{\textcolor{red}{0.8616}}                        \\ 
\cline{2-9}
                              & \multirow{3}{*}{70\%} & Precision & 0.3042                   & 0.4922                    & 0.2416                     & 0.1903                   & 0.1313                       & \textbf{\textcolor{red}{0.8068}}                        \\
                              &                       & Recall    & 0.1481                   & 0.0838                    & 0.6430                     & 0.4375                   & 0.0567                       & \textbf{\textcolor{red}{0.9368}}                        \\
                              &                       & F-measure    & 0.1951                   & 0.1412                    & 0.3495                     & 0.2642                   & 0.0727                       & \textbf{\textcolor{red}{0.8666}}                        \\ 
\cline{2-9}
                              & \multirow{3}{*}{90\%} & Precision & 0.4234                   & 0.4768                    & 0.0797                     & 0.0653                   & 0.0902                       & \textbf{\textcolor{red}{0.7775}}                        \\
                              &                       & Recall    & 0.0695                   & 0.0559                    & 0.3935                     & 0.2663                   & 0.0379                       & \textbf{\textcolor{red}{0.9315}}                        \\
                              &                       & F-measure    & 0.1003                   & 0.0936                    & 0.1314                     & 0.1035                   & 0.0520                       & \textbf{\textcolor{red}{0.8466}}                        \\ 
\hline
\multirow{15}{*}{\textit{Office}}      & \multirow{3}{*}{10\%} & Precision & 0.6602                   & 0.6725                    & 0.6373                     & 0.3291                   & 0.2844                       & \textbf{\textcolor{red}{0.8539}}                        \\
                              &                       & Recall    & 0.6033                   & 0.2673                    & 0.5988                     & 0.2669                   & 0.1307                       & \textbf{\textcolor{red}{0.9443}}                        \\
                              &                       & F-measure    & 0.6158                   & 0.3787                    & 0.6017                     & 0.2926                   & 0.1761                       & \textbf{\textcolor{red}{0.8954}}                        \\ 
\cline{2-9}
                              & \multirow{3}{*}{30\%} & Precision & 0.5971                   & 0.7364                    & 0.5026                     & 0.2721                   & 0.2383                       & \textbf{\textcolor{red}{0.8580}}                        \\
                              &                       & Recall    & 0.4194                   & 0.1417                    & 0.5414                     & 0.2568                   & 0.1247                       & \textbf{\textcolor{red}{0.9415}}                        \\
                              &                       & F-measure    & 0.4779                   & 0.2348                    & 0.5170                     & 0.2616                   & 0.1601                       & \textbf{\textcolor{red}{0.8963}}                        \\ 
\cline{2-9}
                              & \multirow{3}{*}{50\%} & Precision & 0.5258                   & 0.8091                    & 0.3738                     & 0.2051                   & 0.2796                       & \textbf{\textcolor{red}{0.8584}}                        \\
                              &                       & Recall    & 0.2197                   & 0.1022                    & 0.5049                     & 0.2481                   & 0.0834                       & \textbf{\textcolor{red}{0.9450}}                        \\
                              &                       & F-measure    & 0.2844                   & 0.1787                    & 0.4250                     & 0.2211                   & 0.1229                       & \textbf{\textcolor{red}{0.8982}}                        \\ 
\cline{2-9}
                              & \multirow{3}{*}{70\%} & Precision & 0.6565                   & 0.7505                    & 0.2274                     & 0.1272                   & 0.2904                       & \textbf{\textcolor{red}{0.8601}}                        \\
                              &                       & Recall    & 0.0850                   & 0.0867                    & 0.4403                     & 0.2435                   & 0.0665                       & \textbf{\textcolor{red}{0.9483}}                        \\
                              &                       & F-measure    & 0.1496                   & 0.1536                    & 0.2965                     & 0.1643                   & 0.1063                       & \textbf{\textcolor{red}{0.9008}}                        \\ 
\cline{2-9}
                              & \multirow{3}{*}{90\%} & Precision & 0.8400                   & 0.8042                    & 0.0790                     & 0.0446                   & 0.0992                       & \textbf{\textcolor{red}{0.8490}}                        \\
                              &                       & Recall    & 0.0720                   & 0.0782                    & 0.2646                     & 0.2442                   & 0.0657                       & \textbf{\textcolor{red}{0.9397}}                        \\
                              &                       & F-measure    & 0.1303                   & 0.1404                    & 0.1211                     & 0.0746                   & 0.0772                       & \textbf{\textcolor{red}{0.8908}}                        \\ 
\hline
\multirow{15}{*}{\textit{Pedestrians}} & \multirow{3}{*}{10\%} & Precision & 0.6252                   & \textbf{\textcolor{red}{0.8364}}                    & 0.6931                     & 0.6042                   & 0.2750                       & 0.7737                        \\
                              &                       & Recall    & 0.8304                   & 0.1653                    & 0.8708                     & 0.7026                   & 0.4153                       & \textbf{\textcolor{red}{0.9626}}                        \\
                              &                       & F-measure    & 0.7115                   & 0.2758                    & 0.7713                     & 0.6490                   & 0.3291                       & \textbf{\textcolor{red}{0.8577}}                        \\ 
\cline{2-9}
                              & \multirow{3}{*}{30\%} & Precision & 0.6197                   & \textbf{\textcolor{red}{0.8421}}                    & 0.5429                     & 0.4823                   & 0.2073                       & 0.7849                        \\
                              &                       & Recall    & 0.8256                   & 0.0697                    & 0.7956                     & 0.6728                   & 0.3385                       & \textbf{\textcolor{red}{0.9617}}                        \\
                              &                       & F-measure    & 0.7062                   & 0.1238                    & 0.6446                     & 0.5612                   & 0.2553                       & \textbf{\textcolor{red}{0.8642}}                        \\ 
\cline{2-9}
                              & \multirow{3}{*}{50\%} & Precision & 0.3066                   & 0.2665                    & 0.3913                     & 0.3462                   & 0.1232                       & \textbf{\textcolor{red}{0.7972}}                        \\
                              &                       & Recall    & 0.1251                   & 0.0868                    & 0.6984                     & 0.6392                   & 0.2251                       & \textbf{\textcolor{red}{0.9626}}                        \\
                              &                       & F-measure    & 0.1667                   & 0.1209                    & 0.5007                     & 0.4484                   & 0.1570                       & \textbf{\textcolor{red}{0.8719}}                        \\ 
\cline{2-9}
                              & \multirow{3}{*}{70\%} & Precision & 0.5241                   & 0.3147                    & 0.2416                     & 0.2105                   & 0.0834                       & \textbf{\textcolor{red}{0.8159}}                        \\
                              &                       & Recall    & 0.0320                   & 0.0864                    & 0.5373                     & 0.5668                   & 0.1306                       & \textbf{\textcolor{red}{0.9583}}                        \\
                              &                       & F-measure    & 0.0586                   & 0.1299                    & 0.3329                     & 0.3063                   & 0.0990                       & \textbf{\textcolor{red}{0.8811}}                        \\ 
\cline{2-9}
                              & \multirow{3}{*}{90\%} & Precision & 0.2561                   & 0.8010                    & 0.0818                     & 0.0721                   & 0.0703                       & \textbf{\textcolor{red}{0.8448}}                        \\
                              &                       & Recall    & 0.0324                   & 0.0272                    & 0.2427                     & 0.3659                   & 0.0375                       & \textbf{\textcolor{red}{0.9283}}                        \\
                              &                       & F-measure    & 0.0539                   & 0.0486                    & 0.1220                     & 0.1201                   & 0.0428                       & \textbf{\textcolor{red}{0.8838}}                        \\
\hline\hline
\multicolumn{9}{l}{Note: The bold red numbers are the best performance for each case.}
\end{tabular}
}\label{tab: foreground} 
\end{table}

In this experiment, the cases of $10\%$, $30\%$, $50\%$, $70\%$, and $90\%$ missing pixels  are investigated to show the performance of foreground detection under different missing ratios. The quantitative results of all benchmark methods with different missing ratios of simulated \textit{Highway}, \textit{Office}, and \textit{Pedestrians} are summarized in Table~\ref{tab: foreground} in terms of Precision, Recall, and F-measure, respectively. In terms of precision, our method can achieve the best performance, except in one case, namely, \textit{Highway} with 10\% missing pixels. In terms of recall and F-measure, the proposed SRTC can achieve the best performance in all cases.  Our proposed SRTC is the most consistent one while  all the benchmark methods show significant performance degradation when the missing ratio increases. Specifically, the performances of our method with $10\%$, $30\%$, $50\%$, $70\%$, and $90\%$ missing ratios are almost the same.  MCOS and HQ-ASD have good performance when the missing ratio is small. BFMNM and HQ-TCASD perform very poorly for all cases. Overall, our proposed approach shows the best performance due to the advantage of smoothness modeling of the video foreground instead of sparsity. 

\begin{figure}[!ht]
\centering
\includegraphics[width=1\textwidth]{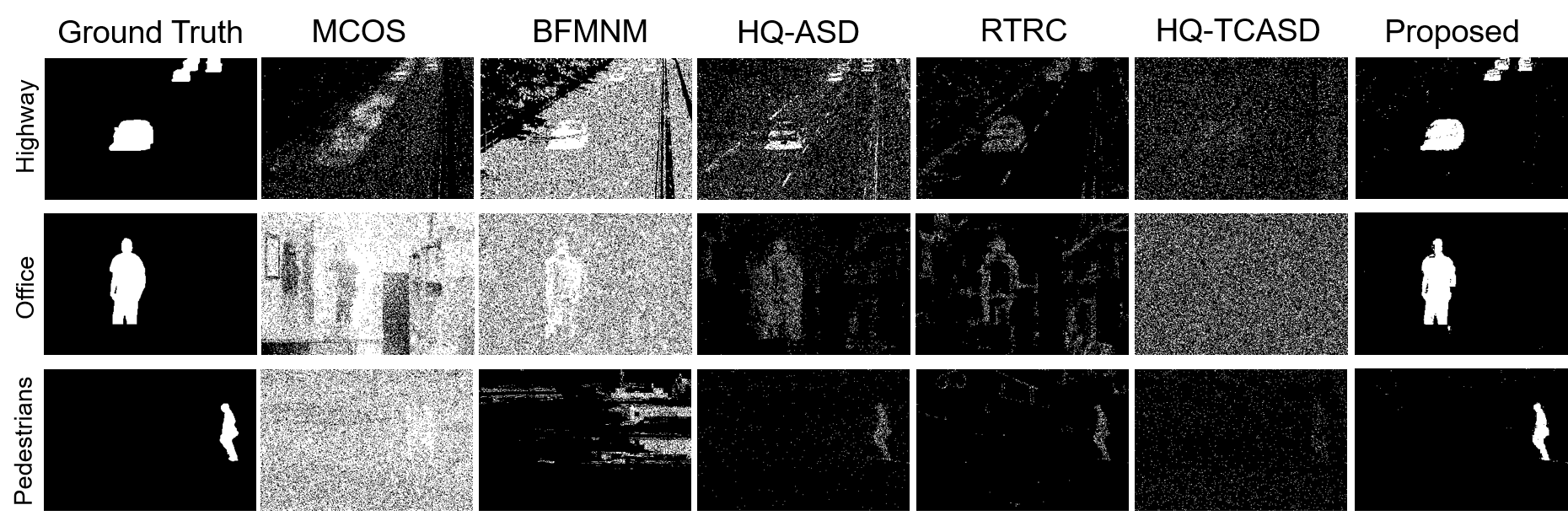}
\caption{Visualization results of different methods on a frame of \textit{Highway}, \textit{Office}, and \textit{Pedestrians} for the foreground with 70\% missing pixels.}
 \label{fig: foreground visualization}
\end{figure}
\begin{figure}[!ht]
	\centering
	\subfloat[\textit{Highway} with $70\%$ missing pixels]{\includegraphics[width=0.333\textwidth]{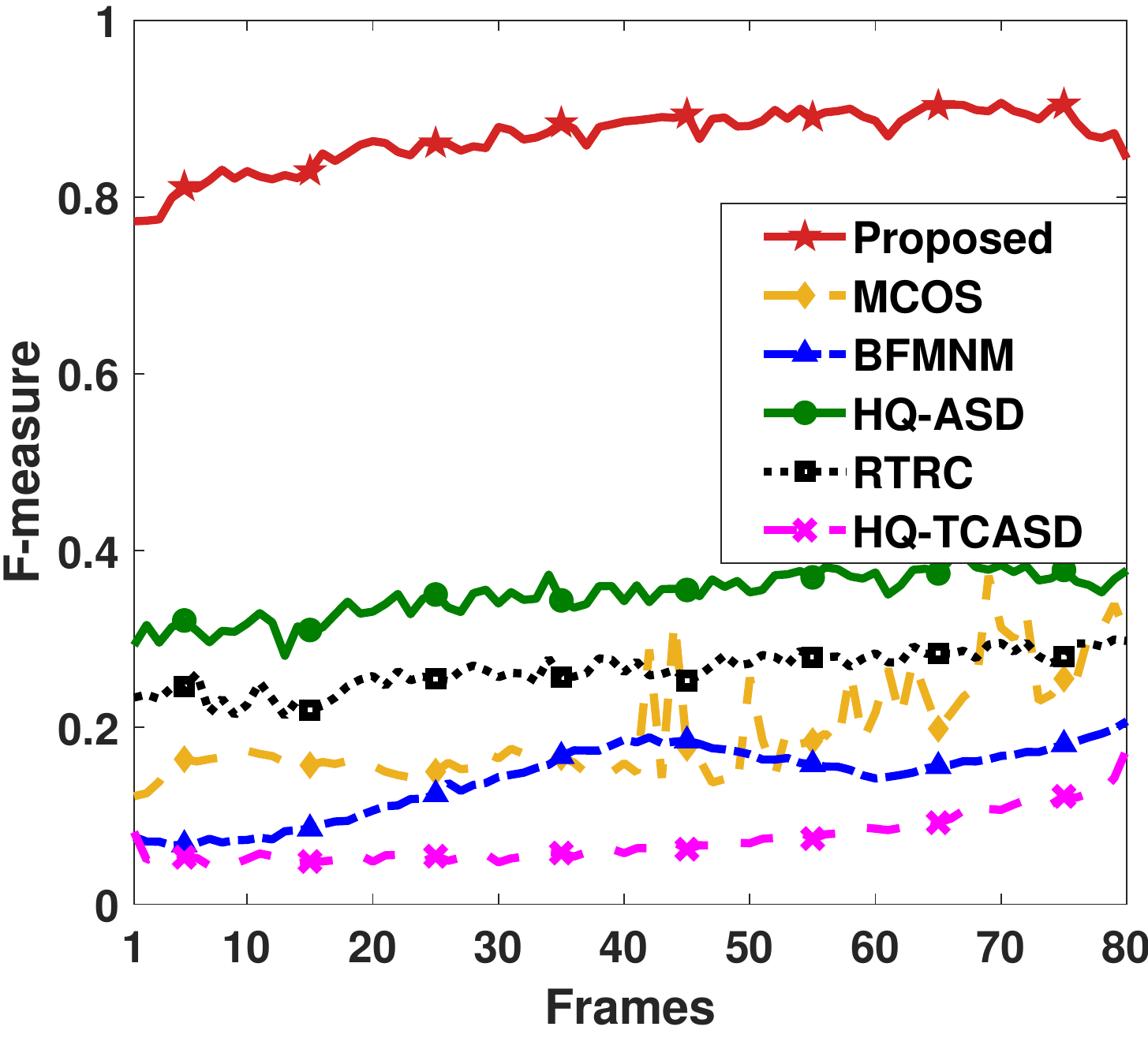}%
	\label{subfig: plot Highway by frame}}
	\subfloat[\textit{Office} with $70\%$  missing pixels]{ \includegraphics[width=0.333\textwidth]{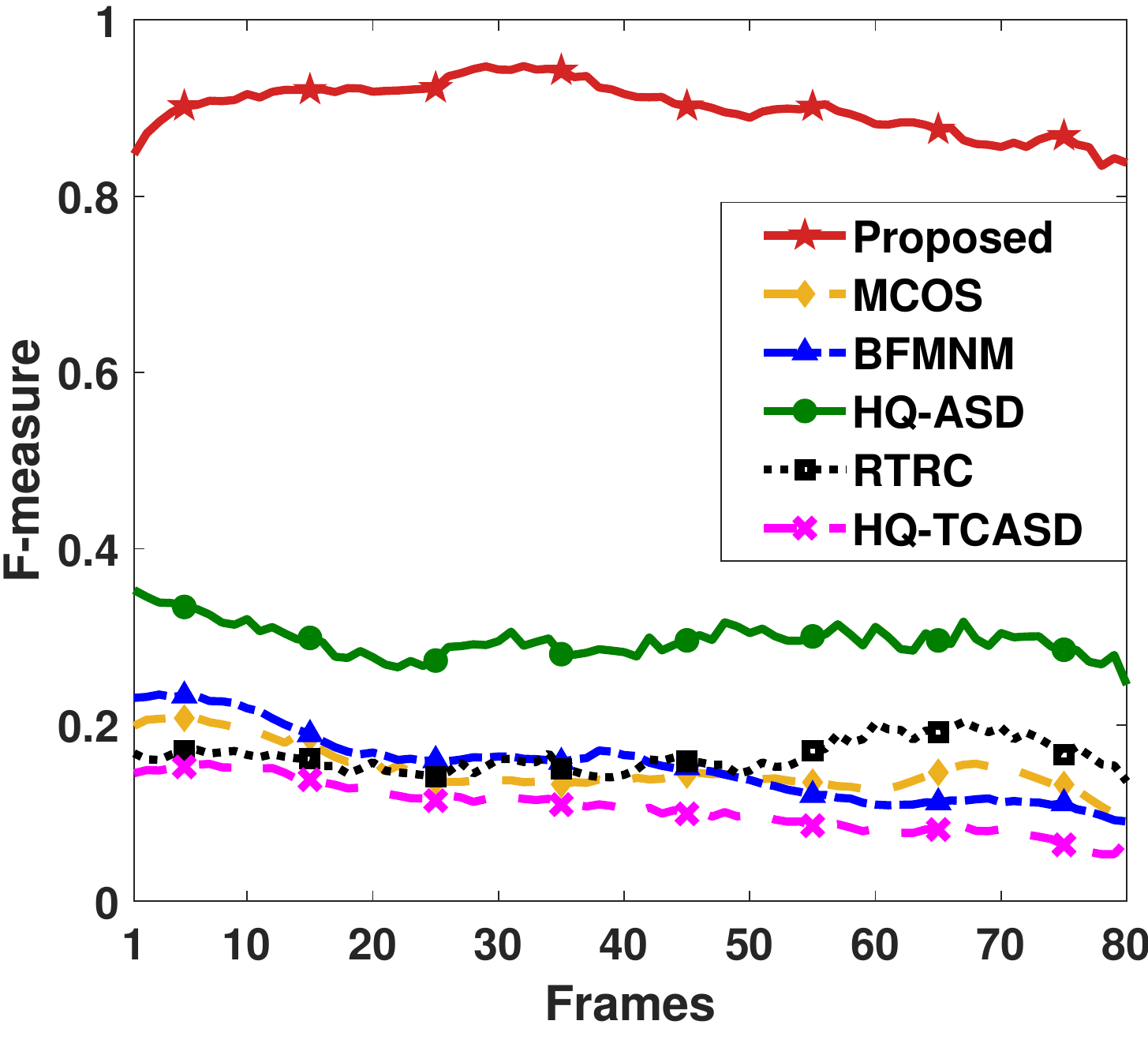}%
	  \label{subfig: plot Office by frame}}
	\subfloat[\textit{Pedestrians} with $70\%$  missing pixels]{ \includegraphics[width=0.333\textwidth]{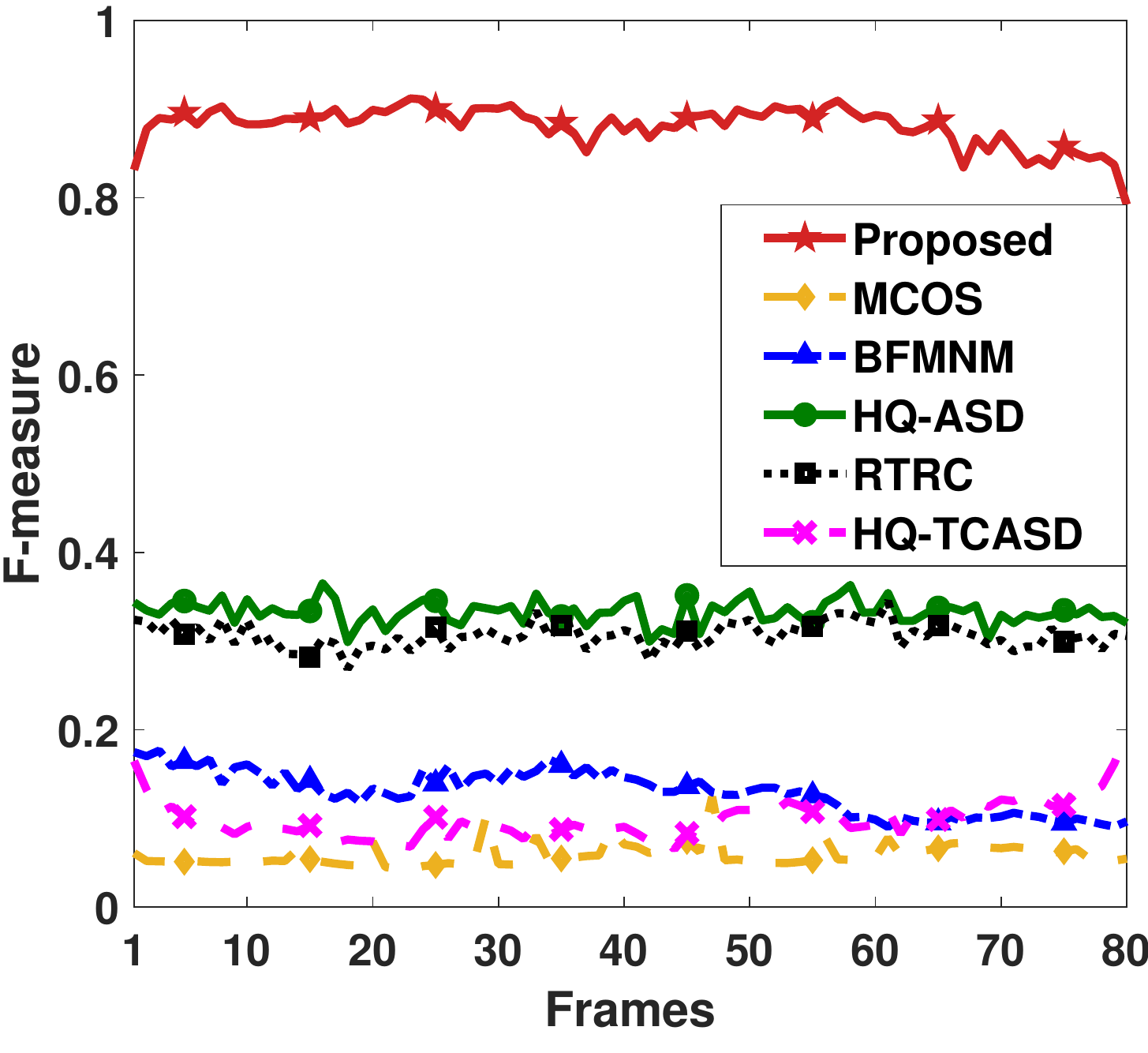}%
	  \label{subfig: plot Pedestrians by frame}}  
\caption{Foreground detection performance comparison by frames for different methods: (a) \textit{Highway} with $70\%$ missing pixels; (b) \textit{Office} with $70\%$  missing pixels; (c) \textit{Pedestrians} with $70\%$  missing pixels.} 
 \label{fig: foreground vs Frame}
\end{figure}
The foreground detection results from the case of $70\%$ noise ratio are demonstrated to show the visualization  result. The visualizations of one frame from each video data set for different methods are shown in Figure~\ref{fig: foreground visualization}. In general, our method can detect the most accurate foreground among all benchmarks even though the video is missing. For the foreground masks subtracted from MCOS, BFMNM, and HQ-TCASD, a lot of noise remains in the foreground causing by the missing pixels, where the moving objects are not well detected either. The results from HQ-ASD and RTRC can detect the shape of the moving objects. However, the detected foreground is not as complete as our result.

Since the quantitative results in Table~\ref{tab: foreground} are the average performance of each algorithm for 80 frames in the video,   the performance variability for each frame is demonstrated.   For the case of  $70\%$ missing ratio, Figure~\ref{fig: foreground vs Frame} shows F-measure of each frame in videos of \textit{Highway},  \textit{Office}, and  \textit{Pedestrians}. This result shows that the proposed method is very stable since the variation among different frames is quite small. In addition, our proposed method can achieve  much better performance than all other benchmark methods for all frames.
% Since the quantitative results in  are the best performance for each algorithm from 100 sets of tuning parameters,  boxplots are used to show the performance variability for each algorithm.  For the case of  $20\%$ noise, the boxplots for   in Figure summarize five-number standards (the minimum, the maximum, the sample median, and the first and third quartiles) for each algorithm in terms of F-measure.  This result shows that our method is very stable for different values of $\lambda$ since these five numbers have very small variations. Even though TVRPCA  can achieve good performance, it is very sensitive to the tuning parameters compared with our method. PRPCA is a very stable algorithm. However, its performance is poor. In this experiment, the worst performance from our method is still much better than the best performance from all other methods.

% \section{Real-world Case Study} \label{sec: real-world case study}
\section{Conclusion} \label{sec: conclusion}
 In this article, a new smooth robust tensor completion is developed for background/foreground separation with missing pixels. The proposed SRTC simultaneously recovers the video data and  decomposes it into low-rank and smooth components, respectively. To achieve the solutions efficiently, a tensor proximal alternating minimization (tenPAM) algorithm for SRTC is implemented. The global convergence of the tenPAM algorithm has also been established. The empirical convergence experiment shows that the proposed SRTC can converge  and run efficiently in practice. The background subtraction and foreground detection results on simulated video data demonstrate that our method outperforms MCOS, BFMNM, HQ-ASD, RTRC, and HQ-TCASD, which are state-of-the-art algorithms in the literature.  These results also  illustrate the effectiveness of  Tucker decomposition for the low-rank tensor and total variation regularization for the smooth tensor.   %More importantly,  the proposed SS-RTD  can be considered as an ready-to-use algorithm since its performance is very stable when the only tuning parameter $\lambda$ takes any value from $[0.2,1]$ 

% In addition, there are still some aspects of SS-RTD that deserve further investigation. First, the video background is assumed to be static in the proposed model. However, the extension  to the case of dynamic background should be further investigated.  Second, the empirical convergence of the proposed algorithm based on ADMM is provided in our case study. Thereafter, the theoretical convergence of the proposed algorithm needs further research.

% {\noindent \em Remainder omitted in this sample. See http://www.jmlr.org/papers/ for full paper.}

% Acknowledgements should go at the end, before appendices and references

\acks{Mr. Shen and Dr. Kong are supported by the Office of Naval Research under Award Number N00014-18-1-2794, and the Department of Defense under Award Number N00014-19-1-2728. Dr. Xie has been supported in part by the National Science Foundation grants 2046426 and 2153607.}


% \newpage

\appendix

% Note: in this sample, the section number is hard-coded in. Following
% proper LaTeX conventions, it should properly be coded as a reference:

%In this appendix we prove the following theorem from
%Section~\ref{sec:textree-generalization}:

% {\noindent \em Remainder omitted in this sample. See http://www.jmlr.org/papers/ for full paper.}
 \section*{Appendix A. Additional Theoretical Results}
 \subsection*{Appendix A1. Convergence of $\BFU$}\label{app: theorem}
 In our main paper, the convergence of the sequence $\{\BFU_{k}\BFU_{k}^{\top}\}_{k>0}$ from our Algorithm~\ref{alg: PAM} can be derived. The following theorem shows how we can get the convergent sequence of $\{\BFU_{k,new}\}_{k>0}$ from $\{\BFU_{k}\BFU_{k}^{\top}\}_{k>0}$.
\begin{theorem} \label{thm: convergence of U}
Let $\BFU_*\BFU_*^{\top}=\{\BFU_*^{(1)}\BFU_*^{(1)\top},\BFU_*^{(2)}\BFU_*^{(2)\top},\dots,\BFU_*^{(N)}\BFU_*^{(N)\top}\}$ is the limit point of Algorithm~\ref{alg: PAM}. By SVD, $\BFU_k^{(n)\top}\BFU_*^{(n)}=\BFW\Sigma\BFV$. $\BFD_k^{(n)}=\BFW\BFV$. Now let $\BFU_{k,new}^{(n)}=\BFU_{k}^{(n)}\BFD_k^{(n)}$, then 
\[\underset{k\to\infty}{\lim} \BFU_{k,new}^{(n)}=\BFU_*^{(n)}\]
\end{theorem}
\begin{proof}
For any $n$, we have that $\underset{k\to\infty}{\lim} \BFU_k^{(n)}\BFU_k^{(n)\top}=\BFU_*^{(n)}\BFU_*^{(n)\top}$, therefore, $\underset{k\to\infty}{\lim} \|\BFU_k^{(n)}\BFU_k^{(n)\top}- \BFU_*^{(n)}\BFU_*^{(n)\top} \|_F=0$.
\[\BFD_k^{(n)}\in \underset{\BFD \in \mathbb{R}^{r_n\times r_n}}{\arg\min} \|\BFU_*^{(n)} - \BFU_k^{(n)}\BFD\|_F,\]
where $\BFD$ is an orthogonal matrix. This problem is called Orthogonal Procrustes problem.
\begin{align*}
   \|\BFU_*^{(n)} - \BFU_k^{(n)}\BFD\|_F^2 & = \mathrm{Tr}\big((\BFU_*^{(n)} - \BFU_k^{(n)}\BFD)(\BFU_*^{(n)} - \BFU_k^{(n)}\BFD)^{\top}\big) \\
   & = \mathrm{Tr}(\BFU_*^{(n)}\BFU_*^{(n)\top})+\mathrm{Tr}(\BFU_k^{(n)}\BFU_k^{(n)\top}) - 2\mathrm{Tr}(\BFU_k^{(n)\top}\BFU_*^{(n)}\BFD^{\top}) \\
   & = 2 r_n - 2\mathrm{Tr}(\BFU_k^{(n)\top}\BFU_*^{(n)}\BFD^{\top})
\end{align*}
Equivalently,
\[\BFD_k^{(n)}\in \underset{\BFD \in \mathbb{R}^{r_n\times r_n}}{\arg\max} \mathrm{Tr}(\BFU_k^{(n)\top}\BFU_*^{(n)}\BFD^{\top}).\]
By SVD, $\BFU_k^{(n)\top}\BFU_*^{(n)}=\BFW\Sigma\BFV$. $\BFD_k^{(n)}=\BFW\BFV$ will be the optimal solution for the above problem, but it may not be a unique solution.
\begin{align*}
     \underset{\BFD }{\max}\ \mathrm{Tr}(\BFU_k^{(n)\top}\BFU_*^{(n)}\BFD^{\top}) & =\|\BFU_k^{(n)\top}\BFU_*^{(n)}\|_* \\
     & = \sum_{i=1}^{r_n}\sigma_i(\BFU_k^{(n)\top}\BFU_*^{(n)}),
\end{align*}
where $\sigma_i(\cdot)$ is $i$-th largest singular value of a matrix. Next, we claim that $\forall k,n$, $\sigma_i(\BFU_k^{(n)\top}\BFU_*^{(n)})\leq 1$.
\begin{align*}
    \sigma_i(\BFU_k^{(n)\top}\BFU_*^{(n)})  & \leq  \sqrt{\lambda_{\textnormal{max}}(\BFU_k^{(n)\top}\BFU_*^{(n)}\BFU_*^{(n)\top}\BFU_k^{(n)})} \\
    & = \sqrt{\underset{\|\bm{x}\|_2=1 }{\max}\mathrm{Tr} (\bm{x}^{\top}\BFU_k^{(n)\top}\BFU_*^{(n)}\BFU_*^{(n)\top}\BFU_k^{(n)}\bm{x})} \\
    & \leq \sqrt{\underset{\|\bm{y}\|_2=1 }{\max}\mathrm{Tr} (\bm{y}^{\top}\BFU_*^{(n)}\BFU_*^{(n)\top}\bm{y})} \\
    & =\sqrt{\lambda_{\textnormal{max}}(\BFU_*^{(n)}\BFU_*^{(n)\top})}=1
\end{align*} 
where $\lambda_{\textnormal{max}}(\cdot)$ is the maximum eigenvalue of a positive semidefinite matrix. The second inequality is because $\bm{y}=\BFU_k^{(n)}\bm{x}$ may not span the entire $\mathbb{R}^{r_n}$ space.
Finally,
\begin{align*}
 \|\BFU_*^{(n)} - \BFU_k^{(n)}\BFD_k^{(n)}\|_F^2 & = 2r_n -2 \sum_{i=1}^{r_n}\sigma_i(\BFU_k^{(n)\top}\BFU_*^{(n)}) \\
 & \leq 2r_n - 2\sum_{i=1}^{r_n}\sigma_i(\BFU_k^{(n)\top}\BFU_*^{(n)})^2 \\
 & = 2r_n - 2\sum_{i=1}^{r_n}\lambda_{i}(\BFU_k^{(n)\top}\BFU_*^{(n)}\BFU_*^{(n)\top}\BFU_k^{(n)}) \\
  & = 2r_n - 2\mathrm{Tr}(\BFU_k^{(n)\top}\BFU_*^{(n)}\BFU_*^{(n)\top}\BFU_k^{(n)}) \\
 & = \|\BFU_k^{(n)}\BFU_k^{(n)\top} - \BFU_*^{(n)}\BFU_*^{(n)\top}\|_F^2,
\end{align*}
where $\lambda_i(\cdot)$ is $i$-th largest eigenvalue of a matrix. The first inequality is due to $x\geq x^2, \forall x\in [0,1]$.
Now let $\BFU_{k,new}^{(n)}=\BFU_{k}^{(n)}\BFD_k^{(n)}$, then  $\underset{k\to\infty}{\lim} \BFU_{k,new}^{(n)}=\BFU_*^{(n)}.$
\end{proof}

 \subsection*{Appendix A2. Rate of Convergence of tenPAM Algorithm~\ref{alg: PAM}}\label{app: theorem 1}
The following result shows that if we can specify the $\theta$ value in Definition~\ref{df: kl} of function $G(\cdot)$, then we are able to establish the rate of convergence of tenPAM algorithm.
\begin{theorem}[\normalfont\textit{Rate of Convergence}~\citep{attouch2009convergence,bolte2014proximal}] \label{thm: rate of convergence}
Suppose that the desingularizing function of function $G(\cdot)$ is in the form of $\phi(s)=cs^{1-\theta}$ for some positive constant $c>0$. Then one of the following results must hold 
	\begin{enumerate}[label=(\roman*)]
\item If $\theta=0$, then the sequence $\{\BFU_k\BFU_k^{\top},\BFcalS^k,\BFcalX^k\}_{k\geq 0}$ converges after a finite number of iterations;
\item If $\theta \in (0,1/2]$, then there exist $\eta>0$ and $\tau \in [0,1)$ such that $\|(\BFU_k\BFU_k^{\top},\BFcalS^k,\BFcalX^k) - (\BFU_*\BFU_*^{\top},\BFcalS^*,\BFcalX^*)\|_F\leq \eta \tau^k$ ; and
\item If $\theta \in (1/2,1)$, then there exist $\eta>0$ such that 
\[\|(\BFU_k\BFU_k^{\top},\BFcalS^k,\BFcalX^k) - (\BFU_*\BFU_*^{\top},\BFcalS^*,\BFcalX^*)\|_F\leq \eta k^{-\frac{1-\theta}{2\theta-1}}, \]
	\end{enumerate}
	where $(\BFU_*\BFU_*^{\top},\BFcalS^*,\BFcalX^*)$ is the limit point of the sequence $\{\BFU_k\BFU_k^{\top},\BFcalS^k,\BFcalX^k\}_{k\geq 0}$.
\end{theorem}
However, it is difficult to specify the $\theta$ of function $G(\cdot)$. Thus,  the rate of convergence of tenPAM algorithm is unknown in general. This could be one of our future research directions.
\vskip 0.2in
\bibliography{sample}

\end{document}